\documentclass{article} 
\usepackage{iclr2027_conference,times}

\usepackage{hyperref}
\hypersetup{hidelinks}
\usepackage{url}

\usepackage{amsmath,amssymb}
\usepackage{amsthm}
\usepackage{graphicx}
\usepackage{float}
\usepackage{booktabs}
\usepackage{xcolor}
\usepackage{subcaption}
\usepackage{wrapfig}
\usepackage{longtable}
\graphicspath{{./figures/}}

\def\eqref#1{equation~\ref{#1}}

\newtheorem{proposition}{Proposition}
\newtheorem{definition}{Definition}

\title{Conservation Buys Stability and Factoring Buys Counterfactuals in Physical World Models}

\author{Yufeng Wang \thanks{These authors contributed equally to this work}\\
  Stony Brook University \\
  \And
  Lu Wei $^*$\\
  Stony Brook University \\
  \And
  Parivesh Priye $^*$\\
  Georgia Institute of Technology \\
  \And
  Haibin Ling \\
Westlake University \\}

\iclrfinalcopy
\makeatletter
\g@addto@macro\@maketitle{\lhead{}}
\makeatother

\begin{document}

\maketitle

\begin{abstract}
  A learned simulator can reproduce its training conditions accurately yet fail in two distinct ways once those conditions change. Over long rollouts, small errors accumulate until the trajectory drifts away from physically plausible behavior; under an intervention on a physical parameter, the model may continue to follow the law seen during training rather than the intervened one. We show that these two failures require \emph{different structural remedies}. Evolving a learned energy with a symplectic integrator preserves the geometry of the conservative dynamics and keeps rollouts bounded and physically meaningful for up to $100\times$ the training horizon, while equal-capacity predictors, an energy-regularized predictor, and a tuned neural ODE diverge. By contrast, encoding the physical coupling through an explicit linear factorization enables the model to follow a never-seen sign of that coupling, whereas an unrestricted parameterization remains locked to the training law. Crucially, the two mechanisms are separable: removing the structure responsible for long-horizon stability leaves counterfactual transfer intact, while removing the factorized coupling destroys counterfactual transfer without eliminating stability. This \emph{double dissociation}, established with matched controls that remove or replace one structural component at a time, persists beyond the headline three-body system and remains visible when the physical state must be inferred from pixels rather than provided directly. The result is a concrete design principle for physical world models: long-horizon stability and changed-law generalization arise from distinct structural commitments, and each can be imposed deliberately without requiring the other.
\end{abstract}

\section{Introduction}
\label{sec:intro}

A learned simulator predicts the next state of a physical system and iterates that prediction into a trajectory. Modern models can perform well by conventional measures such as rollout accuracy or visual realism, yet they often fail in two distinct ways once the governing conditions move beyond training. First, prediction errors accumulate until long rollouts drift or become unstable. Second, when the physical law itself changes, for example through reversed gravity, a flipped coupling, or reversed time, the model often continues to follow the law seen during training. Both failures remain visible in current systems. The best reported video generator reaches only $22\%$ joint accuracy on the hard split of VideoPhy-2, no model exceeds $3.3/5$ on PhyGround, and PhyWorldBench includes a dedicated ``anti-physics'' category for changed law failures
\citep{yi2020clevrer,bear2021physion,phyworldbench2025,videophy2_2025,cjepa_2026}.
Related limitations also appear in learned simulators across scales
\citep{sanchezgonzalez2020gns,zhong2021benchmark}.
This paper asks: \emph{which piece of physical structure fixes which failure?}

We address this question by encoding physical invariances directly rather than asking the model to infer them from finite data. The ingredients are individually established, including Hamiltonian and Lagrangian neural networks, generalized coordinate models, constrained dynamics, and symplectic structure preserving parameterizations
\citep{greydanus2019hnn,cranmer2020lnn,lutter2019delan,finzi2020constrained,%
zhong2020symoden}.
What remains unclear is which structural commitment produces which behavior and where each commitment stops helping. We show a \emph{double dissociation}. Conserving a learned energy through a symplectic update keeps long rollouts physically bounded but does not determine how the system should respond to a changed law. Conversely, making an interaction force enter linearly through its physical coupling enables transfer to a never seen sign of that coupling but does not prevent long horizon drift. Matched controls that remove or replace one ingredient at a time show that stability depends on geometric and energetic consistency, whereas counterfactual transfer depends on how the intervened parameter enters the dynamics.

This paper makes four contributions. First, we establish this double dissociation with matched controls, showing that conservation produces long horizon stability and a factored coupling produces counterfactual inversion, while neither produces the other and neither is explained by capacity, integrator choice, or formalism alone. Second, we formalize the mechanism behind each behavior and map the boundary at which each prior stops helping. Third, we show that the dissociation extends beyond smooth few body motion to generalized coordinates, contact, many body systems, mixed force families, and pixel based settings with grounded perception, including a learned object binder and frozen public video encoders. Fourth, we extend the same structural account beyond conservative dynamics, using a passive friction port that reaches the correct temperature and a factored drive that generalizes across unseen forcing strengths and signs.

We study a learned transition map from physical state, including positions and momenta, to the next state and iterate it into a trajectory. To separate dynamics from perception, we provide state in three ways: as oracle state; as state recovered from synthetic renders by an encoder whose position channel is anchored by either a fixed renderer or a label free centroid; and as state discovered from pixels with no anchor. The first two settings recover the target behaviors, while the third exposes a boundary of the current approach. Accordingly, this study concerns physical law fidelity rather than visual fidelity and does not claim improvements in real video generation, image quality, or leaderboard performance.

\section{Related work}
\label{sec:related}

\textbf{Structured dynamics and parameter conditioning.} Hamiltonian and Lagrangian networks encode conservation laws or variational structure
\citep{greydanus2019hnn,cranmer2020lnn,lutter2019delan}. Generalised coordinate, constrained, and intrinsically symplectic variants extend these ideas to rigid and structured systems
\citep{finzi2020constrained,zhong2020symoden,jin2020sympnets}, with continuous conservation statements for the learned energy and discrete behaviour determined in part by the numerical integrator
\citep{zhong2021benchmark}. Conditioning dynamics on physical parameters, including coupling interactions through multiplicative coefficients, is also standard
\citep{greydanus2019hnn,cranmer2020lnn,lutter2019delan,zhong2020symoden,%
jin2020sympnets,brunton2016sindy,han2021adaptable,raissi2019pinn}, while context adaptation extends prediction to new parameter draws within the training family
\citep{kirchmeyer2022coda,yin2021leads,wang2022dyad}. Our setting differs in two respects. First, we intervene on a coupling whose sign lies outside the training support. Second, we compare matched factored and unrestricted versions of HNN, DeLaN, and SymODEN to determine which structural commitment supports stability and which supports counterfactual inversion.

\textbf{Simulators, dissipation, contacts, and perception.} Graph network simulators are strong relational predictors
\citep{sanchezgonzalez2020gns,battaglia2016interaction,satorras2021egnn}, and training noise is a common remedy for rollout instability; in the simulator studied here, it delays but does not prevent unbounded drift. Dissipative, port Hamiltonian, and GENERIC or metriplectic models separate reversible from irreversible dynamics
\citep{sosanya2022dhnn,zhong2020dissipative,desai2021phnn,lee2021gfinn}. D-HNN evaluates unseen friction coefficients, port Hamiltonian networks recover energy, forcing, and dissipation, and compositional port Hamiltonian learning assembles systems from learned components
\citep{neary2023compositional}. Differentiable contact models further extend structured dynamics to collisions
\citep{zhong2021contact,hochlehnert2021contact}. On the perception side, Hamiltonian generative models and unsupervised Lagrangian models learn dynamics directly from images
\citep{toth2020hgn,zhong2020unsupervised}, while \citet{botev2021priors} identify continuous time reversibility as the most broadly useful physical prior for pixel based learning. We show that structure alone does not induce a canonical latent representation, and that reversibility helps only when the underlying dynamics are themselves reversible. 
While recent video world models and physical reasoning benchmarks report persistent inconsistencies with physical laws \citep{yi2020clevrer,bear2021physion,riochet2018intphys,phyworldbench2025,%
vjepa2_2025,videophy2_2025,cjepa_2026,phyground_2026}, our work isolates the specific structural commitments required to guarantee long-horizon stability and counterfactual generalisation.

\section{Methods}
\label{sec:method}

\subsection{Setup and the two behaviours we measure}

\begin{figure}[h]
  \centering
  \includegraphics[width=\linewidth]{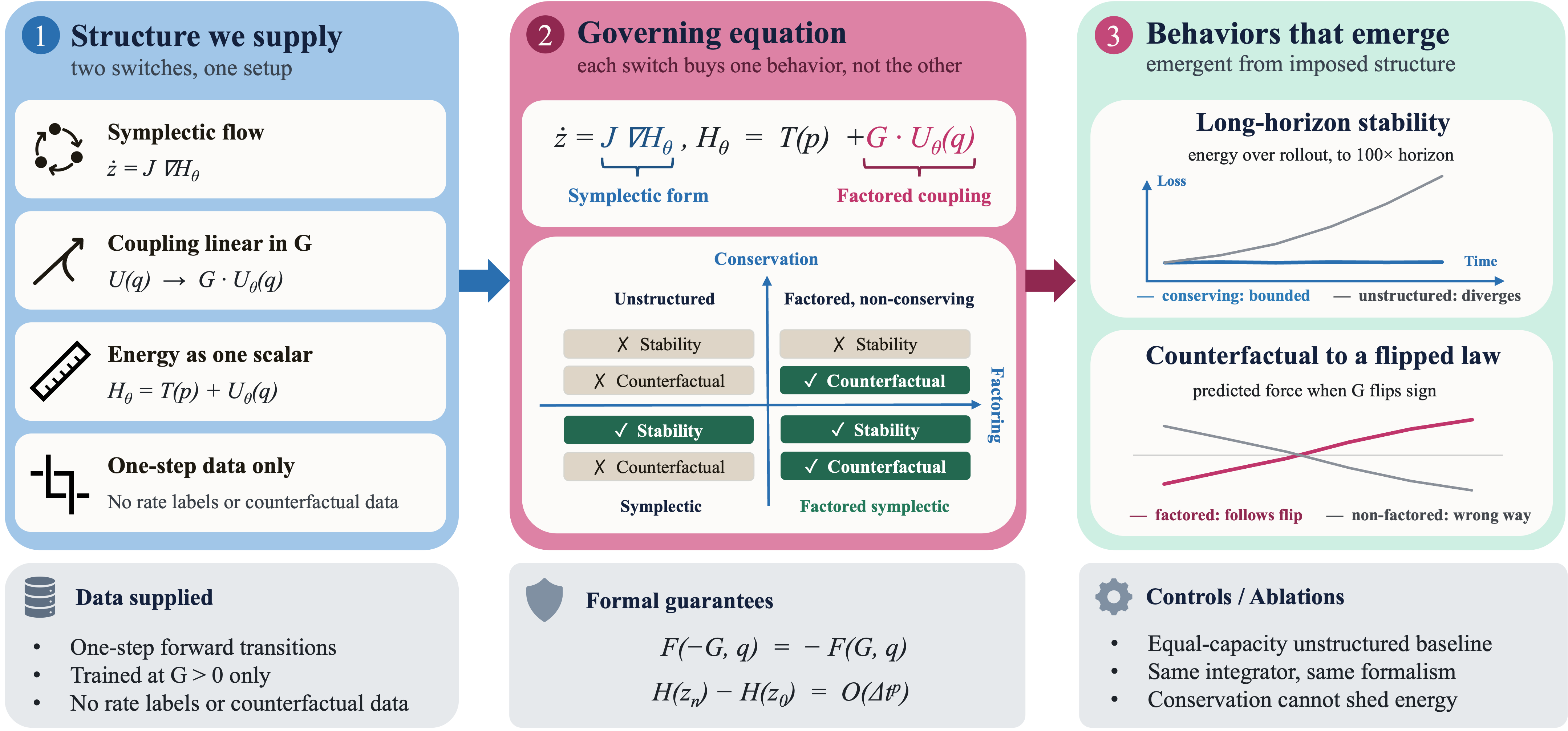}
  \caption{Overview of the two pieces of physical structure we impose (left), the double
    dissociation they produce (centre), and how each behaviour is measured and bounded
    (right and bottom).}
  \label{fig:overview}
\end{figure}

Figure~\ref{fig:overview} summarizes the study. Left: we impose two structural commitments—energy conservation (evolving a learned energy $H_\theta$ with a symplectic step) and factored coupling (the intervened parameter $G$ enters the force linearly, with all other content learned from forward-law data). Center: a double dissociation—conservation yields column-wise stability, factoring yields row-wise changed-law response, so only the model with both is stable \emph{and} answers the counterfactual. Right/bottom: measurements, matched controls that rule out capacity, integrator choice, and formalism, and the limits where each structure stops helping. Below we formalize each component. For $N$ bodies, the state is $x=(q,p)$ with positions $q$ and momenta $p$, and parameters $c=(G,m_1,\dots)$ include the coupling $G$ and masses. A dynamics model is a transition map iterated into a rollout,
\begin{equation}
  x_{t+1}=f_\theta(x_t,c),\qquad
  x_{0:T}=\big(x_0,f_\theta(x_0,c),\dots\big),
\end{equation}
and is trained only on one step prediction from forward law data with attractive gravity,
$G>0$. Our headline system is softened three body gravity,
\[
V(q)=-\sum_{i<j}\frac{Gm_im_j}
{\sqrt{\lVert q_i-q_j\rVert^2+\epsilon^2}},
\]
with evaluation rollouts of $T=2000$ steps, one hundred times the training horizon.
The remaining systems, including Coulomb charges, colliding disks, linear drag, a double
pendulum in generalised coordinates, a confined thermal bath, and driven particles,
together with the datasets, optimiser, and training settings, are specified in
App.~\ref{app:formal}.

We ask two questions of the trained transition map. The first is long horizon stability,
measured by the relative drift of the true energy along a finite rollout,
\begin{equation}
  D(t)=\frac{\lvert E(x_t)-E(x_0)\rvert}
  {\lvert E(x_0)\rvert},\qquad t\le T,
  \label{eq:drift}
\end{equation}
where $E$ is fixed by the reference Hamiltonian. We report both the per episode median
and the worst episode, and count nonfinite trajectories explicitly rather than assigning
them an arbitrary numerical value. The second behaviour is counterfactual inversion. We
train only on attractive gravity, $G>0$, and evaluate the unseen repulsive regime,
$G<0$. Let $\hat x_{0:T}$ be the model's rollout, and let $x^{-}_{0:T}$ and $x^{+}_{0:T}$ be the
ground-truth rollouts under the flipped repulsive law ($G<0$) and the trained attractive law
($G>0$), both started from the same true initial state so that initialisation error is
counted rather than cancelled. The regime-match indicator asks which of these two references
the model's rollout lies closer to,
\begin{equation}
  \mathrm{RM}
  =
  \mathbb{1}\!\left[
  \lVert \hat x_{0:T}-x^{-}_{0:T}\rVert
  <
  \lVert \hat x_{0:T}-x^{+}_{0:T}\rVert
  \right],
  \label{eq:rm}
\end{equation}
so $\mathrm{RM}=1$ when the rollout follows the changed law it never saw in training and
$\mathrm{RM}=0$ when it stays with the training law. Alongside it we report a normalized
trajectory error, the squared error against the flipped-law reference $x^{-}_{0:T}$ divided
by that reference's variance (Def.~\ref{def:metrics}).

\subsection{Two structural commitments, and what each guarantees}

A symplectic model evolves a learned Hamiltonian $H_\theta$ through a symplectic update.
In continuous time, the corresponding flow
\begin{equation}
  \dot q=\frac{\partial H_\theta}{\partial p},
  \qquad
  \dot p=-\frac{\partial H_\theta}{\partial q}
  \label{eq:ham}
\end{equation}
conserves $H_\theta$. The discrete leapfrog update, however, conserves neither
$H_\theta$ nor the reference energy exactly.
\begin{proposition}[Near conservation of a modified Hamiltonian; formal statement in App.~\ref{app:theory}]
  \label{prop:stab}
  Under the regularity, bounded orbit, and small step assumptions of backward error
  analysis, leapfrog applied to Eq.~\eqref{eq:ham} preserves a modified Hamiltonian
  $\widetilde H_\theta=H_\theta+O(dt^2)$ up to exponentially small error over
  exponentially long times \citep{hairer2006geometric}.
\end{proposition}
This statement concerns $\widetilde H_\theta$, not the physical reference energy.
Moreover, a generic MLP potential need not be coercive, so boundedness of the true
energy remains an empirical finite horizon property. We therefore report learned energy
conservation, physical energy drift, state escape, and trajectory fidelity as separate
outcomes.

The second commitment is a coupling factored potential, which makes the interaction
force exactly linear in $G$:
\begin{equation}
  U_\theta(q)=\sum_{i<j}\varphi_\theta(q_i-q_j),\qquad
  V(q,c)=G\,U_\theta(q),\qquad
  F_i=-G\,\partial_{q_i}U_\theta(q),
  \label{eq:factored}
\end{equation}
where $\varphi_\theta$ is learned using attractive data only.
\begin{proposition}[Sign inversion by construction]
  \label{prop:invert}
  For any learned $\varphi_\theta$ and any $G$ independent kinetic energy, the force in
  Eq.~\eqref{eq:factored} is odd in $G$,
  $F_i(q,-G)=-F_i(q,G)$ for all $q$. The dynamics at $-G$ therefore follow the same
  learned interaction under the reversed coupling without requiring training examples
  from the opposite sign regime.
\end{proposition}
The model family supplies the sign inversion, while the learned interaction
$\varphi_\theta$ must still recover the correct magnitude and state dependence.
Counterfactual accuracy therefore holds only where the learned content remains accurate
on the states visited by the flipped rollout. By contrast, a nonfactored model conditions
a black box potential $V_\theta(q,G)$ on $G$ without imposing any relation between
$G>0$ and $G<0$. Section~\ref{sec:r2} shows that such models can represent the flipped
law but do not learn to select it from forward law data alone. Proofs are given in
App.~\ref{app:theory}.

\subsection{Controls, and how claims are measured}

We test each claim with ablations that remove or replace one structural ingredient at a time, attributing behaviour to the required component. For conservation, we compare capacity-matched baselines: an unstructured predictor, a fixed physical-energy penalty, and a tuned neural ODE, to separate the roles of model capacity, objective, and integration. For factoring, we compare a black box that receives the coupling without its linear form and a model that injects it via positive localized features, to test whether inversion arises from capacity or conditioning alone. We repeat these tests across Hamiltonian, Lagrangian, and symplectic-ODE formalisms and a graph-network simulator. In chaotic regimes, primary metrics are energy drift (Eq.~\eqref{eq:drift}) and regime match (Eq.~\eqref{eq:rm}); trajectory error is secondary, and divergence is distinguished from finite error. We repeat experiments across independent initializations, reporting uncertainty via error bars, standard deviations, or $p$-values; near-binary outcomes are assessed per complete run. Results are assigned to one of four pre-registered evidence classes (well-supported, directional, supported, boundary). Full model lists, matching, statistics, and per-run results appear in App.~\ref{app:formal}, App.~\ref{app:results}, and App.~\ref{app:evidence}.

\section{Results}

In this section we first show that energy conservation stabilises long rollouts but does not answer a changed law. We then show that a factored coupling answers the changed law but does not provide stability. We combine both effects in a single dissociation table, rule out the main confounds, map where each structural prior stops helping, and finally show that the pattern extends beyond the headline system.

\subsection{Conservation makes long rollouts stable, but cannot answer a changed law}
\label{sec:r1}

\begin{table}[h]
  \centering
  \caption{Relative true-energy drift $D(t)$ over the rollout.}
  \label{tab:t-stab}
  \begin{tabular}{@{}lcccc@{}}
    \toprule
    Model & $t{=}100$ & $t{=}500$ & $t{=}2000$ & NF\\
    \midrule
    Unstructured MLP (one-step map) & $2.6$ & ${>}10^{4}$ & ${>}10^{4}$ & $61$\\
    Energy-penalised MLP, $\lambda{=}1$ & $3.3$ & ${>}10^{4}$ & ${>}10^{4}$ & $123$\\
    Energy-penalised MLP, $\lambda{=}100$ & $0.011$ & $0.33$ & $115$--${>}10^{4}$ & $0$\\
    Neural ODE (RK4, matched $dt$, tuned) & $0.08$ & $3\times10^{2}$--$4\times10^{3}$ & ${>}10^{4}$ & $0$\\
    Symplectic (ours) & $0.40$ & $2.7$ & $7.1$ [$6.7$, $7.7$] & $0$\\
    \bottomrule
  \end{tabular}
\end{table}

The first question is whether conserving a learned energy stabilises long rollouts where alternative approaches do not. We train the five models in Table~\ref{tab:t-stab} using one-step prediction on attractive-gravity trajectories and roll each model out for $2000$ steps. Every reported value comes from a single raw record that retains the finite status of each episode (App.~\ref{app:A}). In the table, ${>}10^{4}$ denotes a per-episode median above the display cap, while NF counts trajectories that become non-finite by step $2000$, out of $192$ episodes.

\begin{wrapfigure}{l}{0.50\linewidth}
  \centering
  \vspace{-9pt}
  \includegraphics[width=\linewidth]{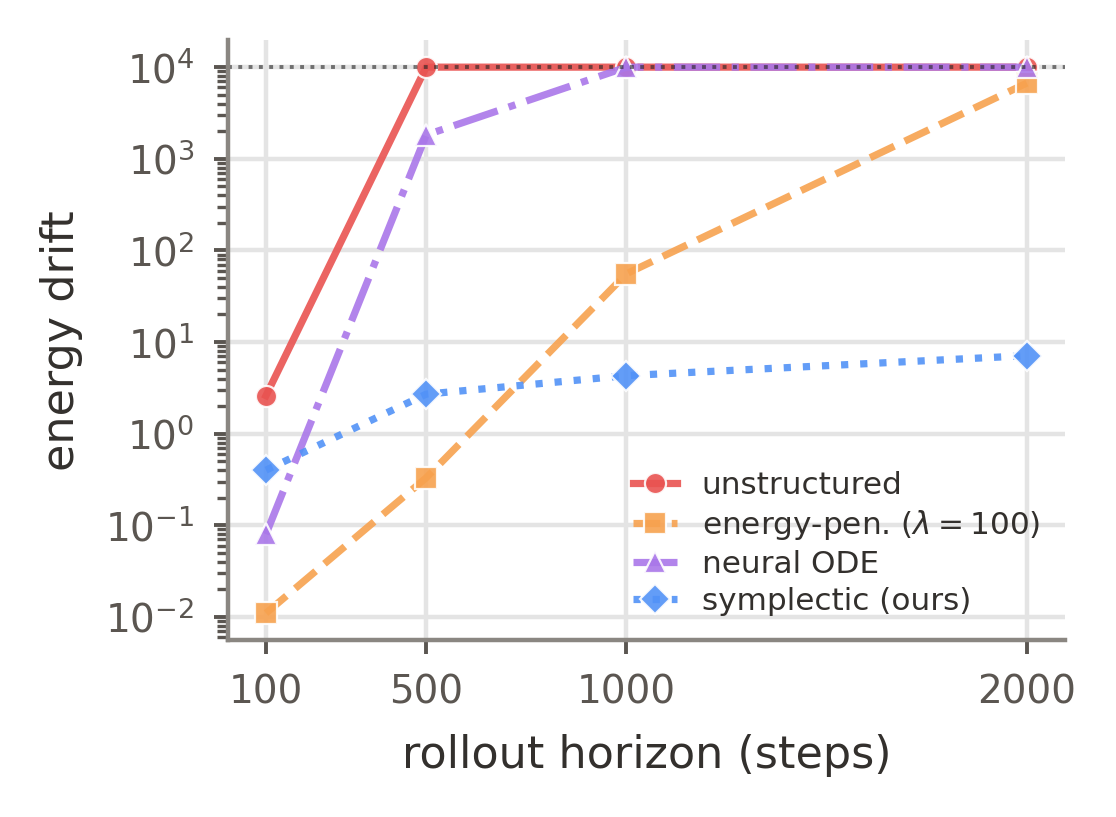}
  \caption{Long-horizon energy drift (log scale, capped at $10^{4}$): only the symplectic
    model stays bounded, while the $\lambda{=}100$ penalty and the neural ODE reach lower
    drift at short horizons before diverging.}
  \label{fig:stability}
  \vspace{-10pt}
\end{wrapfigure}

The symplectic model is the only one that stays finite and below the $10^{4}$ cap in all episodes (Fig.~\ref{fig:stability}), but finiteness does not imply accuracy. Its true-energy drift increases from $0.40$ at $100$ steps to $7.1$ at $2000$ ($710\%$ error) because Prop.~\ref{prop:stab} guarantees near-conservation for a modified learned Hamiltonian, not the physical energy. The one-step MLP exceeds the cap in the median episode by step $500$, and $61/192$ trajectories become non-finite between steps $1000$ and $2000$. Adding an energy penalty ($\lambda=1$) does not help: $123/192$ episodes become non-finite. Because non-finite trajectories have undefined (not infinite) error, we report non-finite counts, energy drift, and rollout error separately rather than as a single ratio. Thus, the comparison is finite rollouts with large drift versus divergence, and the cap should not be treated as a measured lower bound.

The neural ODE is the stronger control: at matched step size, objective, capacity, and learning-rate sweep, every trajectory stays finite and its $100$-step drift ($0.08$) is five times smaller than the symplectic model's, yet by $500$ steps it reaches $3\times10^{2}$--$4\times10^{3}$ and exceeds the cap by $1000$. The short horizon favours the neural ODE, the long horizon favours Hamiltonian structure, and since equal step size does not imply equal integration error for RK4 and leapfrog, the comparison is between complete updates rather than integrators in isolation.

The most informative control is the energy-penalised model, whose target is the analytic physical energy: its short-horizon drift is far smaller than the symplectic model's, yet grows without bound by $2000$ steps (with trajectory error following), so penalising the true energy delays drift without bounding it, and the two are always reported together. Two further probes on the factored model sharpen this picture: a short-rollout loss more than halves the $2000$-step drift while preserving inversion, and over $10{,}000$ steps the factored model's drift saturates from $1.9$ toward $0.89$ while the non-factored symplectic model keeps drifting.

Stability is nevertheless all that conservation provides. Trained only on attractive gravity, the same symplectic model does not follow the unseen repulsive law. Its rollout remains closer to the attractive world, with regime match $0$. We turn to that second behavior next.

\subsection{A factored coupling answers a changed law, but cannot make rollouts stable}
\label{sec:r2}

\begin{table}[h]
  \centering
  \caption{Following a never-trained sign flip of the coupling to $G\in\{-0.5,-1.0,-1.5\}$.}
  \label{tab:inversion}
  \begin{tabular}{@{}lcc@{}}
    \toprule
    Model & regime match & rollout error (nMSE) \\
    \midrule
    factored symplectic     & $1$ at every $G$          & $0.03$--$0.18$ \\
    non-factored symplectic & $0$ at every $G$          & ${\approx}2.7$ \\
    \bottomrule
  \end{tabular}
\end{table}

\begin{wrapfigure}{r}{0.50\linewidth}
  \centering
  \vspace{-9pt}
  \includegraphics[width=\linewidth]{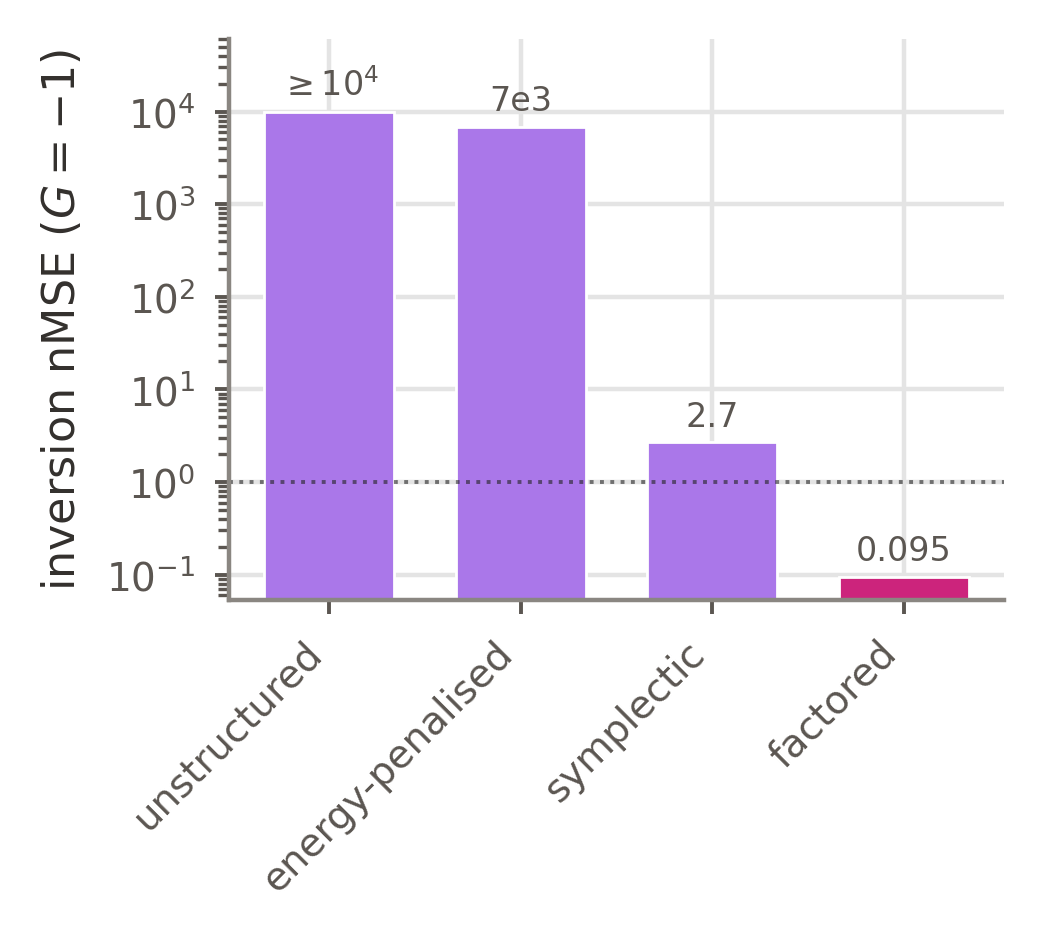}
  \caption{Inversion error at $G=-1$; only the factored model falls below nMSE $1$.}
  \label{fig:inversion}
  \vspace{-10pt}
\end{wrapfigure}

We next ask whether a model trained only on attractive gravity can follow the unseen repulsive law (Table~\ref{tab:inversion}). Regime match measures the fraction of rollouts that follow the repulsive rather than the attractive reference, while rollout error is measured against the true repulsive trajectory. The factored model follows the flipped-law regime at every tested coupling, with regime match $1$ and nMSE $0.03$--$0.18$ for $G\in\{-0.5,-1.0,-1.5\}$. In contrast, the non-factored symplectic model remains bounded but never inverts, with regime match $0$ at every coupling and nMSE near $2.7$. Its rollout remains closer to the attractive reference. At $G=-1$, only the factored model falls below nMSE $1$, by four orders of magnitude (Fig.~\ref{fig:inversion}). The mechanism is Prop.~\ref{prop:invert}: because the force is linear in $G$, the sign reversal is supplied by construction, while the learned interaction shape $\varphi_\theta$ determines the magnitude of the repulsive trajectory. The non-factored model fails not because it lacks capacity, but because its black-box potential imposes no relation between the unseen $G<0$ regime and the observed $G>0$ regime.

The converse also holds: factoring provides the counterfactual but not stability. Placed on a non-conserving base, the factored coupling still inverts the regime but diverges over long rollouts (\S\ref{sec:r3}). On smooth parameter shifts the distinction becomes quantitative rather than binary: every model's one-step error rises smoothly, and the factored model achieves the lowest extrapolation error ($10$ to $30\times$ below the training band) for the coupling it explicitly factors, but gains no advantage for the mass parameter it does not.

\subsection{The dissociation, and the controls that rule out confounds}
\label{sec:r3}

\begin{table}[t]
  \centering
  \caption{The double dissociation, from one matched comparison on oracle three-body gravity.}
  \label{tab:dissociation}
  \setlength{\tabcolsep}{5pt}
  \begin{tabular}{@{}lcccc@{}}
    \toprule
    Model & conserves & factored & stability: drift & inversion: match (nMSE) \\
    \midrule
    unstructured        & no  & no  & ${>}10^{4}$          & no ($0$) \\
    symplectic          & yes & no  & $3.9$                & no ($0$) \\
    factored graph-net  & no  & yes & ${>}10^{4}$          & \textbf{yes} ($1$; $0.35$) \\
    factored symplectic & yes & yes & \textbf{$1.7$}       & \textbf{yes} ($1$; $0.08$) \\
    \bottomrule
  \end{tabular}
\end{table}

The two effects meet in Table~\ref{tab:dissociation}, which crosses energy conservation with coupling factorization in one matched comparison. The stability column reports true-energy drift at $2000$ steps, where bounded models remain below $2$ and divergent models reach the $10^{4}$ display cap. The inversion column reports regime match together with rollout nMSE under the sign flip. The pattern is a clean double dissociation. Stability appears when the update is symplectic, while inversion appears when the coupling is factored, and the two effects vary independently. The factored symplectic model is the only model that both remains stable and answers the changed law; the unstructured model does neither. We next remove the most plausible confounds, namely capacity, integrator choice, and the choice of formalism.

The first concern is that the counterfactual benefit may come from some feature that accompanies the factored potential rather than from factorization itself. We test this by varying the state representation, parameter map, and integrator one at a time using an exact linear diagnostic. Each candidate fits a stiffness matrix by ridge regression to force labels containing a fixed small non-gradient curl component and observation noise. Hamiltonian candidates are projected onto the symmetric cone, while the conditional parameter map is constructed from positive-localised features. The design crosses leapfrog, RK4, and implicit midpoint at three step sizes in two coercive coupled-oscillator systems. Every candidate fits the training regime comparably, so the comparison is not driven by state escape (paired runs, Holm $p\approx0.003$). At matched in-distribution fit, the factored map reaches error $0.0035$ at the unseen coupling, while an unrestricted map trained on the same data reaches $0.985$. The Hamiltonian representation keeps true-energy error at $5$ to $8\times10^{-4}$, compared with $0.31$ to $0.33$ for an unrestricted vector field.

\begin{wrapfigure}{l}{0.50\linewidth}
  \centering
  \vspace{-9pt}
  \includegraphics[width=\linewidth]{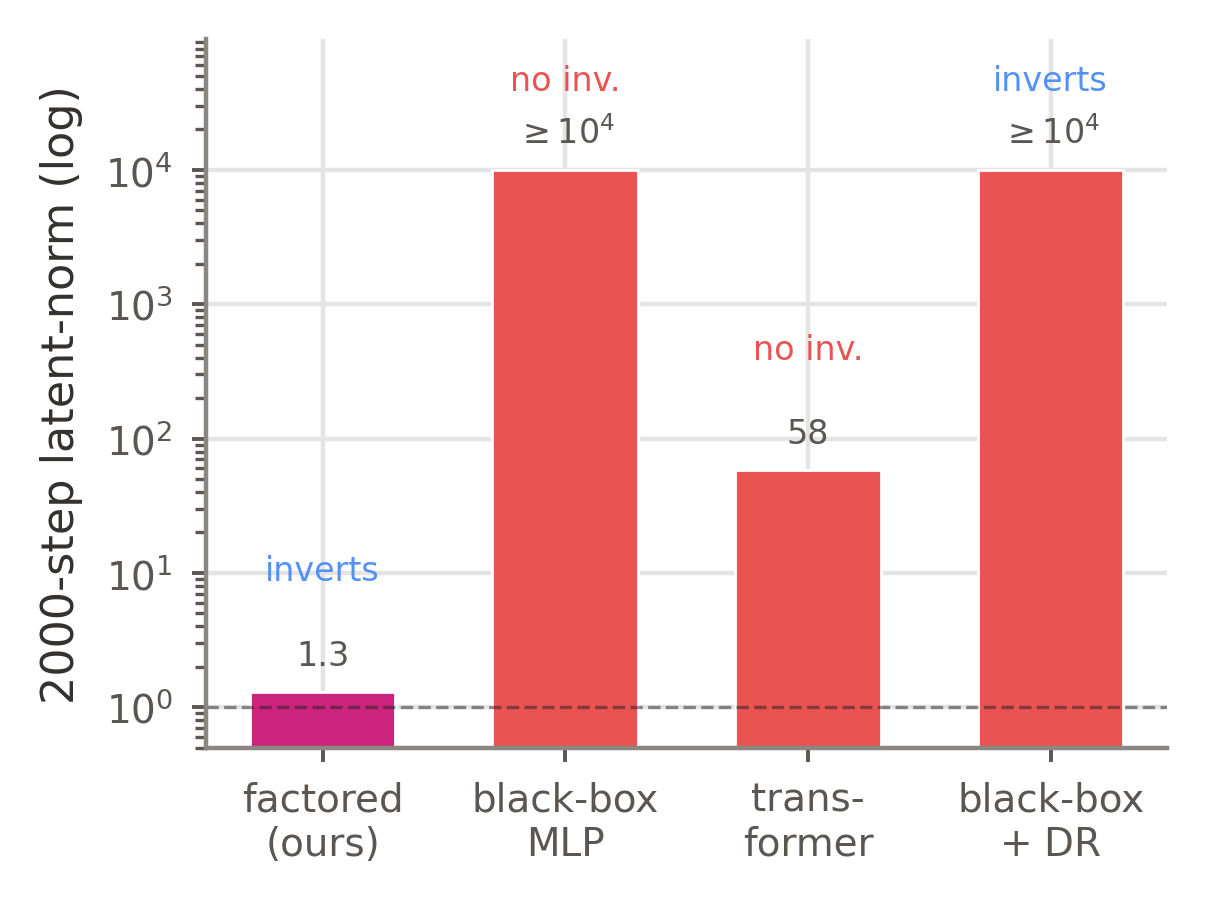}
  \caption{Stronger baselines on the sign flip: $2000$-step latent-norm (stability, log), with each model's inversion outcome marked.}
  \label{fig:baselines}
  \vspace{-10pt}
\end{wrapfigure}

Two controls clarify this result. First, the energy advantage tracks the injected curl: at zero curl both representations drift little ($0.0007$ and $0.002$), but as contamination rises to curl fractions $0.03$, $0.06$, $0.10$ the unrestricted model climbs to $0.35$, $0.94$, $4.30$ while the Hamiltonian stays near $0.0008$, and since no configuration escapes, the factorial isolates fidelity and inversion rather than mere prevention of divergence. Second, the inversion gap reflects the extrapolation prior, not capacity: under the same objective a multiplicative map transfers whether it has eight scalars or $4.5$k parameters, a flexible network receiving the coupling directly only partially extrapolates ($0.071$--$0.075$), and one receiving it through positive-localised features stays near $1.0$ despite the most parameters. The transferable ingredient is the multiplicative structure, supplied rather than identified from positive-support data.

The separation persists across formalisms: on the engine gravity flip both factored models invert, while the three non-factored structured models fail at comparable stability. Capacity is ruled out directly: across a width sweep the unstructured model diverges at every width and the factored model inverts at every width, while adding capacity to the non-factored conserving model only worsens its drift. Graph-network simulators give the complementary control: factoring their messages restores the short-horizon response, but both variants fail the $2000$-step test, so parameter transport alone does not provide long-horizon prediction (App.~\ref{app:B}).

A stronger architecture and a direct data remedy leave the dissociation unchanged (Fig.~\ref{fig:baselines}). A permutation-equivariant set-transformer, used as a stronger relational black box, still diverges over $2000$ steps and fails to invert the unseen sign flip. Domain randomisation trains a black box on both signs of the coupling and therefore recovers inversion because the counterfactual is now part of its training distribution, but it provides no stability and diverges like the ordinary black box. Only the factored symplectic model both remains bounded and inverts, showing that neither attention nor data coverage substitutes for the relevant structural commitment.

\subsection{Where conservation and factoring each stop helping}
\label{sec:r4}

Both structural priors have clear limits. Under linear drag, a port-Hamiltonian model dissipates energy with the correct sign and order of magnitude, at roughly half the true rate ($-0.014/-0.063$ compared with $-0.033/-0.091$). A conserving model cannot dissipate at all. Conservation is therefore the wrong prior once the physical system genuinely loses energy.

A factored law with the wrong parameter dependence gives a subtler failure. A linear-$G$ prior trained on a $G^{2}$ law fits the observed regime well ($0.014\approx0.010$) but predicts a spurious sign flip with error $0.343$. Only the correctly specified $G^{2}$ prior restores the counterfactual ($0.022$). A wrong interaction structure, such as a pairwise prior fitted to a three-body interaction, instead produces poor in-distribution fit ($0.43$), so coupling-power misspecification is the more dangerous case, failing without an in-distribution warning. The ambiguity nonetheless has a constructive resolution: when linear ($G$) and even ($|G|$) couplings fit the positive-support data equally well, positive-only queries reach unique selection only half the time ($0.50$ across four candidate maps), whereas a single fixed query at $G=-1$, using no ground truth, reaches $1.00$, and a non-adaptive query over the full domain is intermediate ($0.775$; App.~\ref{app:G}).

Perception introduces a third boundary: a single frame recovers position and fields but not velocity, so the inferred state falls off the correct energy shell, and the limiting variable is momentum rather than perception as a whole (App.~\ref{app:D}). One further diagnostic sharpens the map. Time-reversal consistency helps exactly when the underlying dynamics are reversible, giving retrace match $1$ under conservative contact and only partial recovery under dissipation, which refines the conclusion of \citet{botev2021priors}. A truth-assisted transport score from the bound in Prop.~\ref{prop:transport-bound} ranks which representation and map roll out worst under the flip (pooled Spearman $0.87$ versus $0.09$ for the initial-state term alone), a model-selection diagnostic rather than a certificate.

\subsection{The dissociation is not an artefact of the headline system}
\label{sec:r5}

\begin{table}[t]
  \centering
  \setlength{\tabcolsep}{10pt}
  \caption{Sign-flip inversion error (nMSE at $G<0$) across state groundings.}
  \label{tab:grounding}
  \begin{tabular}{@{}lcc@{}}
    \toprule
    State grounding & factored & unstructured twin \\
    \midrule
    oracle state                       & $0.067$          & --- \\
    pixel latent $+$ canonical kinetic & $0.065$--$0.092$ & $0.47$ / div. \\
    grounded pixel model               & $0.12$           & --- \\
    decode-free (JEPA) $+$ centroid anchor    & $0.40$           & $0.68$ \\
    learned object binder              & $0.15$           & never \\
    frozen VideoMAE                    & $0.13$           & $0.65$ \\
    frozen V-JEPA2                     & $0.59$           & $0.82$ \\
    DINOv2 $+$ frame differencing      & $0.096$          & $0.62$ \\
    \bottomrule
  \end{tabular}
\end{table}

The same two mechanisms extend well beyond smooth three-body gravity with oracle state. In generalized coordinates, a Hamiltonian model with a learned mass metric on a MuJoCo double pendulum keeps its energy drift near $0.1$ while unstructured baselines diverge, and the gravity flip still inverts only when the coupling is factored, with regime match $1$ versus $0$. Under contact, the symplectic model remains bounded as collisions become stiffer while the unstructured model diverges (App.~\ref{app:C}). The factored coupling also transfers across object count and force family. A model trained on three bodies still inverts with thirty-two; it transfers to Coulomb charges under a held-out sign combination, with error $0.52$ compared with $1.42$ for a black box; and a single factored model trained on a mixed gravity-plus-Coulomb generator inverts both laws about $120\times$ more accurately than a black box receiving the same couplings.

The inversion comes from the factored channel itself rather than from additional data or an unrestricted shortcut. Increasing the training set does not improve the non-factored model and slightly worsens it: its inversion error rises from nMSE $0.76$ with $32$ trajectories to $2.67$ with $512$, because additional attractive trajectories reinforce a continuation that becomes incorrect after the sign flip. The structural channel is also necessary rather than incidental. When a free black-box shortcut is added alongside the factored pathway, it absorbs about $42\%$ of the force and degrades inversion. Penalising the shortcut reduces its contribution to $10^{-4}$ and restores fidelity, improving nMSE from $0.43$ to $0.09$ (App.~\ref{app:E}).

Because neither behavior should rely on a hand-built state, we also ground the counterfactual in pixels (Table~\ref{tab:grounding}). A latent learned from images recovers inversion once a fixed canonical kinetic form is imposed, reaching near oracle-state error. A decode-free objective inverts once a label-free image centroid fixes the latent gauge. A learned object binder over a fixed detector restores inversion after removing object identity had broken it, succeeding on $10/10$ seeds versus $0/10$ for the unstructured twin (exact binomial $p\approx10^{-3}$). The same structure transfers to frozen public video encoders we did not train (VideoMAE, V-JEPA2), where the factored readout inverts more reliably than an unstructured readout.

Encoder scale alone does not recover the physical readout: in V-JEPA2 from $\sim300$M to $1$B parameters, inversion error is essentially flat ($1.02\times$) and the dissociation persists, so perception-side scale does not supply the counterfactual. The split appears within a single model: a readout recovers position and momentum from frozen V-JEPA2 tokens (nMSE $0.28$--$0.43$), yet its action-conditioned predictor under no-op actions is no better than repeating the final frame. The missing ingredient is momentum: per-frame encoders capture position but not velocity, and inversion returns only when momentum is estimated by differencing features across frames, reducing error from $0.62$ to $0.096$ and matching video encoders. Thus the momentum boundary of \S\ref{sec:r4} reappears at the representation level (App.~\ref{app:H}).

The structural account also extends beyond conservative dynamics. In a confined thermal bath, a positive-semidefinite friction port with fluctuation–dissipation noise reaches the correct stationary temperature, while a conserving model does not thermalize. The same passive port cannot represent anti-damping (Prop.~\ref{prop:tension}), so energy injection must enter via a separate factored drive. That drive recovers the driven steady state across trained and held-out strengths, including an unseen cooling sign (Prop.~\ref{prop:driven}). With video observations, a grounded energy-balance estimator recovers the damping rate; on time-reversed footage it yields a negative rate that the sign-constrained port cannot represent, giving an arrow-of-time measure paired with an architectural refusal (App.~\ref{app:F}, \ref{app:H}).

\section{Conclusion and Limitations}
\label{sec:conclusion}
\label{sec:limits}

In this work we investigated which structural commitments make a learned world model obey physical laws, evaluating each through long-horizon stability and counterfactual generalisation. First, energy conservation through a symplectic update provides long-horizon stability where equal-capacity predictors, tuned neural ordinary differential equations, and fixed energy penalties fail, yet it does not improve counterfactual generalisation. Second, making the interaction force linear in the physical coupling enables extrapolation to never-seen signs of that coupling, which an unrestricted parameter map fails to recover, yet it does not improve stability. Third, we map the boundary of each mechanism and identify where its structural prior ceases to help. Fourth, we extend the same account beyond conservative dynamics, using a passive friction port that thermalises correctly and a factored drive that generalises across unseen forcing strengths and signs. The broader lesson is that physical structure does not make a model uniformly more physical. Specific commitments produce specific behaviours, allowing a modeller to combine structural components according to the capabilities a task requires.

However some limitations remains. Conservation becomes the wrong prior when the true system genuinely dissipates energy, while factoring helps only when the intervened parameter follows the supplied functional form. A misspecified coupling power can therefore fit the training regime well while failing under intervention. A fixed kinetic term and a label-free anchor reduce the ambiguity, but joint learned detection and binding remain open. Finally, all evidence in this study comes from synthetic simulators and rendered observations rather than real video, the general deployment guarantee of our work is the next step we aim to ensure.

\section*{AI use statement}
A large language model assisted the authors in polishing the language of the manuscript. The authors verified that all claims, proofs, mathematical formulations, and reported values are valid, checked them against the implementation and results, and take full responsibility for the manuscript.

\section*{Ethics statement}
This work uses only synthetic physics simulations, so no human subjects,
personal data, or scraped datasets are involved. We do not foresee direct
harmful applications of the findings, which concern the architectural conditions
for physical generalisation in learned simulators. The main integrity-relevant
practices, namely prospective committed protocols with metrics fixed in advance,
equal-capacity baselines, and the reporting of negative and boundary results, are
described in the paper and its record appendices.

\section*{Reproducibility statement}
App.~\ref{app:formal} specifies all systems, models, integrators, metrics, and the
perception pipeline. App.~\ref{app:theory} gives the derivations and proofs.
App.~\ref{app:results} gives the per-experiment protocols, seeds, and per-seed results.
App.~\ref{app:prereg} gives the pre-registration protocol (metrics fixed in advance, the
equal-capacity rule, and non-finite capping), the statistics conventions, and the compute
and hyperparameter defaults (Table~\ref{tab:compute}). Our implementation and the
pre-registration documents are provided as supplementary material.

\bibliography{refs}
\bibliographystyle{iclr2027_conference}

\appendix

\section{Formal setup and definitions}
\label{app:formal}

A system of $N$ bodies has state $x=(q,p)$, with
positions $q\in\mathbb{R}^{Nd}$ and conjugate momenta $p\in\mathbb{R}^{Nd}$ ($d$
is the spatial dimension, equal to $2$ except for the pendulum, where $q$ are joint
angles). A parameter vector $c=(G,\{m_i\})$ collects the coupling constant and
masses. A dynamics model is a map $f_\theta:(x_t,c)\mapsto x_{t+1}$, iterated to
a rollout $x_{0:T}$. Ground truth comes from a reference integrator $\Phi$.
Training minimises one-step squared error on forward-law data only. The $k$-step
variant instead sums this error over a $k$-step rollout.

The evaluated systems are as follows. (i) \textbf{Conservative gravity} (oracle state):
$H=\sum_i\lVert p_i\rVert^2/2m_i+V(q)$ with softened-core potential
\begin{equation}
  V(q)=-\sum_{i<j}\frac{G\,m_im_j}{\sqrt{\lVert q_i-q_j\rVert^2+\epsilon^2}},
  \qquad \epsilon=0.1,\; dt=0.002 .
  \label{eq:grav}
\end{equation}
The softening keeps the ground-truth integrator near-conservative (drift
  ${\sim}0.01$; $\epsilon=10^{-2}$ increases the reference drift by
$110\times$). This softening and the associated reference-integrator drift are reported with
all gravity results. (ii) \textbf{Coulomb}: the same
form with $Gm_im_j$ replaced by a charge product $k\,z_iz_j$, $z_i\in\{\pm1\}$.
(iii) \textbf{Springs-with-cutoff} (necessity screen). (iv) \textbf{Soft-sphere
disks} on a periodic torus (minimum-image), contact stiffness series set by the
stiffness $k\in\{30,300,\infty\}$. (v) \textbf{Dissipative gravity}:
Eq.~\eqref{eq:grav} plus linear drag, $\dot p_i=-\partial V/\partial
q_i-\gamma\,p_i/m_i$. (vi) \textbf{Double pendulum} (MuJoCo): joint angles $q$,
configuration-dependent $M(q)$, generalised momentum $p=M(q)\dot q$, RK4 engine
ground truth; conservative and joint-damped variants.

Parameter counts are approximately matched in the five-model gravity comparison.
Other comparisons use the experiment-specific budgets reported below. Equal
parameter count does not imply equal functional capacity.
The unstructured MLP is a residual MLP, $x_{t+1}=x_t+g_\theta(x_t,c)$.
The energy-penalised model is the unstructured MLP plus an energy-drift penalty
$\lambda\,\lvert E_\theta(x_{t+1})-E_\theta(x_t)\rvert$.
The interpretation of this baseline depends on how $E_\theta$ is obtained:
a jointly learned unconstrained energy can become constant and make the
penalty vacuous. Its failure alone therefore cannot establish the inadequacy
of a penalty evaluated using a fixed physical energy.
The penalty energy is \emph{fixed and analytic}, not learned or pretrained: the true
physical energy $E=\sum_i\lVert p_i\rVert^2/2m_i - G\sum_{i<j}m_im_j/\sqrt{r_{ij}^2+\epsilon^2}$
with the true masses and $G$ and the same softening $\epsilon=0.1$ as the generator,
applied as $\lambda\,(E(\hat x_{t+1})-E(x_t))^2$
on the de-standardised prediction. The penalty does not by itself exclude the
constant-energy solution (a wrong trajectory that holds $E$ near its initial value), which
is why rollout error is always co-reported with energy drift. In the raw record the
$\lambda=100$ arm does not take that route: its energy error and its rollout error grow
together (drift $115$--${>}10^{4}$ with normalised rollout error $82$--${>}10^{4}$ at
$2000$ steps), after an early phase in which its energy error is below the symplectic
model's (Table~\ref{tab:f1}).
The symplectic model is a symplectic HNN, $H_\theta=T_\theta(p)+V_\theta(q,c)$, one leapfrog
step.
The factored symplectic model is the factored HNN, potential Eq.~\eqref{eq:factored}; the canonical
\emph{fixed-kinetic} variant sets $T=\sum_i\lVert p_i\rVert^2/2m_i$.
DeLaN \citep{lutter2019delan}: Lagrangian $L_\theta=\tfrac12\dot
q^\top M_\theta(q)\dot q-V(q,c)$, Euler--Lagrange accelerations, factored/
non-factored $V$.
SymODEN \citep{zhong2020symoden}: control-Hamiltonian, learned
$M_\theta(q)^{-1}$ and $V_\theta$, RK4.
Generalised-coordinate HNN: $H=\tfrac12 p^\top M_\theta(q)^{-1}p+g\,U_\theta(q)$,
Cholesky-SPD metric, factored potential.
Port-Hamiltonian: conservative flow plus a learned dissipative port
$-\gamma\,R_\theta(p)$.
The graph-network simulator and its factored variant \citep{sanchezgonzalez2020gns,battaglia2016interaction}:
an interaction network; the factored form multiplies the aggregated message by $G$.
Hybrid (necessity): factored pathway plus a free black-box shortcut with an
$L_2$ penalty $\lambda$ on the shortcut force.

Separable Hamiltonians use symplectic leapfrog. The
non-separable generalised-coordinate Hamiltonian uses Tao's explicit symplectic
scheme \citep{tao2016explicit}, admitted only after a calibration gate (a known
Hamiltonian must stay within $10^{-4}$ relative energy drift over 2000 steps).

The relative energy drift is
\begin{equation}
  D(t)=\frac{\lvert E(x_t)-E(x_0)\rvert}{\lvert E(x_0)\rvert},
\end{equation}
with $E$ the true energy of the predicted state; regime-match and the affine
identifiability residual are Def.~\ref{def:metrics}. Energy-shed rate
(dissipation) and time-reversal match (retrace a reversed-momentum end state)
are defined analogously.

For the stochastic dissipative results (\S\ref{sec:r5}, App.~\ref{app:F}),
we separate mechanical friction from a closed-system GENERIC formulation.
For canonical kinetic energy and momentum-space friction, the continuous model is
\begin{equation}
 dq=\nabla_p H\,dt,\qquad
 dp=\big[-\nabla_qH-M\nabla_pH+f_{\rm ext}\big]dt+\Sigma\,dW,
 \qquad \Sigma\Sigma^\top=2k_BT\,M,\quad M\succeq0 .
 \label{eq:generic}
\end{equation}
The deterministic friction contribution satisfies
$\dot H_{\rm fric}=-(\nabla_pH)^\top M\nabla_pH\le0$.
This is a passive mechanical channel, not a claim that subsystem energy
cannot rise through work or thermal fluctuations. State-dependent
mobilities require a separately specified stochastic drift and convention;
the isolated OU result below concerns constant coefficients.
A closed-system GENERIC model additionally requires degeneracy conditions
$L\nabla S=0$ and $M\nabla E=0$, beyond antisymmetry and positive semidefiniteness
\citep{lee2021gfinn}. Those closed-system laws are not established by the
passive-channel tests here. The
consequences of the imposed relation are reported. (vii)
\textbf{Confined Langevin $N$-body}: softened gravity plus a harmonic trap and an
exact Ornstein--Uhlenbeck thermostat at bath temperature $T$; the diagnostic
observable is the kinetic temperature $\hat T=\langle p^2/m\rangle$ (equal to $T$
at equilibrium by equipartition, $k_B{=}1$). (viii) \textbf{Damped double pendulum}
with linear and nonlinear ($v^2$) joint friction (the engine of \S\ref{sec:r5}).
(ix) \textbf{Driven-dissipative particles}: system (vii) plus a linear active drive
of strength $A$ opposing friction $\gamma$, reaching a \emph{non-equilibrium steady
state} at $T_{\text{eff}}=T\gamma/(\gamma-A)$ (App.~\ref{app:theory}). Models add a
dissipation channel to the symplectic backbone: a diagonal positive-semidefinite
friction with fluctuation--dissipation noise (the \emph{guarantee} side, ``$M\succeq
0$''); the full PSD friction operator $M=LL^\top$ (Cholesky) with correlated noise;
a \emph{factored} friction $M=\gamma M_0$ whose sign flips with $\gamma$ and
permits inversion with respect to $\gamma$; and, for system (ix), a factored drive $A\,w_\theta$ with
net damping $\gamma_f-A\,w$ and fluctuation--dissipation noise tied to $\gamma_f$.
Primary metric is the temperature relative error $|\hat T-T^\ast|/T^\ast$ (against
  the correct equilibrium or steady-state temperature $T^\ast$, the true simulator's
measured value which targets agreement with the discrete reference rather than cancellation of model-dependent integration bias); secondary is the energy-shed rate and
the injection rate at $\gamma<0$.

A scene is rendered to synthetic frames by a fixed
differentiable renderer $R$ (Gaussian or coloured-disk splats at body positions);
an encoder $E_\phi$ maps a frame stack to a latent $(\hat q,\hat p)$. In the
supervised setting, $E_\phi$ is trained against the oracle state. In the
discovered setting, it is trained only through future-frame prediction,
$\lVert R(\text{rollout of }E_\phi(\cdot))-\text{frames}\rVert^2$, without
oracle-state supervision. The symplectic-consistency prior adds $\lVert\hat
p/m-v_{\text{obs}}\rVert^2$, $v_{\text{obs}}$ a finite difference of soft-argmax
centroids across frames (an observed velocity, not oracle state). In the
end-to-end pixel comparison every model shares $E_\phi$ and $R$; a separate
branch removes $R$ to test unsupervised discovery.

\section{Derivations}
\label{app:theory}

\begin{proof}[Proof of Proposition~\ref{prop:invert}]
  Write $f_i(q,G)=-\partial V/\partial q_i=-G\,\partial_{q_i}U_\theta(q)$.
  $V$ is linear in $G$ with $G$-independent coefficient $\sum_{i<j}\varphi_\theta$,
  so $\partial V/\partial q_i$ is linear in $G$ and $f_i(q,-G)=-f_i(q,G)$. The
  Hamiltonian vector field's momentum component $\dot p=f$ therefore changes sign
  with $G$, whereas $\dot q=\partial T/\partial p$ is unchanged. The resulting
  trajectory is the flow generated by the same learned potential under coupling
  $-G$.
\end{proof}
\noindent\emph{Comparison with the non-factored model.} A black-box conditional potential
$V_\theta(q,G)$ has $\partial V_\theta/\partial q$ with no constraint linking
$G>0$ to $G<0$, so $f_\theta(q,-G)$ is unconstrained by $G>0$ training and
inversion is not guaranteed; empirically it fails (App.~\ref{app:B}). Thus,
the dissociation isolates \emph{factoring}, rather than conservation, as the
inversion ingredient.

\noindent\emph{Proposition~\ref{prop:stab}, formal statement.} Let $H_\theta$ be
analytic on a neighbourhood of a compact set that contains the numerical orbit, and let
the step $dt$ be below the threshold of backward-error analysis for that set. Then the
leapfrog map applied to \eqref{eq:ham} is the exact time-$dt$ flow of a modified
Hamiltonian $\widetilde H_\theta=H_\theta+O(dt^2)$ up to a remainder of size
$O(e^{-c/dt})$ per step, so $\widetilde H_\theta$ is conserved up to $O(e^{-c/dt})$ over
times of order $e^{c/dt}$ and the numerical $H_\theta$ stays within $O(dt^2)$ of its
initial value over that interval \citep{hairer2006geometric}. The statement is about
$\widetilde H_\theta$ and the learned $H_\theta$; it says nothing about the physical energy
unless $H_\theta$ equals it.
\noindent\emph{Empirical regularity (not a theorem).} A generic explicit map
$x_{t+1}=x_t+g_\theta(x_t)$ has no corresponding conservation guarantee. Its
energy behaviour therefore depends on the learned map; a contraction, for
example, may remain bounded. In these experiments, the unstructured and
graph-network models exhibit secular drift and reach the reporting cap. The
proposition supplies a sufficient condition for long-time near-conservation,
whereas divergence of the unstructured models is an empirical observation.
\noindent\emph{Remark.} Training-noise injection reduces the per-step prediction
error but does not itself provide an energy-boundedness guarantee. In the graph-network simulator it
delays, but does not prevent, the observed drift (App.~\ref{app:B}). The learned
$\widetilde H_\theta\neq H_{\text{true}}$, which
is why the symplectic model is \emph{bounded but not flat}; multi-step training and
longer horizons shrink and saturate the gap.

\begin{proposition}[Finite-time transport bound]
  \label{prop:transport-bound}
  Let $F^\star$ be the reference field, $F_\theta^{\rm true}$ the learned
  content evaluated with the true parameter map, and $F_\theta^{\rm decl}$
  the deployed field using the declared map, all at the tested coupling.
  Suppose the reference and deployed trajectories remain in a compact region
  $\Omega$ on $[0,T]$, and $F_\theta^{\rm decl}$ is Lipschitz there with
  constant $L\ge0$. Define
  \[
   \epsilon_c=\sup_\Omega\|F_\theta^{\rm true}-F^\star\|,\qquad
   \epsilon_m=\sup_\Omega\|F_\theta^{\rm decl}-F_\theta^{\rm true}\|.
  \]
  Then
  \[
   \|\hat x_t-x_t\|
   \le e^{Lt}\big(\|\hat x_0-x_0\|+t(\epsilon_c+\epsilon_m)\big),
   \qquad 0\le t\le T .
  \]
\end{proposition}
\begin{proof}
  Adding and subtracting the deployed field at the reference state bounds the
  trajectory-error growth by $L\|\hat x-x\|+\epsilon_c+\epsilon_m$.
  Integration and Gronwall's inequality give the displayed bound.
\end{proof}
The auxiliary true-map field and both suprema are truth-assisted quantities.
They are not generally available at deployment. The bound assumes trajectories
remain in $\Omega$ and therefore does not establish global stability.
The reported pooled Spearman association ($0.87$ versus $0.09$ using only
initial-state error) describes a ranking on the evaluated models
(\S\ref{sec:r4}). It is not coverage calibration or a certified prediction
of error on new model families.

\begin{proposition}[Passive friction and sign-reversed anti-damping]
  \label{prop:tension}
  Consider the deterministic momentum-friction channel
  $\dot p|_{\rm fric}=-M\nabla_pH$ with $M\succeq0$.
  Its contribution to $\dot H$ is nonpositive for every represented state.
  If $M=\gamma M_0$, with $M_0\succ0$, changing $\gamma$ to a negative
  value makes this contribution positive whenever $\nabla_pH\ne0$.
  Thus a globally PSD parameterisation cannot represent that sign-reversed
  channel. For a bath with $T>0$, negative $\gamma$ also prevents
  $\Sigma\Sigma^\top=2k_BT\gamma M_0$ from being a real noise covariance.
\end{proposition}
\begin{proof}
  The chain rule gives
  $\dot H_{\rm fric}=-(\nabla_pH)^\top M\nabla_pH$.
  Its sign follows from positive semidefiniteness; for negative $\gamma$
  and nonzero velocity the sign reverses. A covariance is positive
  semidefinite, whereas $2k_BT\gamma M_0$ is negative definite.
\end{proof}
This restriction concerns the same passive channel. It neither forbids
external work or thermal energy injection nor makes every anti-damped
stochastic model ill-posed. It also does not establish the invariant law
of a coupled numerical simulator (App.~\ref{app:F}).

\begin{proposition}[Effective stationary temperature for linear drive]
  \label{prop:driven}
  Let $H(q,p)=U(q)+\|p\|^2/(2m)$, $m>0$, $\gamma>0$, $T>0$, and
  $A<\gamma$. Consider the continuous underdamped dynamics
  \[
   dq=(p/m)\,dt,\qquad
   dp=-\nabla U(q)\,dt-(\gamma-A)(p/m)\,dt+
       \sqrt{2\gamma k_BT}\,dW .
  \]
  If the density is normalizable and boundary terms vanish, then
  $\rho_\infty(q,p)\propto\exp[-H(q,p)/(k_BT_{\rm eff})]$
  is stationary, where $T_{\rm eff}=T\gamma/(\gamma-A)$.
  Each momentum component has variance $m k_BT_{\rm eff}$.
\end{proposition}
\begin{proof}
  Hamiltonian transport annihilates any differentiable density depending
  only on $H$. The momentum probability current from friction and noise is
  \[
    (\gamma-A)(p/m)\rho_\infty+
        \gamma k_BT\nabla_p\rho_\infty
    =\big((\gamma-A)-\gamma T/T_{\rm eff}\big)(p/m)\rho_\infty=0 .
  \]
  Thus the Fokker--Planck operator annihilates the stated density.
  Its momentum marginal is Gaussian with the given variance.
\end{proof}
The isolated OU momentum process has no stationary probability law when
$A\ge\gamma$ and $T>0$. The proposition makes no corresponding
unqualified assertion for arbitrary coupled nonlinear models, and proves
neither uniqueness nor exact preservation by a numerical integrator.
The externally driven bath interpretation is distinct from the mathematical
effective Gibbs form. A factored drive supplies the coefficient dependence;
predictive accuracy still depends on the remaining learned dynamics.

\begin{definition}[Regime-match and affine recovery]
  \label{def:metrics}
  \emph{Regime-match.} Let $\hat x$ be the model's rollout under the flipped law,
  and $x^{-},x^{+}$ the ground-truth flipped-law and training-law (default) rollouts, both from the
  \emph{true} initial state (so initialisation error is counted, not cancelled).
  The regime is matched iff $\lVert\hat x-x^{-}\rVert<\lVert\hat x-x^{+}\rVert$; we
  report the matched fraction and the normalised MSE $\lVert\hat x-x^{-}\rVert^2/
  \mathrm{Var}(x^{-})$, for nonzero reference variance.
  Ties do not count as matches. A non-finite rollout is reported as a failure
  rather than assigned a nearest-reference label.
  Two conventions are made explicit here. First, the distance is the batch-pooled
  normalised MSE over the whole rollout, which gives one regime-match per run rather than
  per episode. Second, the sign convention is fixed: the correct reference is rolled at the
  tested coupling $G_{\rm test}$ and the default reference at $|G_{\rm test}|$, and the
  regime is scored as $\mathbb{1}[\mathrm{nMSE}(\hat x,x^{-})<\mathrm{nMSE}(\hat x,x^{+})]$.
  \emph{Affine recovery residual} of a discovered latent $\hat z$ with respect to
  the true $(q,p)$: $\min_{A,b}\lVert A\hat z+b-(q,p)\rVert^2/\mathrm{Var}(q,p)$. This
  measures linear recoverability under a fitted affine map, not uniqueness
  of the latent coordinates. Values below $0.5$ are a descriptive threshold
  used in this study, not an identifiability theorem.
\end{definition}

\section{Extended results and findings}
\label{app:results}

This appendix reports the findings summarised in the main text,
together with their figures, multi-seed estimates, experiment-specific methods,
and stated limitations. The unified notation of App.~\ref{app:formal} is used:
state $x=(q,p)$, coupling $G$,
relative energy drift $D(t)$ (Eq.~\ref{eq:drift}), regime-match and affine
identifiability residual (Def.~\ref{def:metrics}). Each finding is tagged
\emph{[well-supported]}, \emph{[directional]}, \emph{[supported]}, or
\emph{[boundary]} (criteria in App.~\ref{app:evidence}).

Aggregates are the mean over at least three seeds, and at least five seeds at
scale, with the available aggregate summaries in App.~\ref{app:results}. Where a
bracket is labelled a normal-approximation 95\% interval, it retains that meaning
even for small seed counts and is not an observed range. Such small-sample
intervals are descriptive and do not justify significance from non-overlap.
Individual seed values are not tabulated for every comparison.

Display-clipped finite errors and non-finite numerical failures are distinct. A
sentinel substituted for a failure is not a measured error and cannot establish a
quantitative error ratio. These figures retain their plotting thresholds, and
the text distinguishes what the two mean. A pre-registered absolute-error
\emph{backstop} marks a perception result invalid when even its in-distribution
fidelity is worse than the threshold (such runs are reported as invalid, not
compared). Regime-match is meaningful only for a \emph{bounded} rollout. For a
diverged rollout, such as that of the graph network, comparison with the nearest
reference degenerates and is not used as evidence. Such cases are identified
explicitly. Table~\ref{tab:overview} summarises all the findings; in it, ``match'' is the
regime match (1 = follows the unseen flipped law), ``nMSE'' the normalised inversion error,
and ``drift'' the relative true-energy drift of Eq.~\ref{eq:drift}.

{\small
\begin{longtable}{@{}p{0.17\linewidth}p{0.08\linewidth}p{0.52\linewidth}p{0.12\linewidth}@{}}
  \caption{Summary of findings and evidence classes.}\label{tab:overview}\\
  \toprule
  Finding & System & Summary statistic & Class \\
  \midrule
  \endfirsthead
  \multicolumn{4}{c}{\tablename~\thetable\ (continued)}\\
  \toprule
  Finding & System & Summary statistic & Class \\
  \midrule
  \endhead
  \midrule
  \multicolumn{4}{r}{\emph{continued on the next page}}\\
  \endfoot
  \bottomrule
  \endlastfoot
      Stability & gravity & symplectic model drift $0.40\!\to\!7.1$ vs cap $10^4$ & well-supported \\
      Tuning-robust & gravity & every unstructured arm passes the cap; $\lambda{=}100$ delays but does not bound drift & well-supported \\
      Multi-step & gravity & factored symplectic model 2000-step drift $4.4\!\to\!1.9$ & directional \\
      10k-step & gravity & factored symplectic model drift saturates $1.9\!\to\!0.89$ & well-supported \\
      Non-factored & gravity & symplectic model match $0$ (no inversion) & boundary \\
      Factored & gravity & factored symplectic model match $1$, nMSE $0.03$--$0.18$ & well-supported \\
      Param.\ extrap. & gravity & smooth: $2\times$ edge, no abrupt degradation & directional \\
      DeLaN/SymODEN & engine & factored match $1$ vs non-factored $0$ & well-supported \\
      Graph-net & gravity & reported response labels $1$ vs $0$; both fail long-horizon test & boundary \\
      Contacts & disks & bounded stiffness series; sym-inversion $1$ & directional \\
      Engine & pendulum & factored symplectic model match $1$, nMSE $8.7e{-}4$ & well-supported \\
      Contact aug. & disks & no better than smooth model & boundary \\
      Dissipation & drag & conservation sheds $0$; port sheds & boundary \\
      Supervised perc. & pixels & inversion retained; error increases with speed & directional \\
      Discovered perc. & pixels & fixed-kinetic nMSE $0.075\!\approx\!0.067$ & well-supported \\
      Momentum recov. & pixels & $p$-resid $0.72\!\to\!0.35$; pos.\ open & directional \\
      Pixel comparison & pixels & factored fixed-kinetic model nMSE $0.12$; others fail & well-supported \\
      Necessity & gravity & shortcut $42\%\!\to\!10^{-4}$; nMSE $0.43\!\to\!0.09$ & directional \\
      Comp.\ \& scale & gravity & match $1$ to $N{=}32$; drift $0.2\!\to\!102$ & directional \\
      Coulomb, held-out sign & Coulomb & factored $0.52$ vs black-box $1.42$; $2.7\times$ & well-supported \\
      Mixed two-family inversion & Coulomb & factored $0.0035$ vs black-box $0.425$; $120\times$ & well-supported \\
      Data-efficiency & gravity & inverts with 32 traj.; symplectic model error rises with data & well-supported \\
      Capacity & gravity & unstructured MLP diverges at every width & well-supported \\
      Composition & drag & three ingredients compose & directional \\
      Control & gravity & structured plans transfer; unstructured MLP misestimates & directional \\
      Rendering & pixels & structured video consistent; unstructured jumps & supported \\
      Noise screen & gravity & robust to $30\%$ state noise & supported \\
      Temperature & Langevin & structured $\hat T$ has low error; unstruct.\ diverges & well-supported \\
      Tension & Langevin & passivity precludes channel anti-damping; factored form permits it & well-supported \\
      Full operator & Langevin & full $M$ has higher error than diagonal $M$ & supported \\
      Engine dissipation & pendulum & port sheds energy, correct sign, ${\sim}$half rate & directional \\
      Noise robustness & Langevin & correct $\hat T$ to ${\sim}10\%$ meas.\ noise & directional \\
      Driven drive & driven & factored drive inverts strength+sign, err $\le0.09$ & well-supported \\
      Misspecification & gravity & low in-dist.\ error, incorrect counterfactual & well-supported \\
      Transformer/dom-rand & gravity & transformer diverges and does not invert & well-supported \\
      Grounded JEPA & pixels & anchored factored hybrid match $1$, nMSE $0.40$ vs $0.68$ & well-supported \\
      Anonymous object & pixels & identity removed: both collapse to match $0.33$ & boundary \\
      Learned binding & pixels & learned tracker $+$ factored inverts ($10/10$ seeds, nMSE $0.15$; control $0/10$) & well-supported \\
      Detection$+$binding & pixels & slot-attention does not reach it (match $0$) & boundary \\
      External encoder & pixels & frozen VideoMAE/V-JEPA2 invert ($0.13$/$0.59$, both 10-seed, factored $10/10$ $p{=}0.001$; readout-matched $2.6\times$ over our CNN; scale-flat $300$M$\to$$1$B, gain $1.02\times$) & well-supported \\
      Momentum axis & pixels & differencing DINOv2 tokens closes the image-SSL gap ($0.62\!\to\!0.096$) & well-supported \\
      Encoder composition & pixels & VideoMAE $+$ learned detect/bind $+$ prior does not compose (unstable, no clean inversion) & boundary \\
      Second law from video & pixels & reported damping diagnostic orders damped $1.25\!>\!$ cons $0.84$ ($5/5$ seeds); passive friction by construction & supported (ordering only) \\
      Floor located $+$ refusal & pixels & oracle control rules out grounding noise as the sole explanation; training-level refusal directionally consistent $5/5$ but below the pre-registered bar & boundary \\
      Energy-aux unification & pixels & the auxiliary-loss lever fails all three pre-registered predictions; with a learned meter it backfires ($\sigma_{\text{damped}}$ $1.25\!\to\!0.62$) & pre-registered negative \\
      Measure, don't fit & pixels & the energy-balance estimator removes the residual bias ($\hat\gamma_{\text{cons}}$ $0.84\!\to\!-0.006$); arrow-of-time rate read with $\sigma\!\ge\!0$ refusal; composition instability cured, grounding boundary stands & well-supported (one gate marginal) \\
\end{longtable}
}

\subsection{A. Stability from conservation (extends \S\ref{sec:r1})}
\label{app:A}

\textbf{Conservation $\to$ stability.} Five approximately
capacity-matched models (unstructured MLP, energy-penalised MLPs, and neural ODE
at ${\sim}139$k parameters; symplectic model at ${\sim}138$k) predict one step of
conservative gravity and roll to $T=2000$ ($100\times$ the 20-step training
horizon). The unstructured MLP and the $\lambda=1$ penalised MLP pass the $10^{4}$
cap in the median episode by step $500$ and lose $61$ and $123$ of $192$ episodes
to non-finite states between steps $1000$ and $2000$. The $\lambda=100$ penalised
MLP and the neural ODE stay finite in every episode but pass the cap. The
symplectic model stays finite and below the cap in every episode
(Table~\ref{tab:f1}). It remains bounded but has non-zero drift because
near-conservation concerns a modified learned Hamiltonian under the assumptions of
Prop.~\ref{prop:stab}, not the true energy. Multi-step and longer-horizon training
reduce this discrepancy.
\begin{figure}[H]
  \centering
  \includegraphics[width=0.62\linewidth]{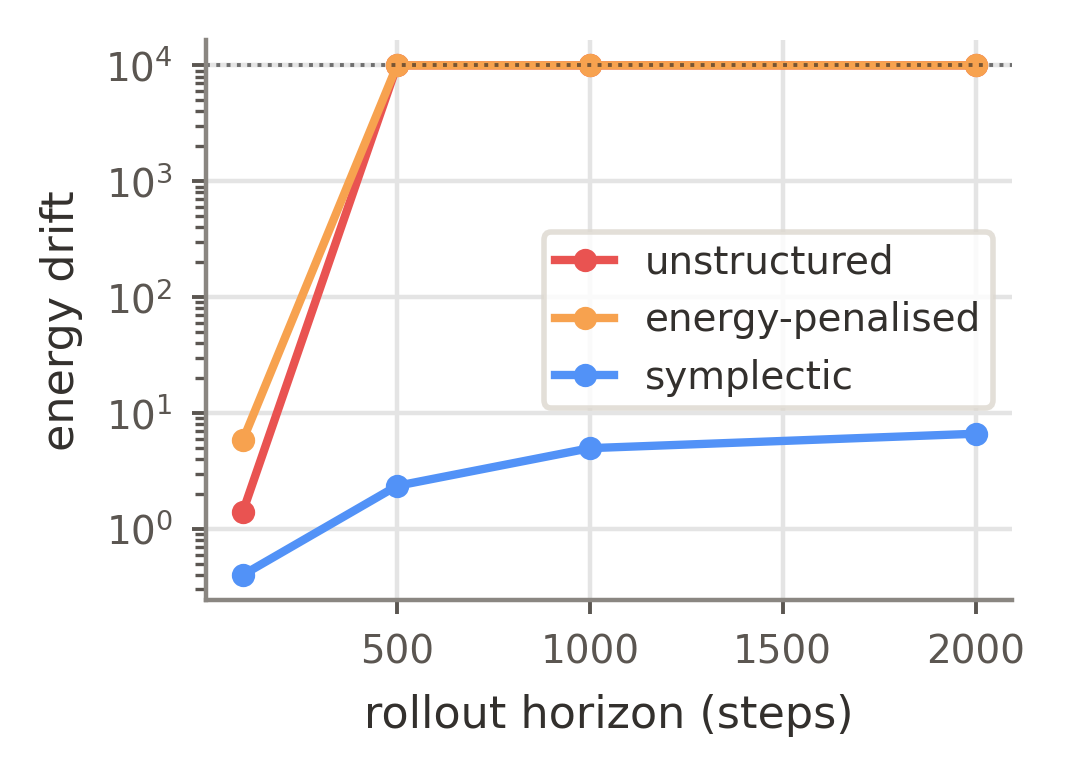}
  \caption{Relative energy drift $D(t)$
  vs rollout step (log $y$, 3 seeds; the shaded band is the plotting normal-approximation interval, descriptive at three seeds) in the original three-model display. The unstructured MLP and energy-penalised model reach the
$10^4$ display cap by ${\sim}$step 300; the symplectic model plateaus. Numerical summaries
are taken from the raw record (Table~\ref{tab:f1}), not from this capped display.}
\end{figure}

\begin{table}[H]
  \centering
  \scriptsize
  \setlength{\tabcolsep}{3.5pt}
  \caption{Raw stability record: per-episode median
  relative energy drift over $64$ held-out episodes, mean [min, max] over three runs;
  ${>}10^{4}$ marks a median above the cap. NF is the number of episodes non-finite by step
  $2000$ (of $192$; none is non-finite by step $1000$); $\mathrm{err}_{2000}$ is the median
  normalised rollout error at $2000$ steps, range over seeds. Learning rates were selected
  per seed on held-out one-step error from $\{10^{-3},3\times10^{-4},10^{-4}\}$.}
  \label{tab:f1}
  \resizebox{\linewidth}{!}{%
  \begin{tabular}{@{}lcccccc@{}}\toprule
    model & $t{=}100$ & $t{=}500$ & $t{=}1000$ & $t{=}2000$ & NF & $\mathrm{err}_{2000}$ \\\midrule
    unstructured MLP & $2.6$ [$0.6$, $5.1$] & ${>}10^{4}$ & ${>}10^{4}$ & ${>}10^{4}$ & $61$ & ${>}10^{4}$ \\
    energy-penalised, $\lambda{=}1$ & $3.3$ [$2.1$, $4.0$] & ${>}10^{4}$ & ${>}10^{4}$ & ${>}10^{4}$ & $123$ & ${>}10^{4}$ \\
    energy-penalised, $\lambda{=}100$ & $0.011$ [$0.009$, $0.016$] & $0.33$ [$0.16$, $0.49$] & $56$ [$1.4$, $141$] & $115$--${>}10^{4}$ & $0$ & $82$--${>}10^{4}$ \\
    neural ODE (RK4) & $0.08$ [$0.07$, $0.11$] & $1.8\times10^{3}$ [$3.0\times10^{2}$, $4.0\times10^{3}$] & ${>}10^{4}$ & ${>}10^{4}$ & $0$ & ${>}10^{4}$ \\
    symplectic model & $0.40$ [$0.33$, $0.46$] & $2.7$ [$2.0$, $3.7$] & $4.3$ [$3.4$, $6.2$] & $7.1$ [$6.7$, $7.7$] & $0$ & $10$--$14$ \\
    \bottomrule
  \end{tabular}}
\end{table}

\paragraph{The neural-ODE and energy-penalty controls.}
The neural vector field with RK4 uses matched capacity, the same step size and
one-step objective, and the same learning-rate sweep. All three seeds remain finite
in every episode. Its median energy drift is $0.08$ at $100$ steps,
$3\times10^2$--$4\times10^3$ at $500$, and above the cap by $1000$. Equal step size
does not imply equal integration error across RK4 and leapfrog. The $\lambda=100$
energy-penalised MLP has the lowest energy error of all the models at $100$ and $500$
steps, after which the drift grows without bound, with the rollout error growing
alongside it.

All models share the evaluation episodes within a seed, so the record supports
paired comparisons. The summaries here support a long-horizon energy-error
contrast, not a claim that Hamiltonian structure is necessary for finiteness.
The comparison is generated from one raw record rather than from capped plots. For each
model and run, that record stores the learning-rate grid and the selected rate, the
held-out one-step error used for selection, the per-episode finite and non-finite status
and counts at each checkpoint, the uncapped relative energy drift on finite episodes
(median, minimum, maximum, and fraction above $1$), and the rollout error, together with
the selected checkpoint for each model and seed. Seed-level summaries are the mean,
minimum, and maximum over seeds of per-episode medians.

\subsection{B. Law-inversion and the double dissociation (extends \S\ref{sec:r1}--\ref{sec:r2})}
\label{app:B}

\textbf{Conservation alone does not provide inversion.} Under never-trained
repulsive $G<0$, the symplectic model is stable but its rollout stays closer to the
training-law trajectory (regime-match $0$ at every $G$ and seed). It learned
a black-box $V_\theta(q,G)$ on $G>0$ only (Prop.~\ref{prop:invert} contrast).
\begin{figure}[H]
  \centering
  \includegraphics[width=0.5\linewidth]{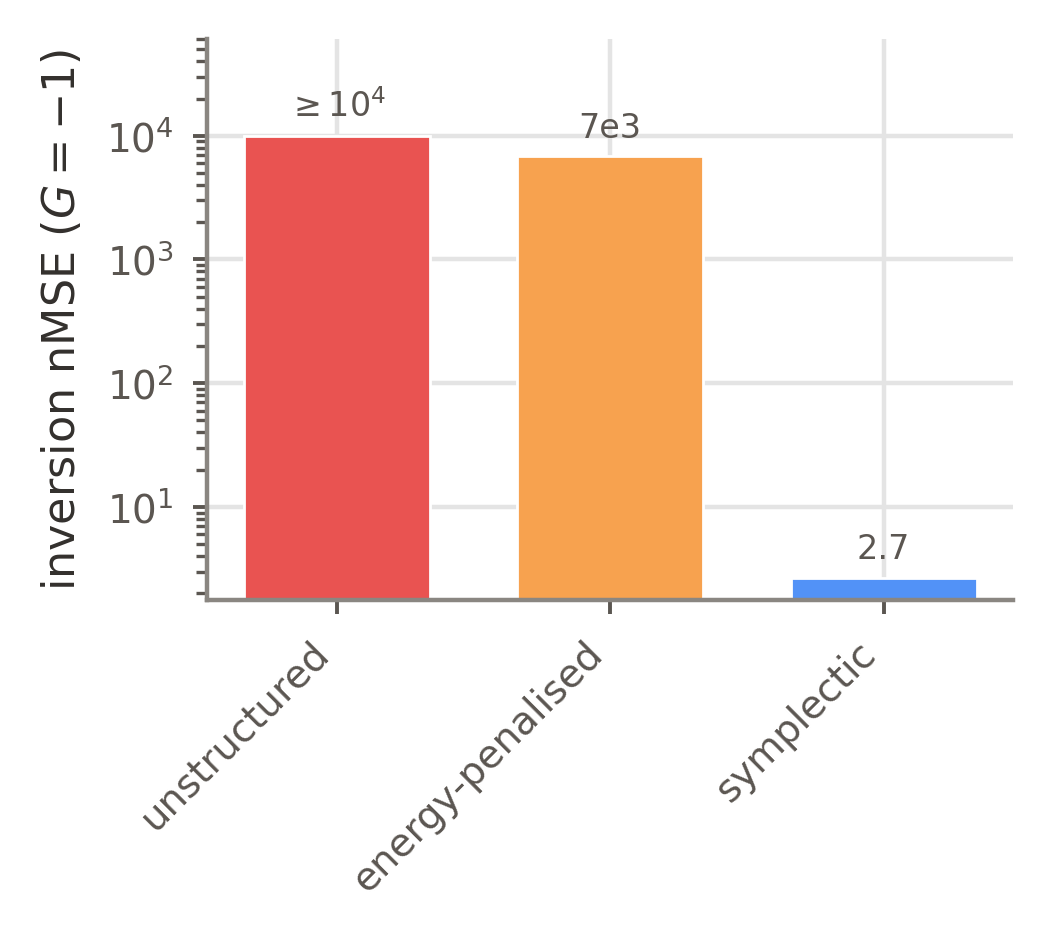}
  \caption{Non-factored HNN: stable but does not
invert; repulsive rollout matches the attractive default.}
\end{figure}

\textbf{The factored coupling yields inversion.} Trained
on attraction only, the factored symplectic model (Eq.~\ref{eq:factored}) tracks unseen
repulsion at nMSE $0.03$--$0.18$ across $G=-0.5$ to $-1.5$ (regime-match $1$ everywhere),
more than an order of magnitude below the non-factored symplectic model and far below the
(capped) unstructured baselines, while attaining the lowest in-distribution
error with fewer parameters. Together with
the non-factored result, this is the double dissociation.
\begin{figure}[H]\centering
  \begin{subfigure}[t]{0.46\textwidth}
    \centering
    \includegraphics[width=\linewidth]{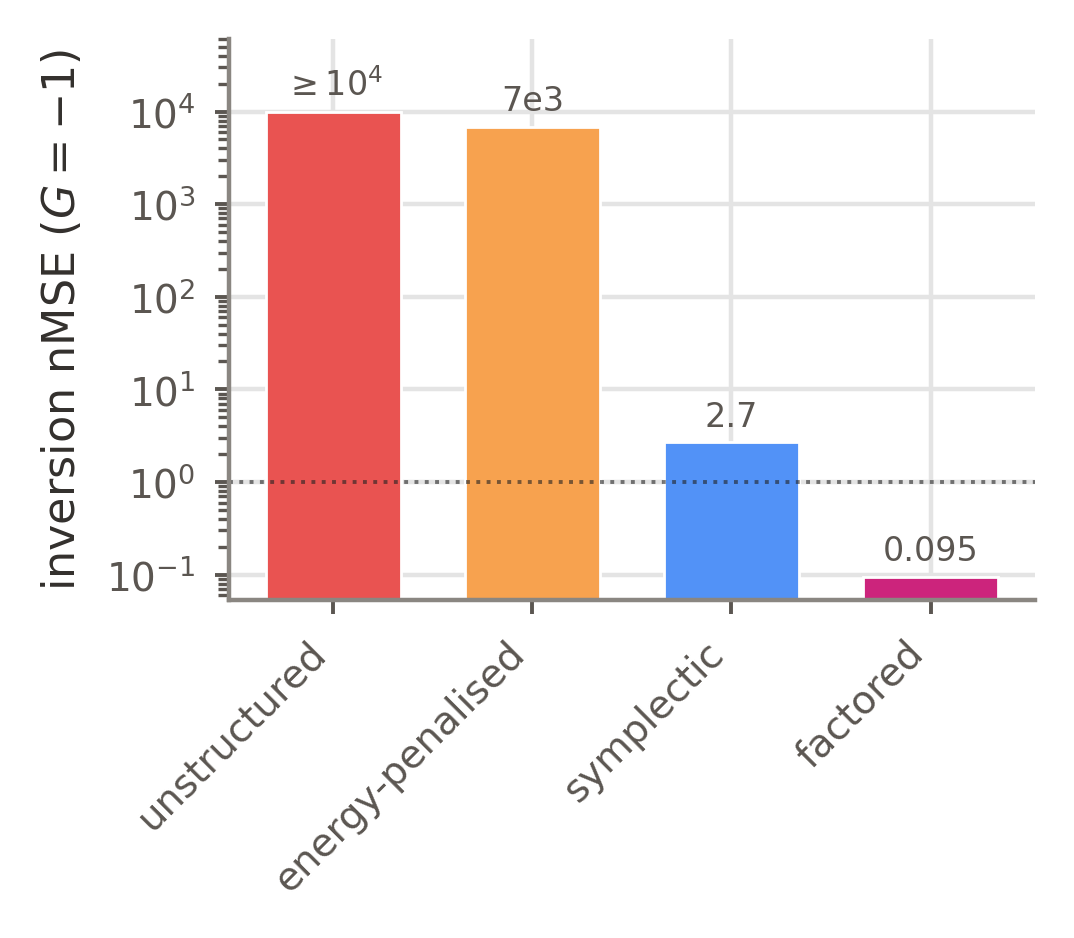}
    \caption{inversion error}
  \end{subfigure}\hfill\begin{subfigure}[t]{0.46\textwidth}
    \centering
    \includegraphics[width=\linewidth]{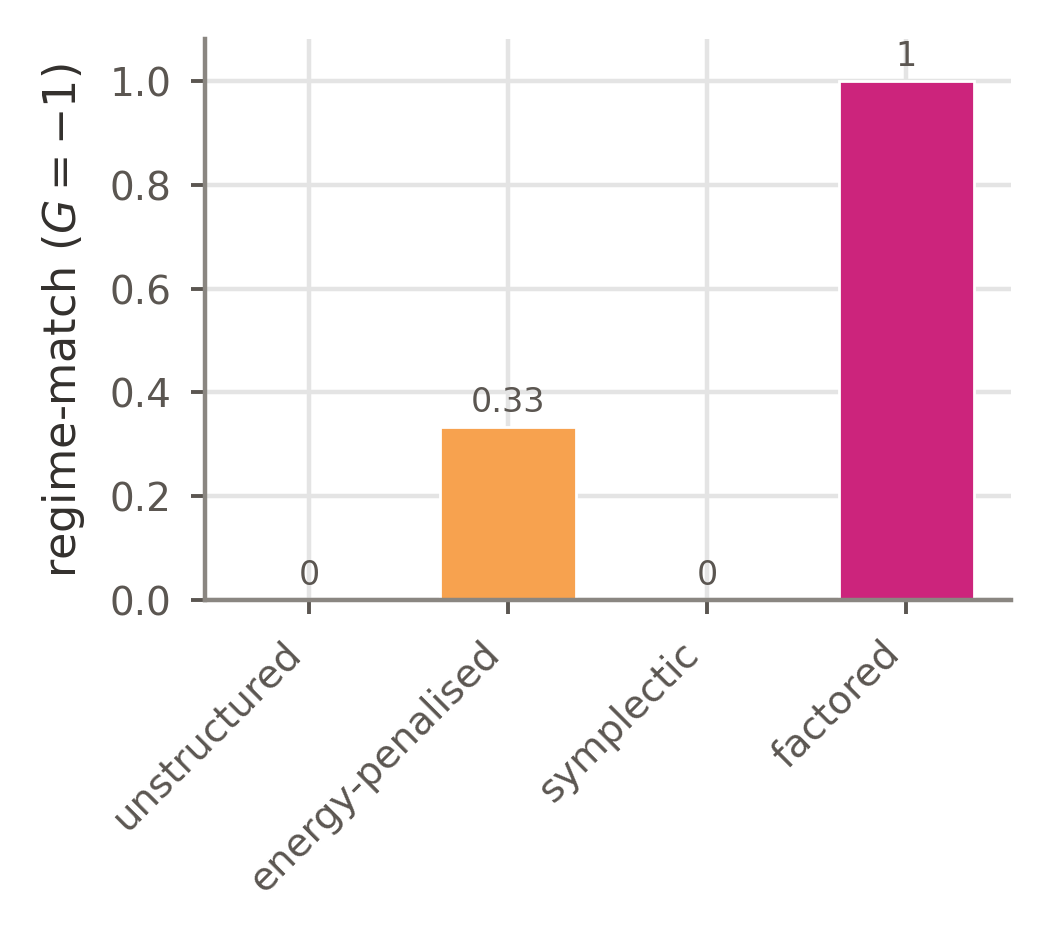}
    \caption{regime-match}
  \end{subfigure}
  \caption{Factored HNN tracks the true
    reversed ($G<0$) reference trajectory; the other evaluated models diverge or
  remain closer to the training-law trajectory.}
\end{figure}

\textbf{DeLaN/SymODEN: comparison across formalisms and integrators.} On the
engine gravity-flip, three non-factored structured models fail and both factored
models invert (Table~\ref{tab:delan}), isolating the axis to the factoring.
\begin{figure}[H]
  \centering
  \includegraphics[width=0.62\linewidth]{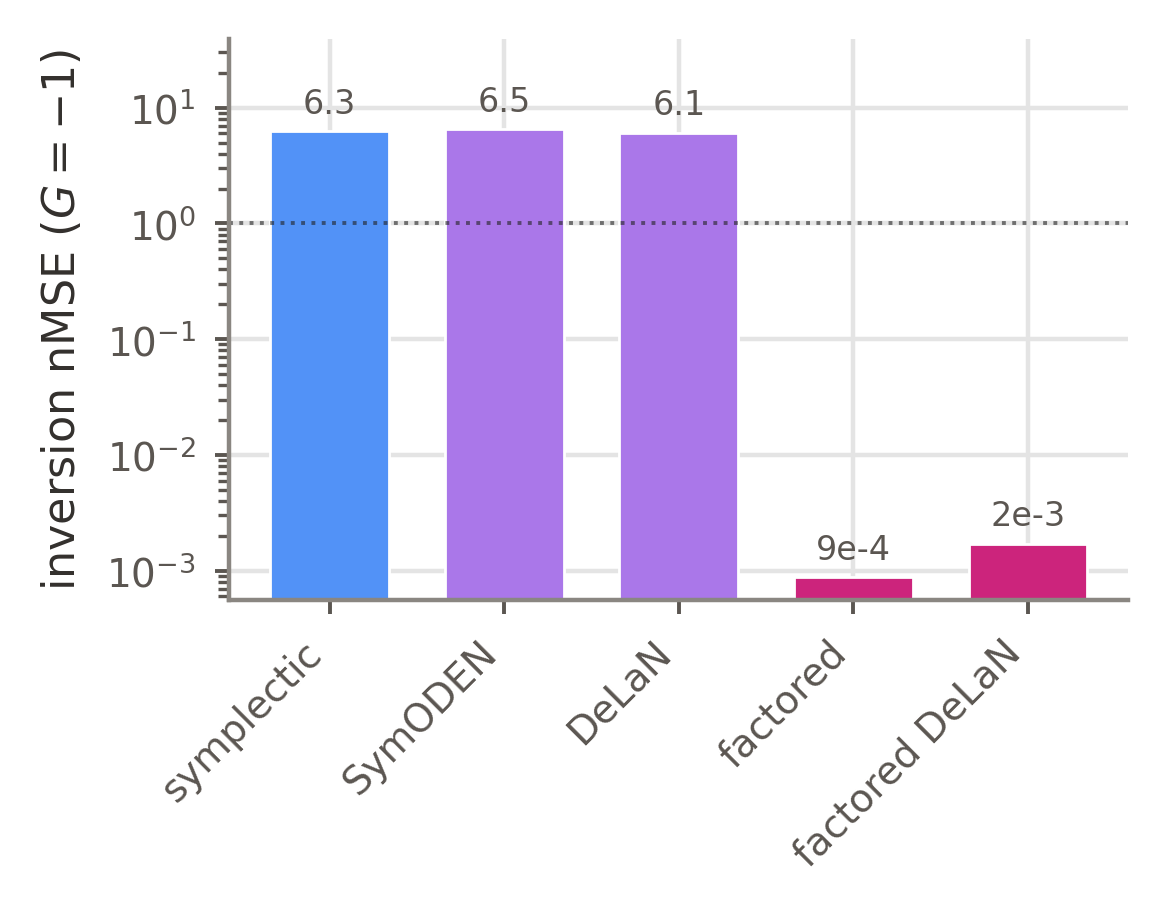}
  \caption{\textbf{DeLaN/SymODEN.} Inversion across five
  structured models: non-factored fail (match 0), factored invert (match 1), all
equally stable.}
\end{figure}

\begin{table}[H]
  \centering
  
  \caption{\textbf{DeLaN/SymODEN} engine gravity-flip inversion.}
  \label{tab:delan}
  \begin{tabular}{@{}lcc@{}}\toprule
    model & regime-match & inversion error \\\midrule
    symplectic model (Tao, non-fac.) & 0 & ${\sim}6$ \\
    SymODEN (RK4, non-fac.) & 0 & ${\sim}6$ \\
    DeLaN (non-fac.) & 0 & ${\sim}6$ \\
    factored symplectic model (factored) & 1 & $8.7\times10^{-4}$ \\
    factored DeLaN (factored) & 1 & $1.7\times10^{-3}$ \\
    \bottomrule
  \end{tabular}
\end{table}

\textbf{Matched causal factorial.} A natural worry is that our
stability and inversion contrasts change several things at once, namely the state
representation, the update rule, and the numerical integrator. We therefore run a
controlled factorial that varies these three axes independently. We cross a
Hamiltonian representation against an unrestricted vector field, a
coupling-factored parameter map against an unrestricted conditional map, and
three integrators (leapfrog, RK4, and implicit midpoint) at three step sizes,
over twelve paired seeds in two coercive coupled-oscillator systems, with every
model fitted to the same observed regime. The unrestricted vector field has
more coefficients than the Hamiltonian model. Parameter counts alone do not establish equal approximation
capacity. Two results follow. First, at matched fit the Hamiltonian
representation reaches a maximum true-energy error of about $5\times10^{-4}$ to
$8\times10^{-4}$, against $0.31$ to $0.33$ for the unrestricted field, and a
trajectory error of $5\times10^{-4}$ to $1.7\times10^{-3}$ against $0.044$ to
$0.059$ (Table~\ref{tab:p1-factorial}). Second, at the unseen negative coupling
the factored map reaches an error of about $0.0035$, while the unrestricted
conditional map stays near $0.985$, a gap that holds in both systems
(Holm-adjusted $p\approx0.003$, twelve seeds; Table~\ref{tab:p1-negmap}). The
differences persist across the evaluated representations, parameter maps,
and solvers. This does not establish equal functional capacity or exclude
all optimisation and discretisation effects. Every factorial cell stays bounded
here (escape fraction $0$), so the factorial isolates a fidelity and inversion
effect rather than a boundedness effect. The boundedness contrast comes instead
from the softened-gravity system of \S\ref{sec:r1}, where the unrestricted and
$\lambda=1$ energy-penalised models pass the $10^{4}$ cap by step 500 and lose episodes to
non-finite states by step 2000, while the symplectic model stays finite and below the cap.

\begin{table}[H]
  \centering
  
  \caption{Controlled factorial at step size $0.05$ over twelve training seeds. Ranges span the two systems and the relevant model cells.}
  \label{tab:p1-factorial}
  \begin{tabular}{@{}lcc@{}}
    \toprule
    Comparison & Hamiltonian representation & Unrestricted vector field \\
    \midrule
    Maximum true-energy error & $0.00050$ to $0.00084$ & $0.314$ to $0.332$ \\
    Trajectory normalised error & $0.00054$ to $0.00169$ & $0.0439$ to $0.0586$ \\
    Escape fraction & $0$ & $0$ \\
    \bottomrule
  \end{tabular}
\end{table}

\begin{table}[H]
  \centering
  
  \caption{Matched-state response at the unseen negative coupling.}
  \label{tab:p1-negmap}
  \begin{tabular}{@{}lcc@{}}
    \toprule
    System & Coupling-factored map & Unrestricted conditional map \\
    \midrule
    Coupled oscillator & $0.00348$ & $0.9853$ \\
    Modified oscillator & $0.00382$ & $0.9853$ \\
    \bottomrule
  \end{tabular}
\end{table}

\textbf{Graph-network response and long-horizon failure.}
Both graph variants fail the 2000-step test, so message factoring alone
does not prevent their observed long-horizon failure.
The source reports nearest-reference labels of 1 for the factored variant
and 0 for the other variant. Without verified short-horizon finiteness
and fidelity, those labels are descriptive and are not counted here as
validated counterfactual success. Training-noise injection does not
remove the reported long-horizon failure (Table~\ref{tab:f18}).

\begin{figure}[H]
  \centering
  \begin{subfigure}[t]{0.46\textwidth}
    \centering
    \includegraphics[width=\linewidth]{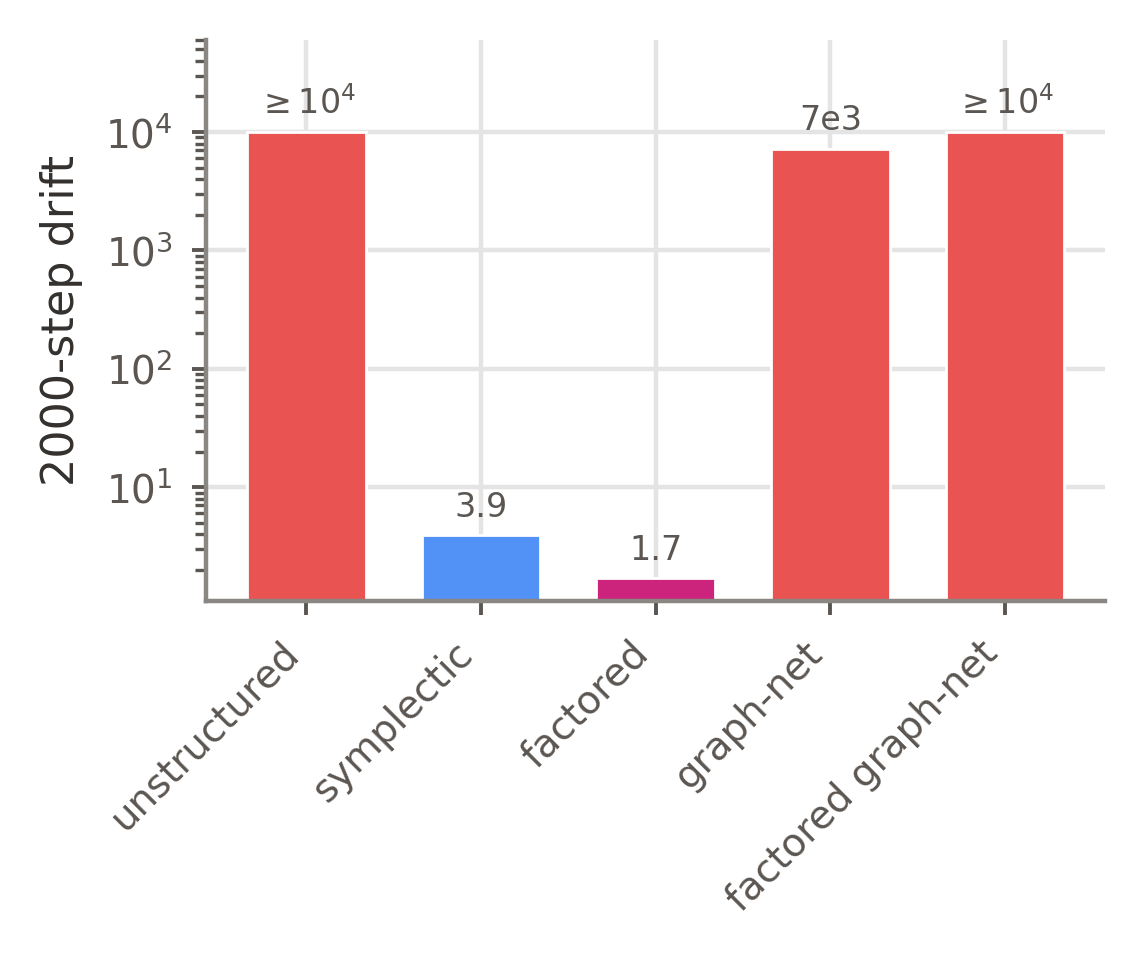}
    \caption{2000-step drift}
  \end{subfigure}\hfill\begin{subfigure}[t]{0.46\textwidth}
    \centering
    \includegraphics[width=\linewidth]{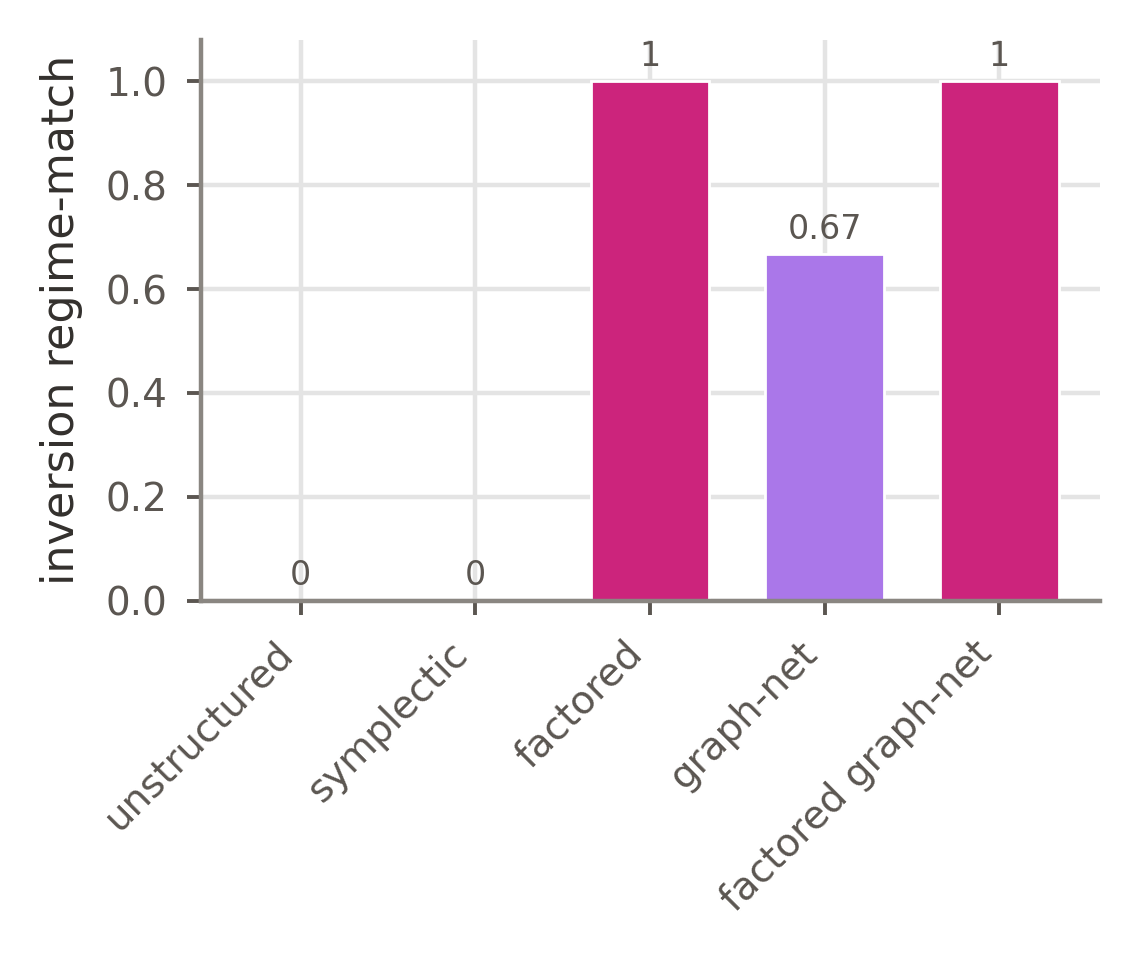}
    \caption{inversion regime-match}
  \end{subfigure}
  \caption{Graph-network long-horizon failure and reported response labels.
  The labels require separate short-horizon validity checks before being
  interpreted as counterfactual success.}
\end{figure}

\begin{table}[H]
  \centering
  
  \caption{Graph-network summaries. Failure is separated from finite energy
  error; reported response labels are descriptive pending horizon-validity checks.}
  \label{tab:f18}
  \begin{tabular}{@{}lcc@{}}\toprule
    model & 2000-step drift & regime-match \\\midrule
    unstructured MLP (MLP) & numerical failure & 0 \\
    graph-network simulator & long-horizon failure & 0 \\
    factored graph-network simulator & long-horizon failure & 1 \\
    symplectic model/factored symplectic model & bounded & --/1 \\
    \bottomrule
  \end{tabular}
\end{table}

\subsection{C. Contacts and the engine (extends \S\ref{sec:r5})}
\label{app:C}

\textbf{Contacts.} On colliding disks (contact stiffness series: soft
$k{=}30$, stiff $k{=}300$, impulsive), the factored symplectic model's energy stays bounded across
the whole series and time-reversal (symmetry) inversion remains present
(Table~\ref{tab:f13}); parameter-inversion regime-matches (match 1) at soft/stiff,
though its fidelity degrades ($\text{nMSE }0.76\to2.9$), and has no referent at the
impulsive level (\S\ref{app:D}).

\begin{figure}[H]
  \centering
  \begin{subfigure}[t]{0.46\textwidth}
    \centering
    \includegraphics[width=\linewidth]{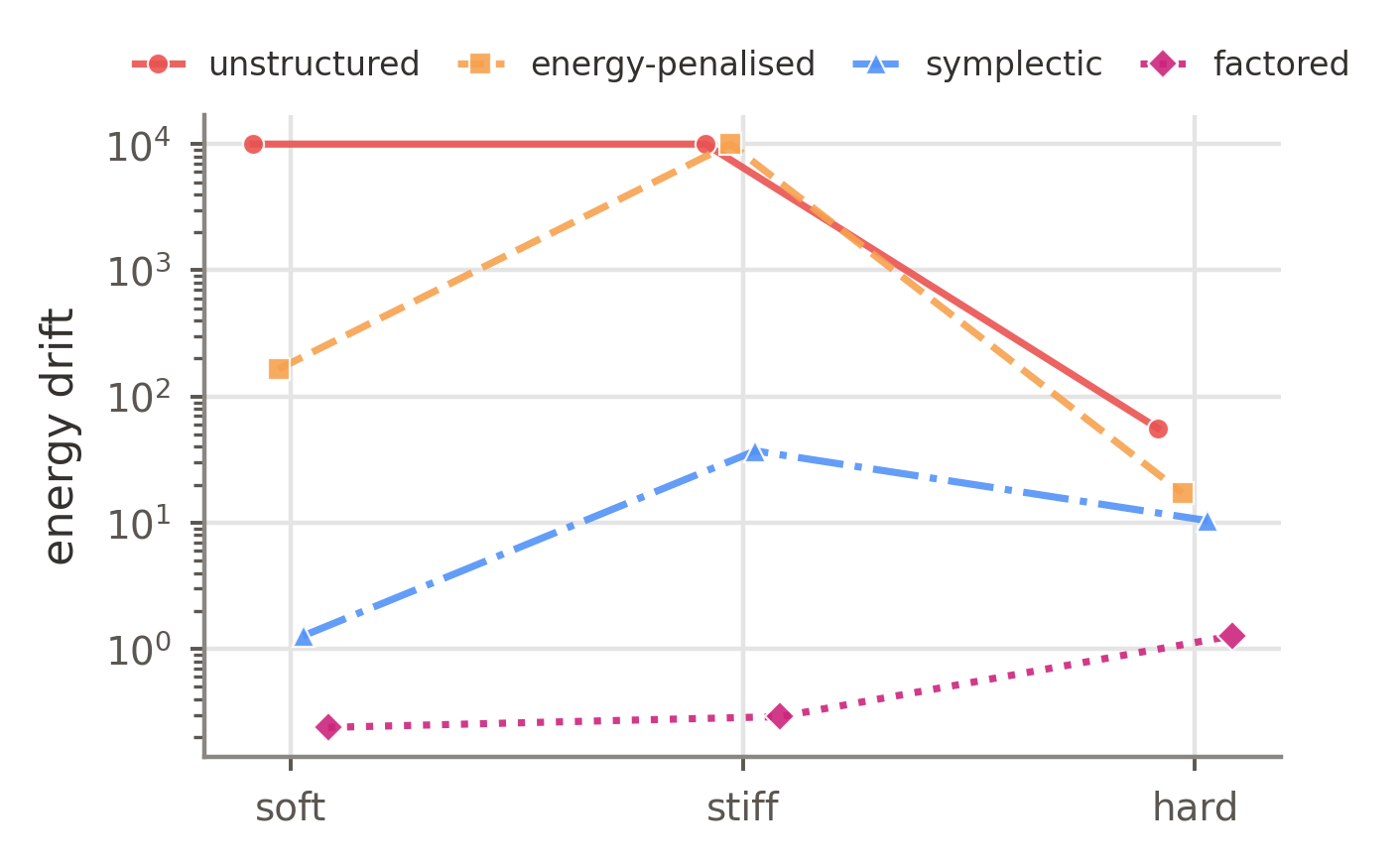}
    \caption{energy drift across the stiffness series}
  \end{subfigure}\hfill\begin{subfigure}[t]{0.46\textwidth}
    \centering
    \includegraphics[width=\linewidth]{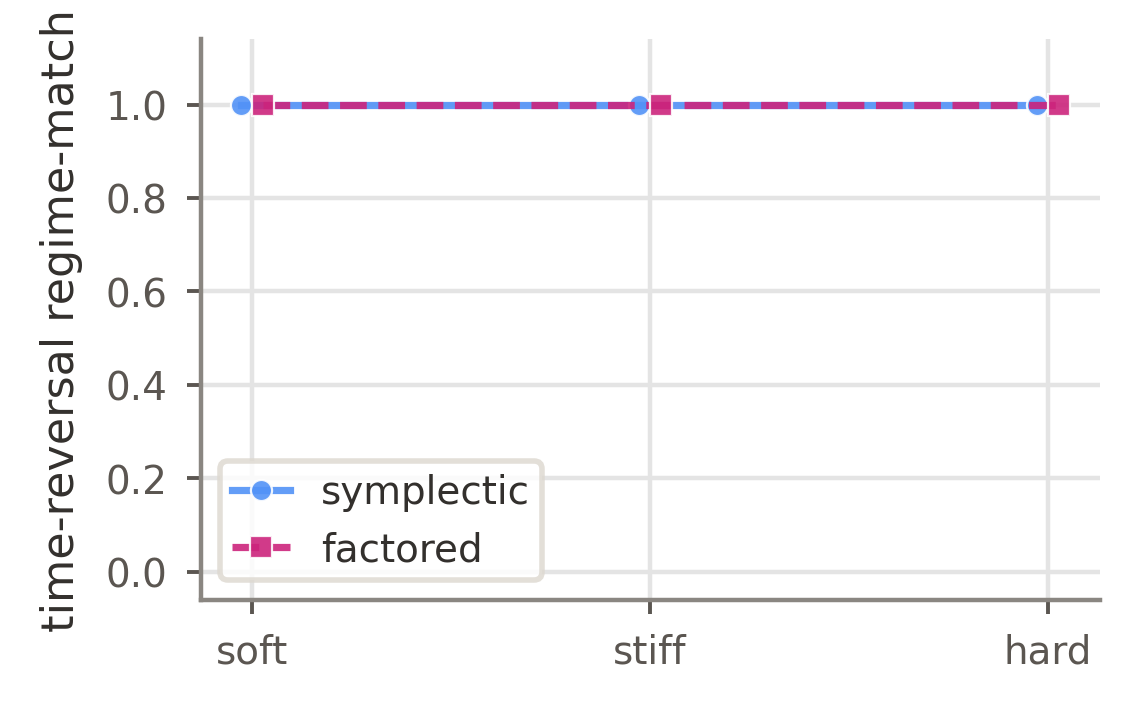}
    \caption{time-reversal regime-match}
  \end{subfigure}
  \caption{Contact stiffness series: bounded energy and
  surviving symmetry-inversion where unconstrained models diverge.}
\end{figure}

\begin{table}[H]
  \centering
  
  \caption{Across the stiffness series (factored symplectic model).}
  \label{tab:f13}
  \begin{tabular}{@{}lccc@{}}\toprule
    level & drift & sym-inversion match & param-inversion nMSE \\\midrule
    soft & $0.24$ & 1 & $0.76$ \\
    stiff & $0.29$ & 1 & $2.9$ \\
    impulsive & $1.27$ & 1 & undefined \\
    \bottomrule
  \end{tabular}
\end{table}

\subsection{D. The boundary map (extends \S\ref{sec:r4})}
\label{app:D}

\begin{figure}[h]
  \centering
  \begin{subfigure}[t]{0.32\textwidth}
    \centering
    \includegraphics[width=\linewidth]{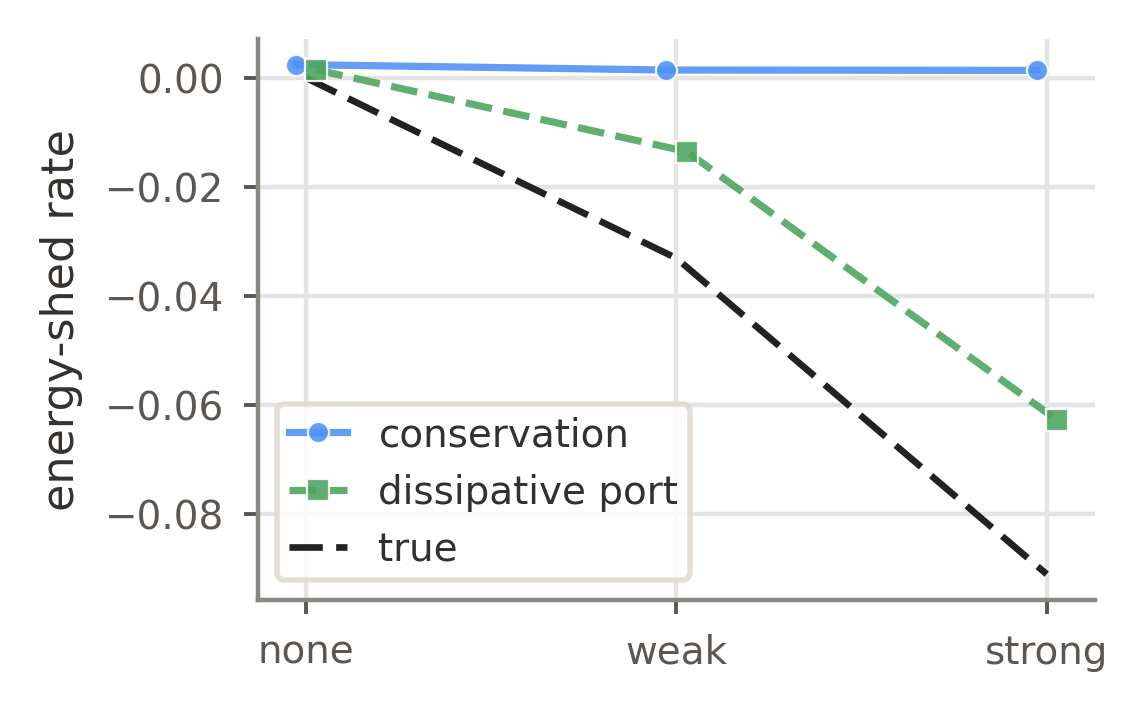}
    \caption{Dissipation: conservation is incompatible}
  \end{subfigure}\hfill
  \begin{subfigure}[t]{0.32\textwidth}
    \centering
    \includegraphics[width=\linewidth]{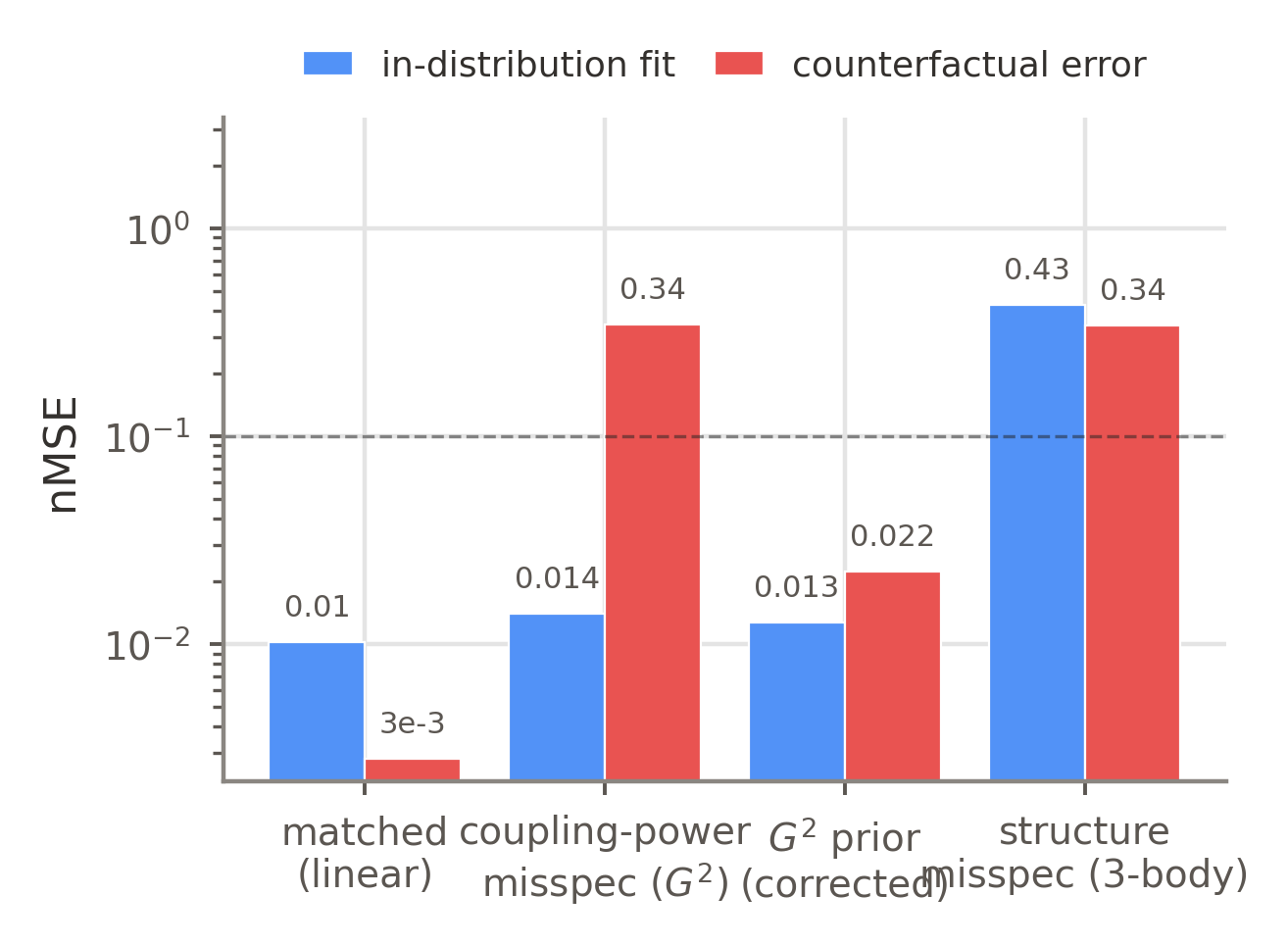}
    \caption{A misspecified factored law}
  \end{subfigure}\hfill
  \begin{subfigure}[t]{0.32\textwidth}
    \centering
    \includegraphics[width=\linewidth]{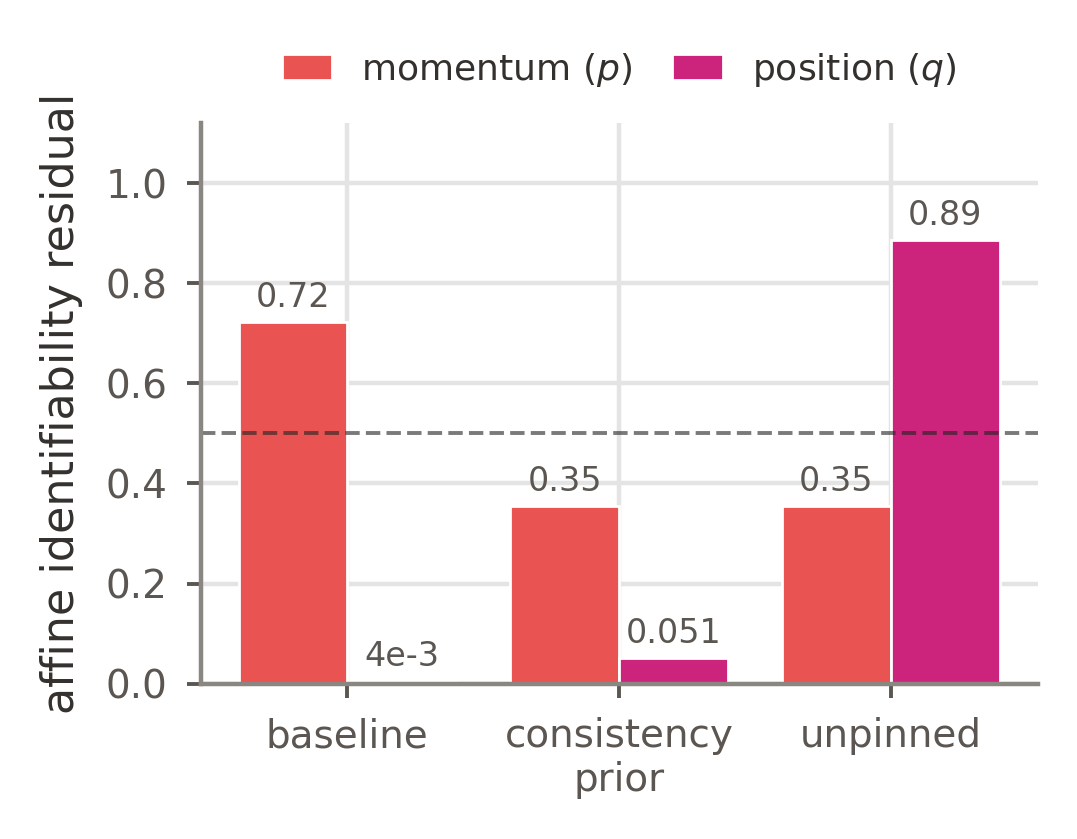}
    \caption{Perception: momentum is the open axis}
  \end{subfigure}
  \caption{The boundaries. Conservation cannot shed energy under drag; a
    factored law of the wrong power fits in-distribution yet gives a wrong
  counterfactual; and a single frame recovers position but not momentum.}
  \label{fig:boundary}
\end{figure}

\textbf{The impulsive-contact boundary.}
The stiffness-sign intervention used here has no specified counterpart in
the impulsive hard-contact model. This does not rule out interventions
on restitution, friction, or forcing in a separately defined hybrid law.
The tested differentiable-restitution variant does not improve on the
smooth model and loses time-reversal performance at the impulsive level;
the free-learned impulse model diverges (Table~\ref{tab:f26}).
These results delimit the tested augmentation rather than all
contact-capable structured models.

\begin{table}[H]
  \centering
  
  \caption{Drift at the soft, stiff, and impulsive levels.}
  \label{tab:f26}
  \begin{tabular}{@{}lccc@{}}\toprule
    model & soft & stiff & impulsive \\\midrule
    factored symplectic model (smooth) & $0.11$ & $0.29$ & $1.06$ \\
    diff.-restitution & $0.92$ & $16.9$ $[0.25, 33.6]$ & $0.91$ (loses retrace) \\
    free-learned impulse & not evaluated & not evaluated & $3360$ (diverges) \\
    \bottomrule
  \end{tabular}
\end{table}

\textbf{Conservation is incompatible with dissipative dynamics.} With
linear drag the truth sheds energy (Fig.~\ref{fig:boundary}); symplectic model/factored symplectic model shed ${\sim}0$, and only
a port-Hamiltonian variant more closely reproduces the reference shedding rate
and inverts the drag sign
(Table~\ref{tab:f14}). Energy-shed rate is primary here (trajectory-MSE does not
separate the regimes short-horizon).
\begin{figure}[H]
  \centering
  \includegraphics[width=0.5\linewidth]{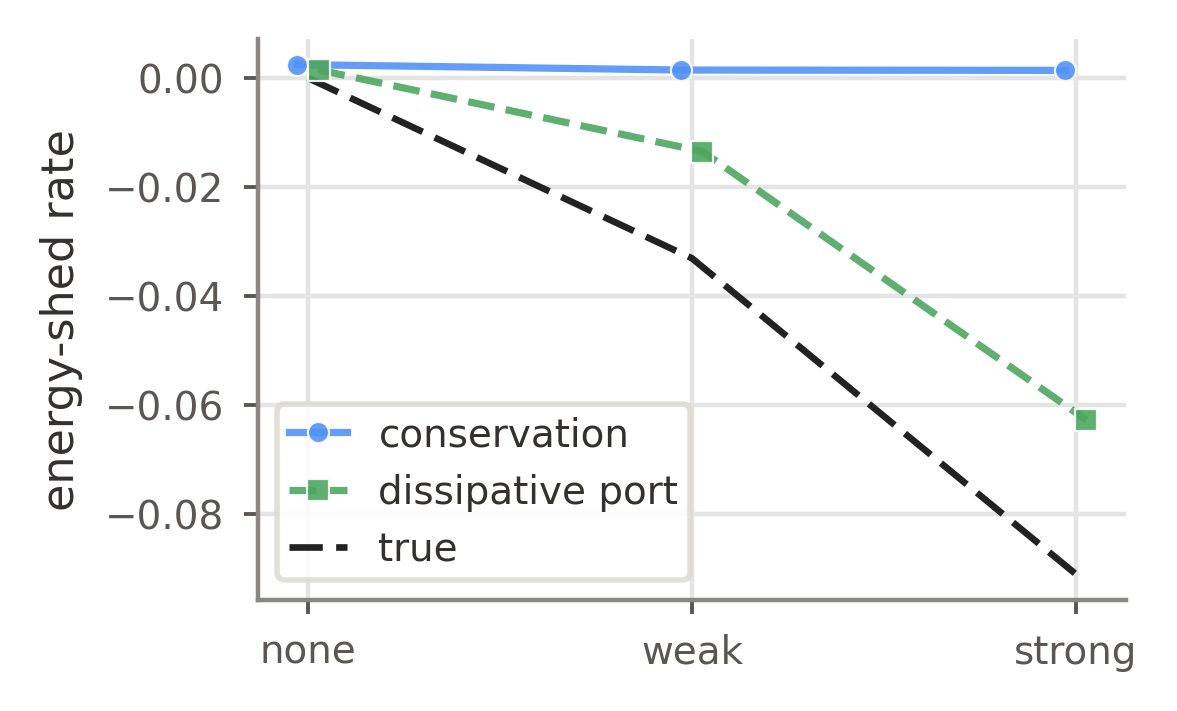}
  \caption{Conservative models have approximately
zero energy-shed rate; the port variant more closely reproduces the reference rate.}
\end{figure}

\begin{table}[H]
  \centering
  
  \caption{Energy-shed rate (negative = losing energy).}
  \label{tab:f14}
  \begin{tabular}{@{}lcc@{}}\toprule
    & weak drag & strong drag \\\midrule
    truth & $-0.033$ & $-0.091$ \\
    symplectic model/factored symplectic model & ${\sim}0$ & ${\sim}0$ \\
    port-Hamiltonian & $-0.014$ & $-0.063$ \\
    \bottomrule
  \end{tabular}
\end{table}

\textbf{Supervised perception.} With a soft-argmax encoder
supervised to the oracle state, position is near-perfect and flat across a visual-difficulty
series; the limit is object \emph{speed} (per-body momentum error $0.01$ slow
$\to0.87$ fast), and downstream inversion is retained (factored symplectic model nMSE $0.062/0.064$
$\approx$ oracle $0.09$). Reference trajectories are rolled from the true initial state so that perception
error is counted rather than cancelled.
\begin{figure}[H]
  \centering
  \includegraphics[width=0.62\linewidth]{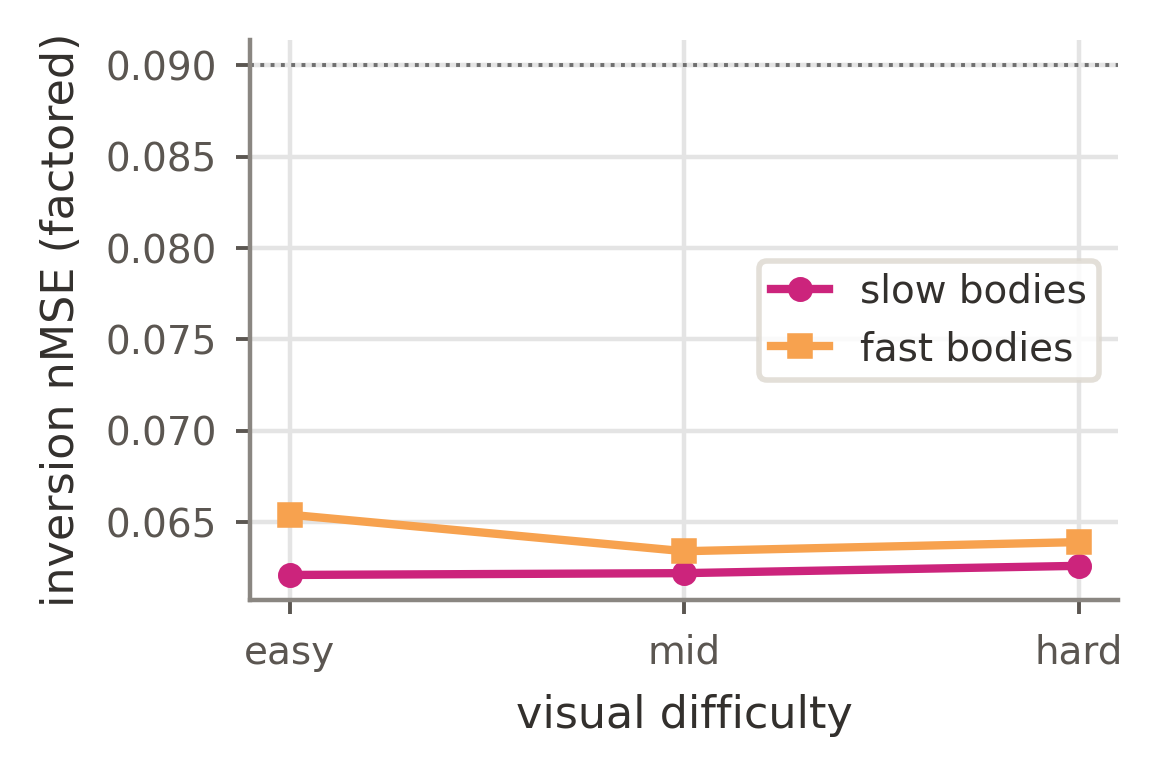}
  \caption{Supervised perception:
  position error remains stable; inversion is retained; momentum error increases
with object speed.}
\end{figure}

\textbf{Discovered perception leaves a large linear momentum-recovery error; a canonical
kinetic recovers it.} A latent learned end-to-end by prediction leaves a large linear momentum-recovery error
(affine $p$-residual $0.72$) with renderer-pinned position ($q$-residual $0.004$);
a canonical \emph{fixed} kinetic recovers inversion with error comparable to the
oracle result (3-seed, 500-epoch
nMSE $0.065$--$0.092$ $\approx$ oracle $0.067$; backstop $0.045$). A learned
decoder is invalid against the backstop. The value is the 500-epoch, three-seed result; a 60-epoch run is undertrained and
disagrees.
\begin{figure}[H]
  \centering
  \includegraphics[width=0.62\linewidth]{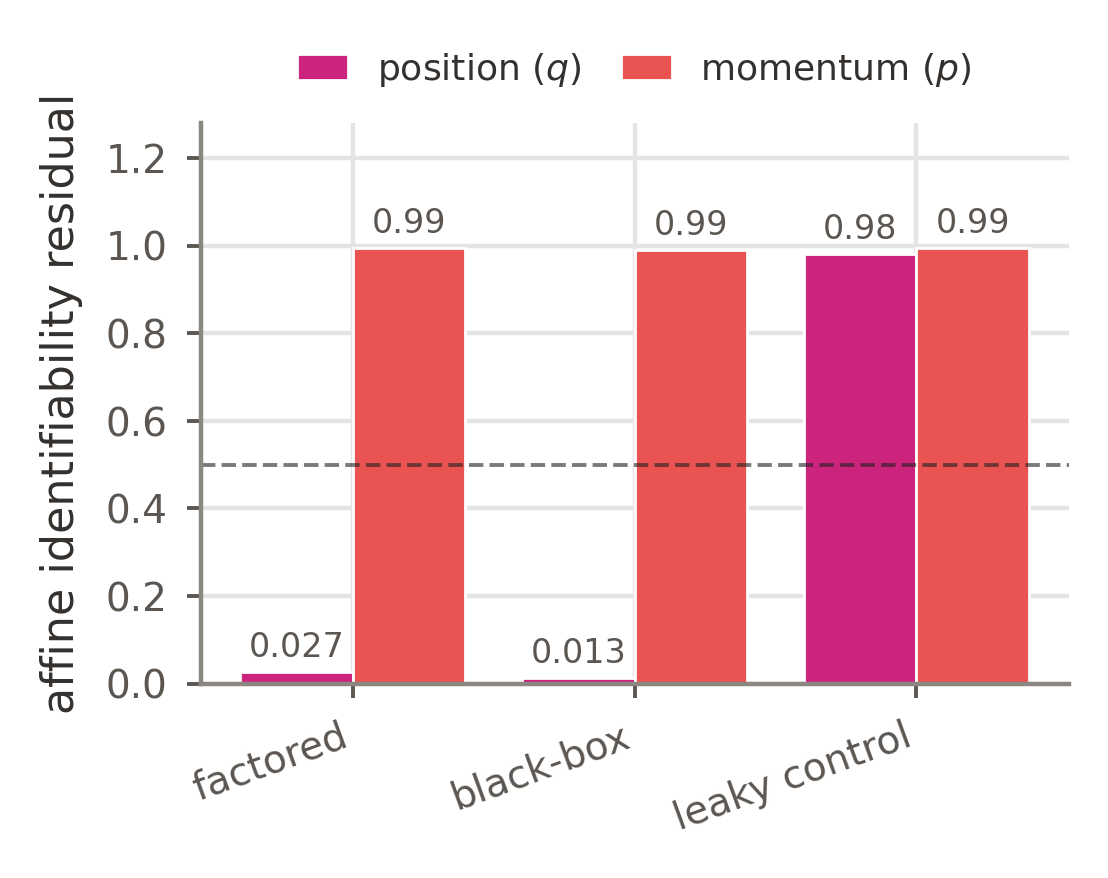}
  \caption{Discovered latent: momentum
  poorly recovered by an affine readout (high $p$-residual); the fixed canonical kinetic reduces inversion
nMSE to $0.065$--$0.092$, compared with the oracle value $0.067$.}
\end{figure}

\textbf{Momentum partially recoverable; position grounding open
[directional/boundary].} A symplectic-consistency prior drops the momentum
residual $0.72\to0.35$ ($<0.5$, improved linear recovery) while keeping position ($q$-residual
$0.05$) and low inversion error (nMSE $0.10$), at valid fidelity (backstop
$0.149$). Here $v_{\text{obs}}$ is frame-derived, not oracle. Removing the renderer to
\emph{discover} position increases the residual ($q$-residual $0.89$; backstop $1.84$, invalid):
the residual boundary is decoder-free position grounding.
\begin{figure}[H]
  \centering
  \includegraphics[width=0.62\linewidth]{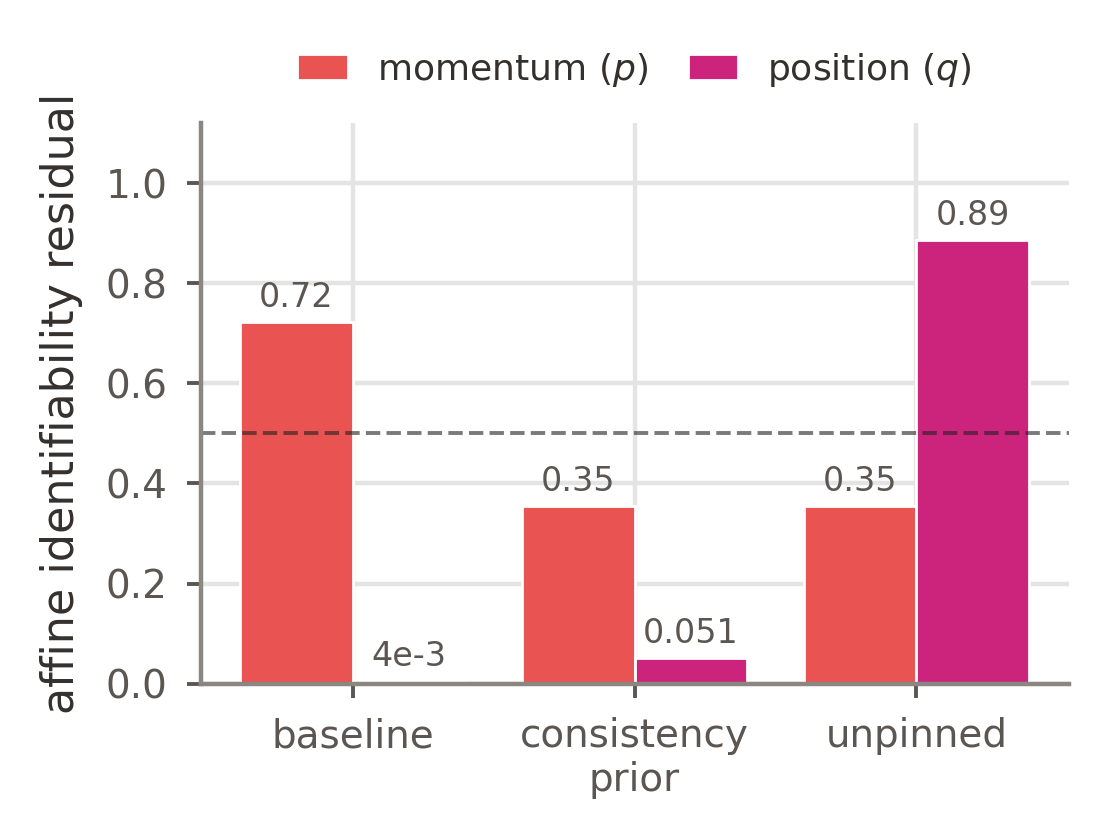}
  \caption{Symplectic-consistency improves linear momentum recovery
  ($p$-residual $0.72\!\to\!0.35$) while retaining position; removing
renderer-based grounding produces a high position residual.}
\end{figure}

\textbf{Pixel-based dynamics comparison under grounded perception
[well-supported, scoped].} Trained end-to-end on synthetic renders with a shared
fixed disk-renderer (only the dynamics differ), symplectic models stay bounded
while unstructured MLP/graph-network simulator/factored graph-network simulator/JEPA-style predictor diverge, and only the
factored fixed-kinetic model approaches the oracle
inversion error (Table~\ref{tab:f28}). The factored symplectic model (learned kinetic) fails like
the non-factored models, reproducing the empirical recovery condition on
pixels. The comparison evaluates dynamics \emph{on} pixels with
renderer-grounded position, not unsupervised inference \emph{from} pixels. The
baselines are restricted to the evaluated architecture families and do not
constitute a published leaderboard. Oracle-latent reconstruction is
$0.22$--$0.53$. The fully unsupervised branch is invalid
(backstops $2.8$--$3.3$) even though the decoder renders from oracle latents, a
position-discovery failure that confirms, rather than resolves, the
identifiability and position-grounding boundaries above. The decode-free JEPA-style predictor column here is the bare
objective (no grounding). It is developed into a grounded hybrid, and into its own
boundary, in App.~\ref{app:H}.

\begin{figure}[H]
  \centering
  \begin{subfigure}[t]{0.46\textwidth}
    \centering
    \includegraphics[width=\linewidth]{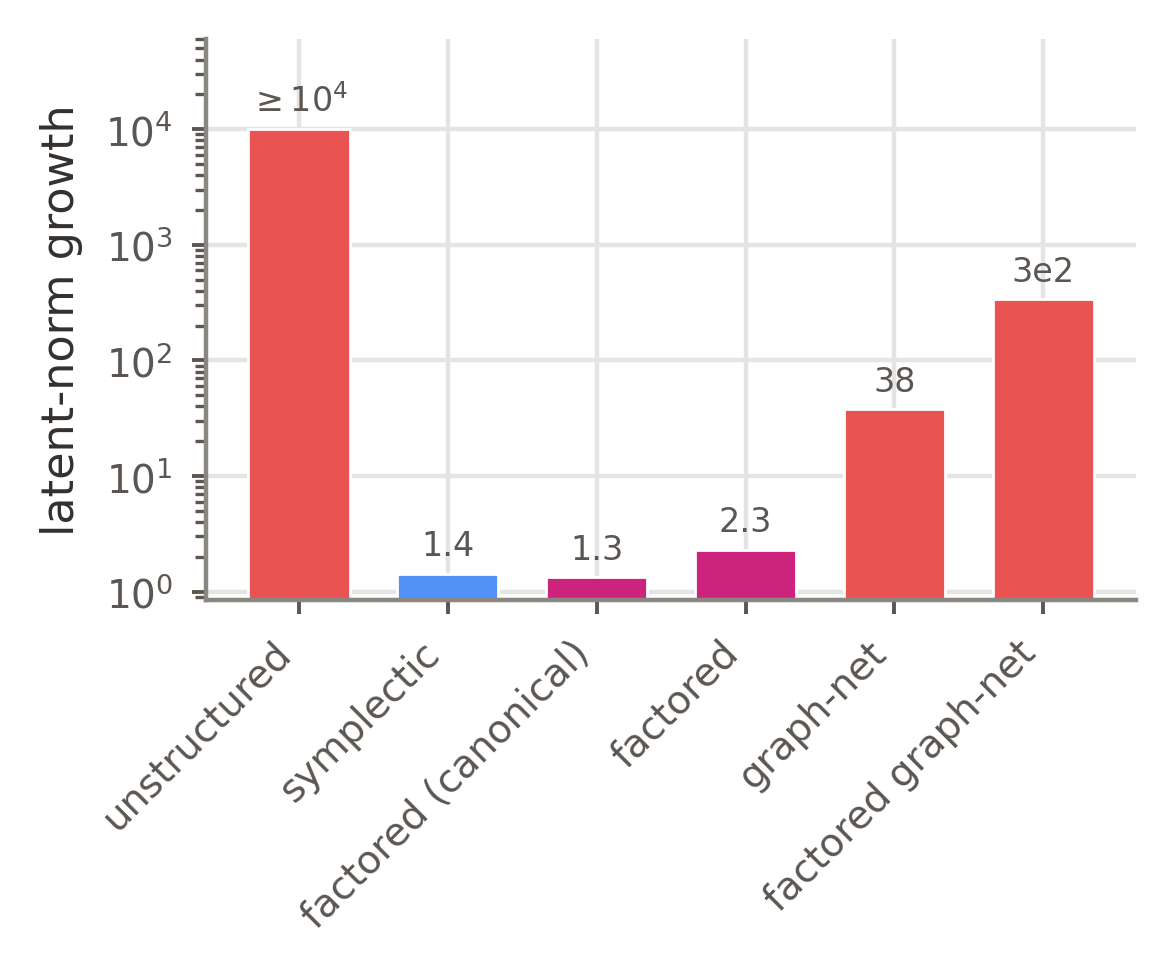}
    \caption{stability (latent-norm growth)}
  \end{subfigure}\hfill\begin{subfigure}[t]{0.46\textwidth}
    \centering
    \includegraphics[width=\linewidth]{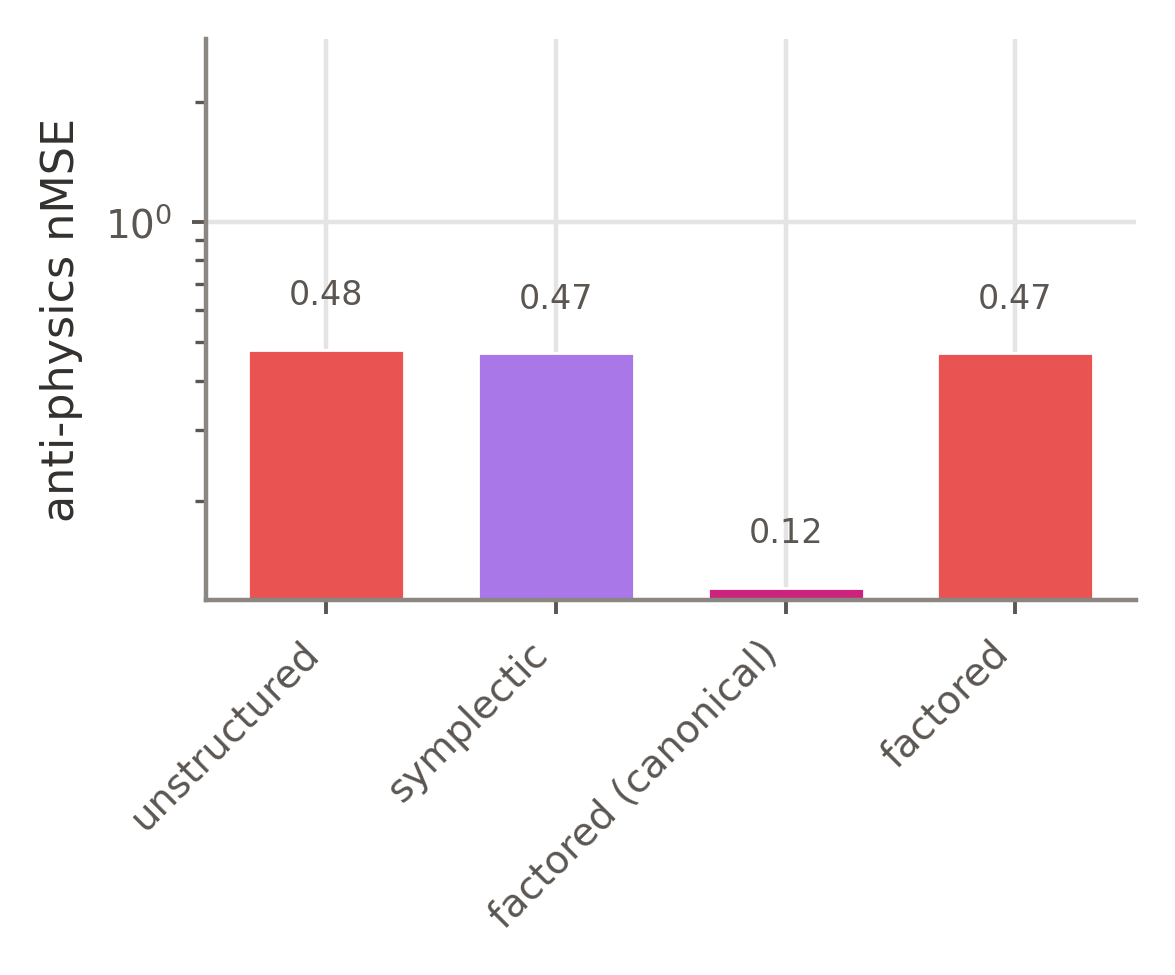}
    \caption{anti-physics inversion error}
  \end{subfigure}
  \caption{End-to-end on pixels: (a) latent-norm growth (log)
    diverges for all but symplectic models; (b) anti-physics inversion error is
    comparable to the oracle result only for the factored-plus-canonical-kinetic model. In-distribution
  backstops are within the valid range throughout (Table~\ref{tab:f28}).}
\end{figure}

\begin{table}[H]
  \centering
  
  \caption{End-to-end on pixels (grounded perception).}
  \label{tab:f28}
  \begin{tabular}{@{}lcc@{}}\toprule
    model & latent-norm growth & inversion nMSE \\\midrule
    unstructured MLP & $10^4$ & $0.478$ \\
    GNS / factored GNS & $38$ / $338$ & div.-contaminated \\
    JEPA-style predictor (decode-free) & $22$ & $1.46$ (diverges) \\
    symplectic model & $1.43$ & $0.468$ \\
    factored symplectic model (learned $T$) & $2.3$ & $0.468$ \\
    factored fixed-kinetic model (canonical $T$) & $1.33$ & $\mathbf{0.121}$ \\
    \bottomrule
  \end{tabular}
\end{table}

\subsection{E. Necessity, compositionality, generality, and scale (extends \S\ref{sec:r5})}
\label{app:E}

\textbf{Necessity.} A first screen yields a null result because
always-active interactions provide no available shortcut. A second screen makes a shortcut
available: a hybrid (factored pathway $+$ free black-box term) assigns
${\sim}42\%$ of the force, with inversion degraded $5$--$10\times$. An $L_2$
penalty drives the share to $10^{-4}$ and restores fidelity (nMSE $0.43\to0.09$).
These results support the shortcut-use interpretation within this physical system.

\textbf{Cross-family.} On Coulomb with mixed signs, a
charge-product coupling has the lowest error on a held-out all-repulsive configuration
(error $0.52$ vs black-box $1.42$, a modest ${\sim}2.7\times$ margin). A single
factored model spanning a mixed gravity$+$Coulomb generator inverts both
${\sim}120\times$ better than a black-box that reads the couplings (nMSE $0.0035$
vs $0.425$).
\begin{figure}[H]
  \centering
  \includegraphics[width=0.5\linewidth]{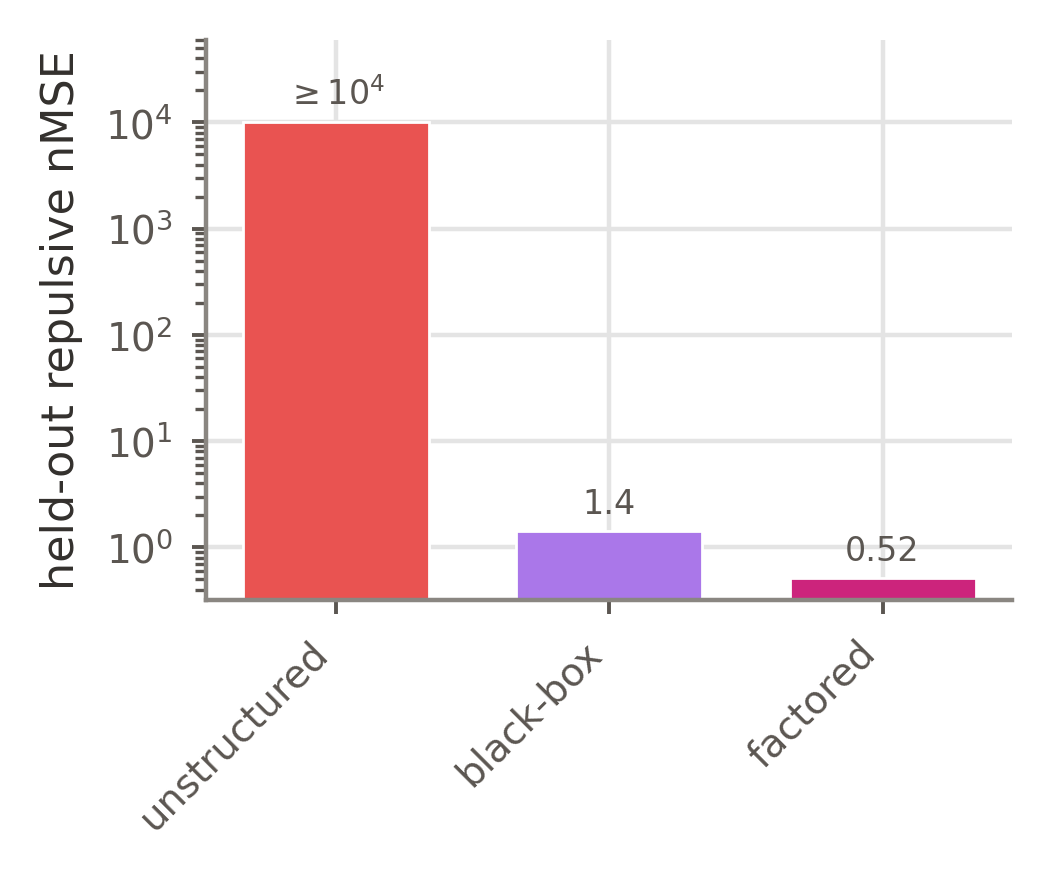}
  \caption{Held-out all-repulsive Coulomb:
the charge-product coupling generalises across sign.}
\end{figure}

\textbf{Data-efficiency.} The factored symplectic model inverts at every training-set
size down to 32 trajectories. Giving the non-factored symplectic model \emph{more} data
makes its inversion \emph{worse} ($0.76\to2.67$ across $32\to512$). This is not
undertraining at small data: the symplectic model's \emph{in-distribution} ($G{>}0$) rollout
stays bounded with low drift at every size (drift ${\sim}1.4$ at 32 trajectories), so it
fits the training law well. Within this experiment, the inversion failure is
therefore not explained by insufficient training data.

\textbf{Capacity and composition [well-supported / directional].} Across a width sweep $32\to512$,
the unstructured MLP diverges at every width and the factored symplectic model inverts at every width. Non-factored
conservation worsens with capacity. Factored potential, dissipative
port, and necessity penalty compose without interference under dissipation, with
the necessity-fidelity effect muted in-regime (shortcut $13\%$ vs the earlier $42\%$).

\textbf{Calibration of in-regime planning.} Planning an impulse through
each model's rollout and executing it in the true physics, structured models'
plans transfer (executed $\approx$ predicted) while unstructured MLP misestimates its
own plans. The unstructured model reports the lowest predicted error but the
highest realised error. Because confidence intervals for realised error overlap,
the supported result concerns prediction--execution calibration rather than a
control-performance advantage.
\begin{figure}[H]
  \centering
  \includegraphics[width=0.5\linewidth]{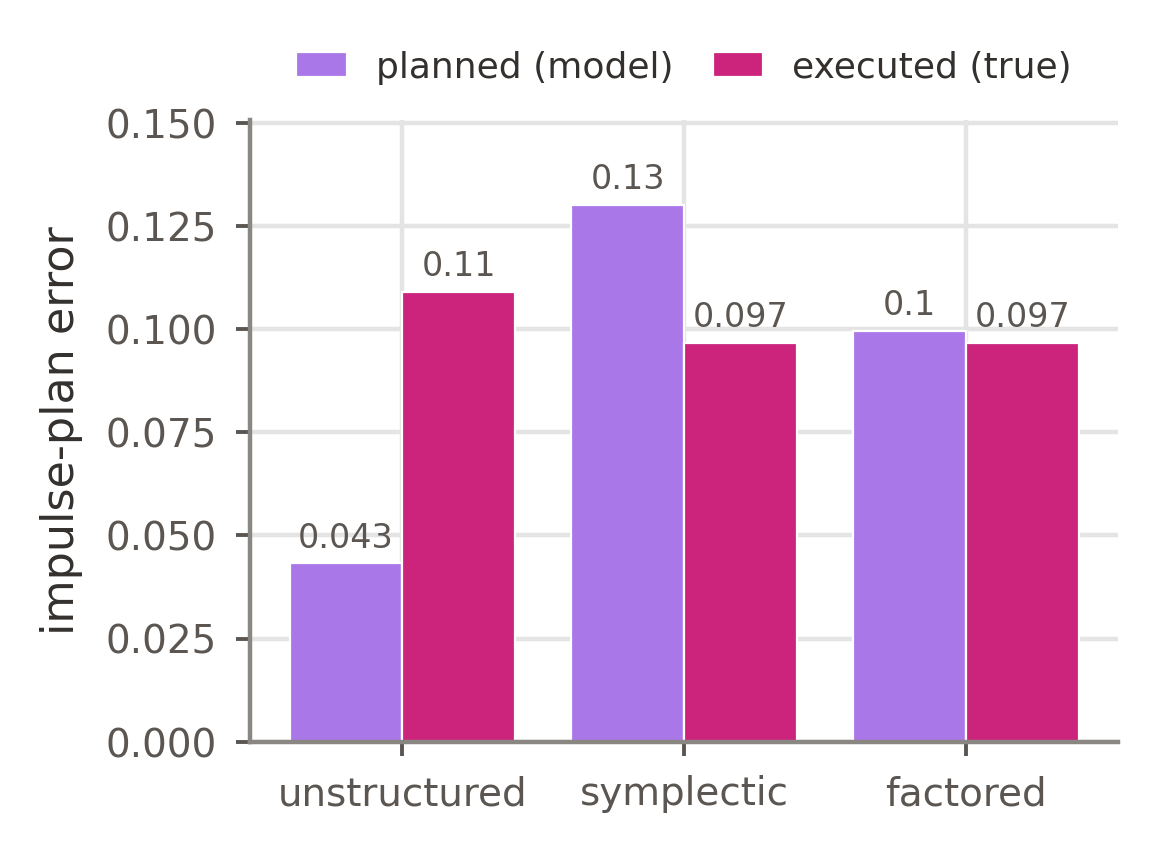}
  \caption{Planned vs executed error: structured
plans transfer; the unstructured model misestimates its own plans.}
\end{figure}

\textbf{Rendering and noise screen.} A structured model's rendered rollout is physically
consistent (re-perceived energy conserved) whereas the unstructured model's bodies
jump discontinuously, a re-visualisation of the stability gap in pixels, not a benchmark. Under
i.i.d.\ initial-state noise to $30\%$ of the state s.d., the factored symplectic model keeps
regime-match 1 and inversion nMSE ${\sim}0.1$, a screen (not perception error)
that motivated the move to pixels.

\subsection{F. Beyond conservation: stochastic dissipation and driven injection (extends \S\ref{sec:r5})}
\label{app:F}

This subsection extends \S\ref{sec:r5} using the passive-friction formulation of
Eq.~\eqref{eq:generic} and the temperature metric of App.~\ref{app:formal}. All
aggregates are mean [95\% CI] over $\ge3$ seeds; temperature error is capped at
$100$ and compared against the true simulator's \emph{measured} equilibrium or
steady-state temperature which targets agreement with the discrete reference rather than cancellation of model-dependent integration bias.

\textbf{Structured fluctuation--dissipation models approach the target
temperature; guarantee--invertibility tension.} On the confined Langevin system (three bodies,
three seeds, four $(\gamma,T)$ settings), a structured model with a
positive-semidefinite friction and fluctuation--dissipation noise reaches the
target temperature with relative error $0.19$ ($\le0.10$ at strong drag), while an
unstructured neural SDE \emph{diverges} ($\hat T$ error ${\sim}10^{3}$--$10^{4}$)
and a reversible-only model \emph{never thermalises} ($11.4$): the
conservation$\to$stability analogue for the thermal case (Table~\ref{tab:f29}).
The same table shows the tension: the guarantee-side model
has the lowest temperature error among the evaluated models but injects energy
only at the reversible baseline
(injection $0.27\approx0.29$, i.e.\ it \emph{cannot} invert), whereas the
factored-friction model injects ${\sim}7\times$ the baseline ($1.91$) at the cost of
being a worse thermostat without the passive-channel sign guarantee (Prop.~\ref{prop:tension}).

\begin{figure}[H]
  \centering
  \begin{subfigure}[t]{0.46\textwidth}
    \centering
    \includegraphics[width=\linewidth]{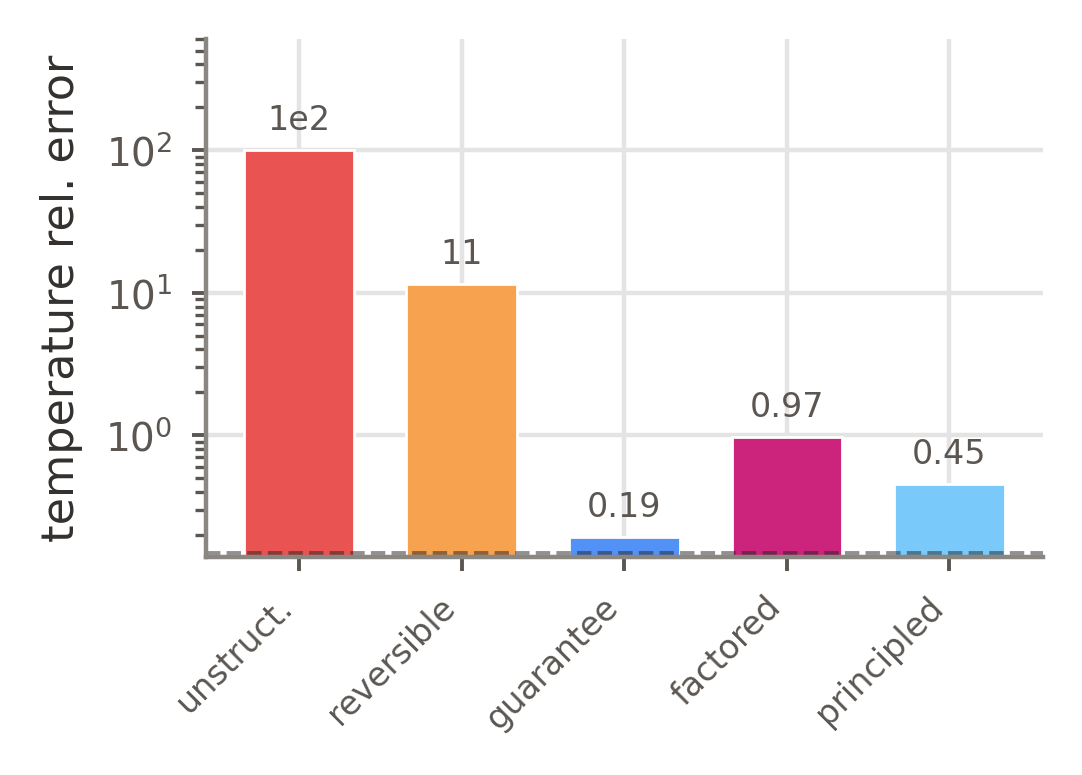}
    \caption{temperature error by model structure}
  \end{subfigure}\hfill\begin{subfigure}[t]{0.46\textwidth}
    \centering
    \includegraphics[width=\linewidth]{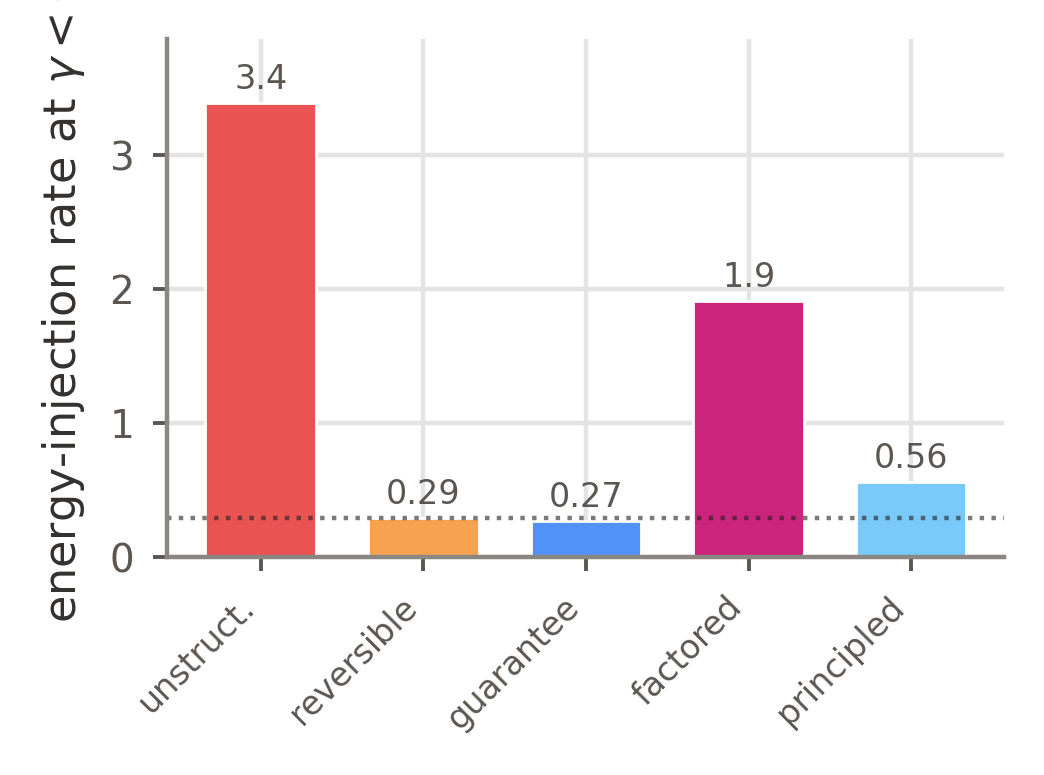}
    \caption{guarantee--invertibility trade-off}
  \end{subfigure}
  \caption{The structured FDT model has the lowest temperature
    error; the unstructured model diverges and the reversible-only model does not
    thermalise. The positive-semidefinite model cannot reverse its passive
  friction channel, whereas the signed factored model can.}
\end{figure}

\begin{table}[H]
  \centering
  
  \caption{Temperature error and energy-injection rate at $\gamma<0$
    (confined Langevin; injection near the reversible baseline $0.29$ means ``cannot
  invert'').}
  \label{tab:f29}
  \begin{tabular}{@{}lcc@{}}\toprule
    model & temperature rel-err & injection rate ($\gamma<0$) \\\midrule
    unstructured neural SDE & $100$ (reporting cap; diverges) & not evaluated \\
    reversible-only & $11.4$ (does not thermalise) & $0.29$ (baseline) \\
    structured, guarantee ($M\succeq0$) & $\mathbf{0.19}$ & $0.27$ (cannot invert) \\
    factored friction (invertible) & $0.97$ & $\mathbf{1.91}$ (injects, no guarantee) \\
    principled full $M=LL^\top$ & $0.46$ & $0.56$ \\
    \bottomrule
  \end{tabular}
\end{table}

\textbf{Full and diagonal friction parameterisations.} The full
positive-semidefinite operator $M=LL^\top$
(Cholesky) with correlated fluctuation--dissipation noise retains the qualitative
thermalisation result but has higher temperature error ($0.46$ versus $0.19$) and is more
difficult to train. In these mechanical experiments, the diagonal
parameterisation therefore performs better; evaluation in non-mechanical
dissipative domains remains future work.

\textbf{Dissipation in MuJoCo, with linear and nonlinear friction.} On the MuJoCo double pendulum in generalised coordinates, the
structured dissipative port reproduces the true energy-shed rate for both linear
and nonlinear ($v^2$) friction, while an energy-conserving control under-sheds
(Table~\ref{tab:f31}). The unstructured model also fits smooth deterministic
friction. The comparison is therefore between the structured model and the
conserving control. The distinction from the unstructured model occurs in the
stochastic temperature results above. Under anti-friction ($\gamma<0$),
the port injects energy at rate $+10.1$ for linear friction but diverges under
nonlinear friction in all seeds. This instability is consistent with
Prop.~\ref{prop:tension}.

\begin{figure}[H]
  \centering
  \begin{subfigure}[t]{0.46\textwidth}
    \centering
    \includegraphics[width=\linewidth]{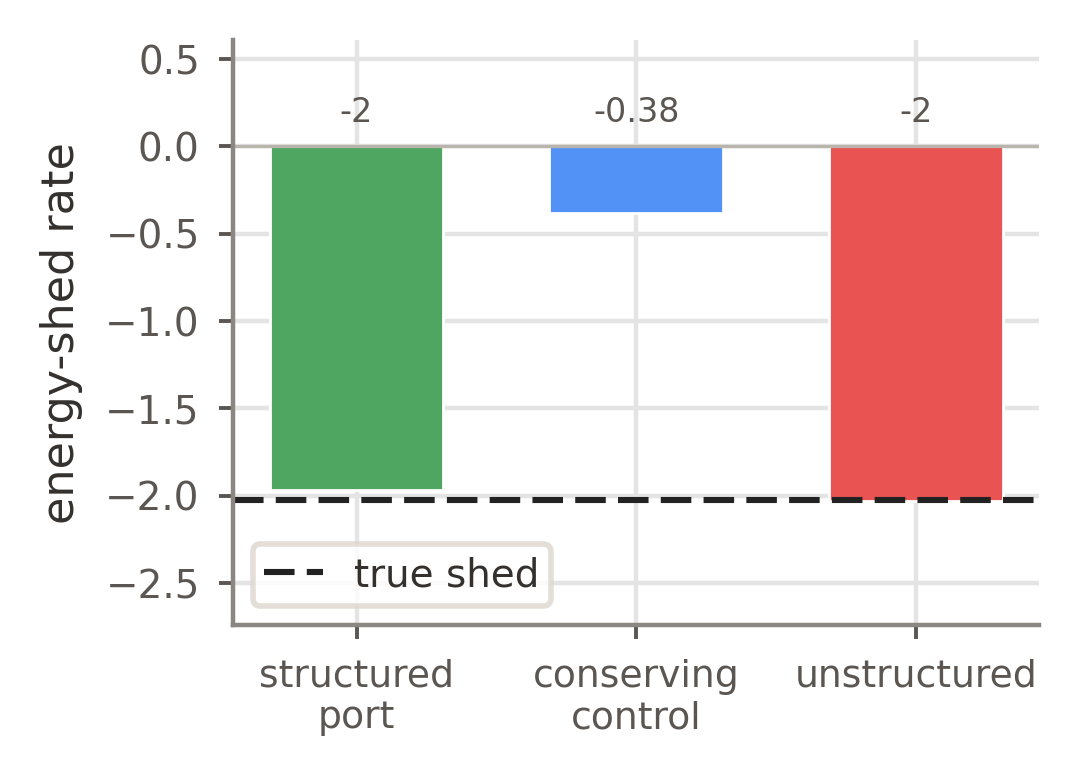}
    \caption{linear friction}
  \end{subfigure}\hfill\begin{subfigure}[t]{0.46\textwidth}
    \centering
    \includegraphics[width=\linewidth]{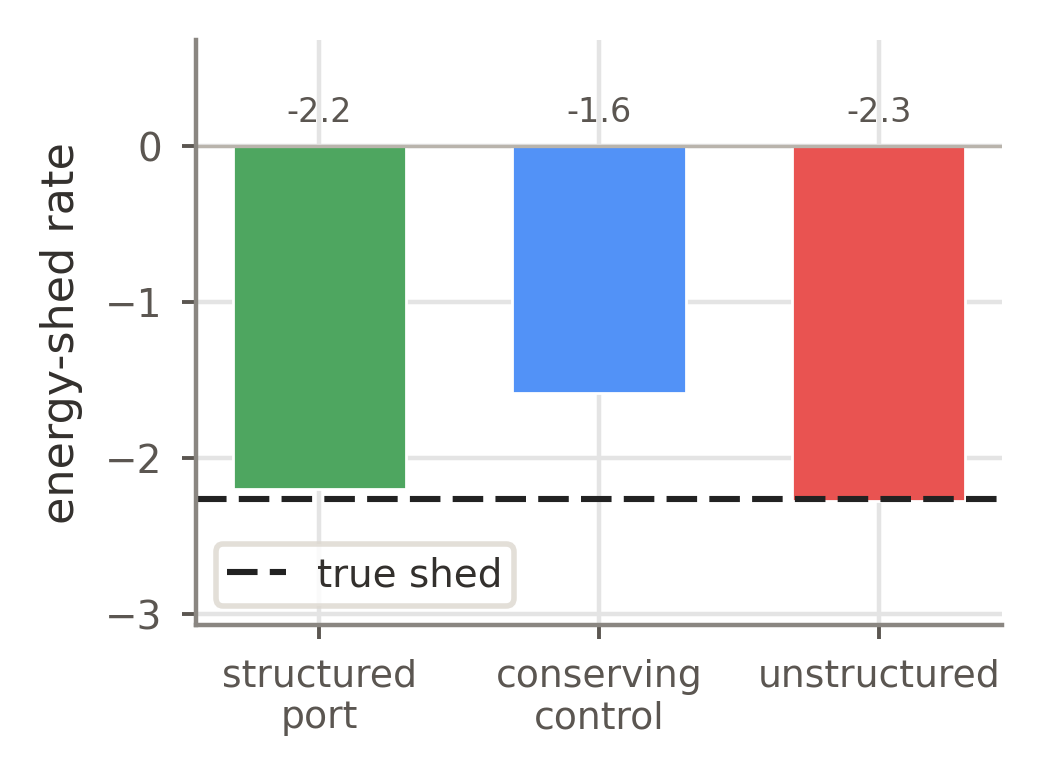}
    \caption{nonlinear ($v^2$) friction}
  \end{subfigure}
  \caption{The structured port sheds the true
    energy-loss rate on the engine for linear and $v^2$ friction; the conservation
  control under-sheds.}
\end{figure}

\begin{table}[H]
  \centering
  
  \caption{Energy-shed rate (negative $=$ losing energy), model vs true.}
  \label{tab:f31}
  \begin{tabular}{@{}lccc@{}}\toprule
    friction & true shed & structured port & conserving control \\\midrule
    linear & $-2.03$ & $\mathbf{-1.97}$ & $-0.38$ ($\sim$19\%) \\
    nonlinear ($v^2$) & $-2.26$ & $\mathbf{-2.20}$ & $-1.59$ \\
    \bottomrule
  \end{tabular}
\end{table}

\textbf{Energy injection as a factored drive generalises over strength and
sign.} Modelling energy injection as an explicit factored \emph{drive}
$A$ (Prop.~\ref{prop:driven}), a structured model reaches the correct driven
non-equilibrium steady state $T_{\text{eff}}=T\gamma/(\gamma-A)$ across trained and
held-out drive strengths and the never-seen cooling sign (rel-err $\le0.094$,
Table~\ref{tab:f33}), while a control without drive input remains
at the bath temperature (rel-err climbing to $0.68$ at the extremes) and the
unstructured model diverges. Beyond the friction ($A\ge\gamma$) there is no steady
state and the driven models exhibit unbounded growth (diverged fraction: factored drive
  $67\%$, unstructured $100\%$, control without drive input $0\%$ because it does not represent
the drive), consistent with the expected unbounded energy growth when net
damping is non-positive. This parallels the gravity-coupling result: a factored
external drive generalises across sign by construction, whereas a friction-sign flip is
thermodynamically ill-posed and should not generalise at all.

\begin{figure}[H]
  \centering
  \begin{subfigure}[t]{0.46\textwidth}
    \centering
    \includegraphics[width=\linewidth]{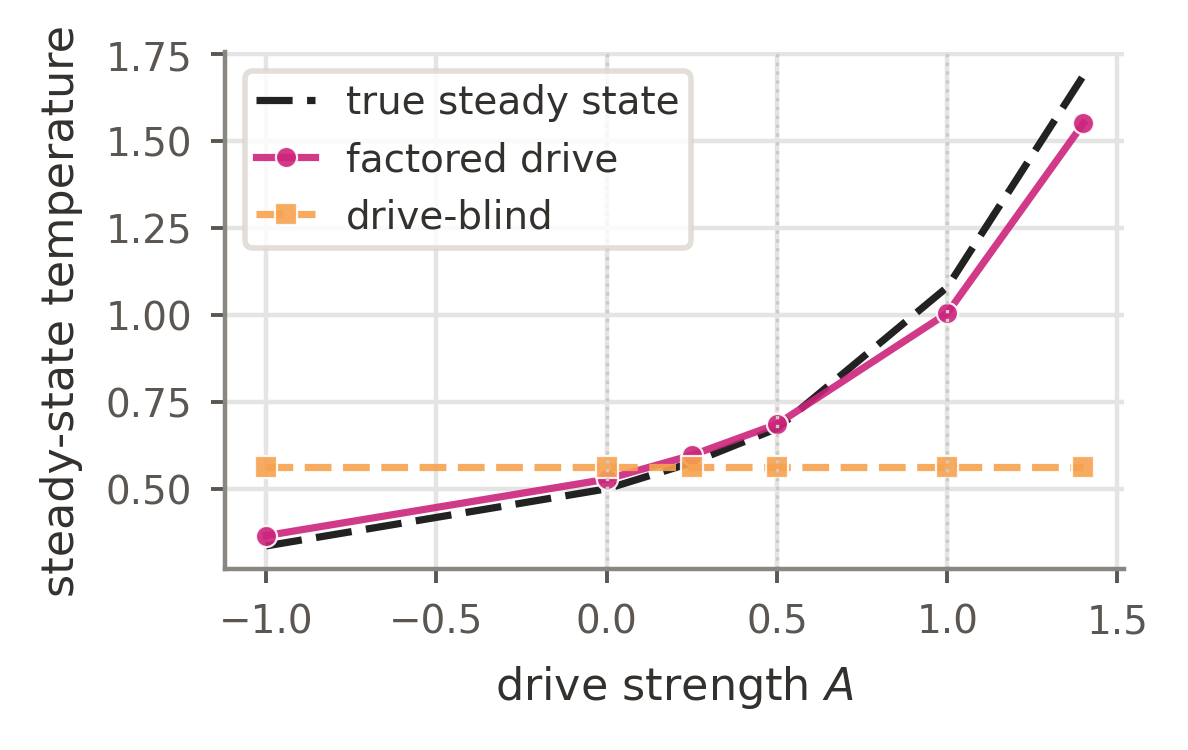}
    \caption{correct driven steady state}
  \end{subfigure}\hfill\begin{subfigure}[t]{0.46\textwidth}
    \centering
    \includegraphics[width=\linewidth]{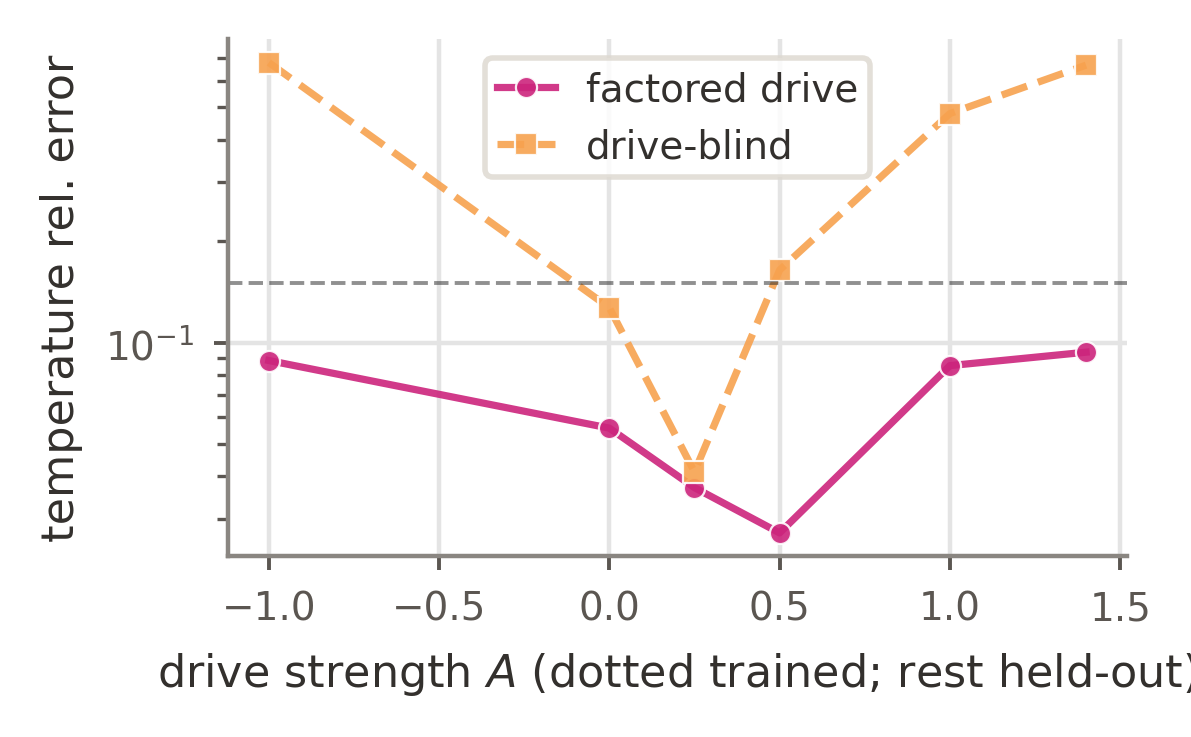}
    \caption{generalising over drive strength/sign}
  \end{subfigure}
  \caption{The factored-drive model follows the
    true driven-steady-state curve across drive strengths and both signs; the
  control without drive input cannot track it.}
\end{figure}

\begin{table}[H]
  \centering
  
  \caption{Driven-steady-state temperature rel-err ($*$ $=$ trained
  drive; rest held-out).}
  \label{tab:f33}
  \begin{tabular}{@{}lcccccc@{}}\toprule
    model & $A{=}{-}1$ & $A{=}0^{*}$ & $A{=}0.25$ & $A{=}0.5^{*}$ & $A{=}1^{*}$ & $A{=}1.4$ \\\midrule
    factored drive & $\mathbf{0.088}$ & $\mathbf{0.056}$ & $\mathbf{0.037}$ & $\mathbf{0.027}$ & $\mathbf{0.085}$ & $\mathbf{0.094}$ \\
    control ($A$ omitted) & $0.68$ & $0.13$ & $0.041$ & $0.16$ & $0.48$ & $0.67$ \\
    unstructured & \multicolumn{6}{c}{diverges (rel-err $100$ at all $A$)} \\
    \bottomrule
  \end{tabular}
\end{table}

\subsection{G. Misspecification and stronger baselines (extends \S\ref{sec:r2} and \S\ref{sec:r4})}
\label{app:G}

\textbf{Inaccurate counterfactuals under coupling-power misspecification.} Every earlier
inversion used a prior whose form matched the law. Here we deliberately
misspecify it (softened
3-body gravity, factored prior $V=G\sum\phi_\theta$, force linear in $G$, 3 seeds). When the true coupling
is instead \emph{even} in $G$ (force $\propto G^2$, so flipping $G$ leaves it unchanged), the model fits
in-distribution as well as the matched case yet predicts a sign flip the truth never
performs, with no in-distribution warning. The correct $G^2$ structure restores
the expected counterfactual, and a mismatched
\emph{interaction} structure (a 3-body term under a pairwise prior) instead produces poor
in-distribution fit (Table~\ref{tab:f34}). Parameter factoring therefore depends
on the assumed functional form: coupling-power misspecification is not detected
by in-distribution error here, whereas interaction-structure misspecification
produces elevated in-distribution error.

\begin{table}[H]
  \centering
  
  \caption{Misspecification series (nMSE; 3 seeds). ``true-cf differs'' ${\approx}0$ means the
  true counterfactual does not change under the flip.}
  \label{tab:f34}
  \setlength{\tabcolsep}{4pt}
  \begin{tabular}{@{}llccc@{}}\toprule
    law & prior & in-dist nMSE & counterfactual nMSE & true-cf differs \\\midrule
    linear (matched) & factored & $0.010$ & $\mathbf{0.003}$ & $0.33$ \\
    quadratic (power misspec) & factored & $0.014$ & $\mathbf{0.343}$ & $0.0$ \\
    quadratic & factored-$G^2$ & $0.013$ & $\mathbf{0.022}$ & $0.0$ \\
    threebody (struct.\ misspec) & factored & $\mathbf{0.43}$ & $0.34$ & $1.23$ \\
    \bottomrule
  \end{tabular}
\end{table}

\textbf{Transformer and domain-randomisation comparisons.} A
permutation-equivariant set-transformer (bodies as tokens, self-attention), a
relational unstructured model, \emph{diverges} over 2000 steps (${\sim}2$ orders more slowly than the MLP,
but still unbounded) and, evaluated at the same 150-step inversion horizon, does not
invert (regime-match $0$, counterfactual nMSE ${\sim}0.3$;
Table~\ref{tab:f35}). Thus, stability is associated with symplectic structure
and inversion with factoring. The same distinction is observed with this
additional baseline. Regime-match is primary here. At 50 steps, the attractive
and repulsive trajectories remain close, producing a small transformer nMSE
despite regime-match $0$. This value reflects the short horizon rather than
inversion. Domain randomisation, in which the black-box model is trained on both
signs of $G$, inverts (regime-match $1$) because the counterfactual is then
in-distribution, but it still diverges, so it gains no stability. The factored
model is the only evaluated model that
both inverts (zero-shot) and stays bounded.
\begin{figure}[H]
  \centering
  \includegraphics[width=0.5\linewidth]{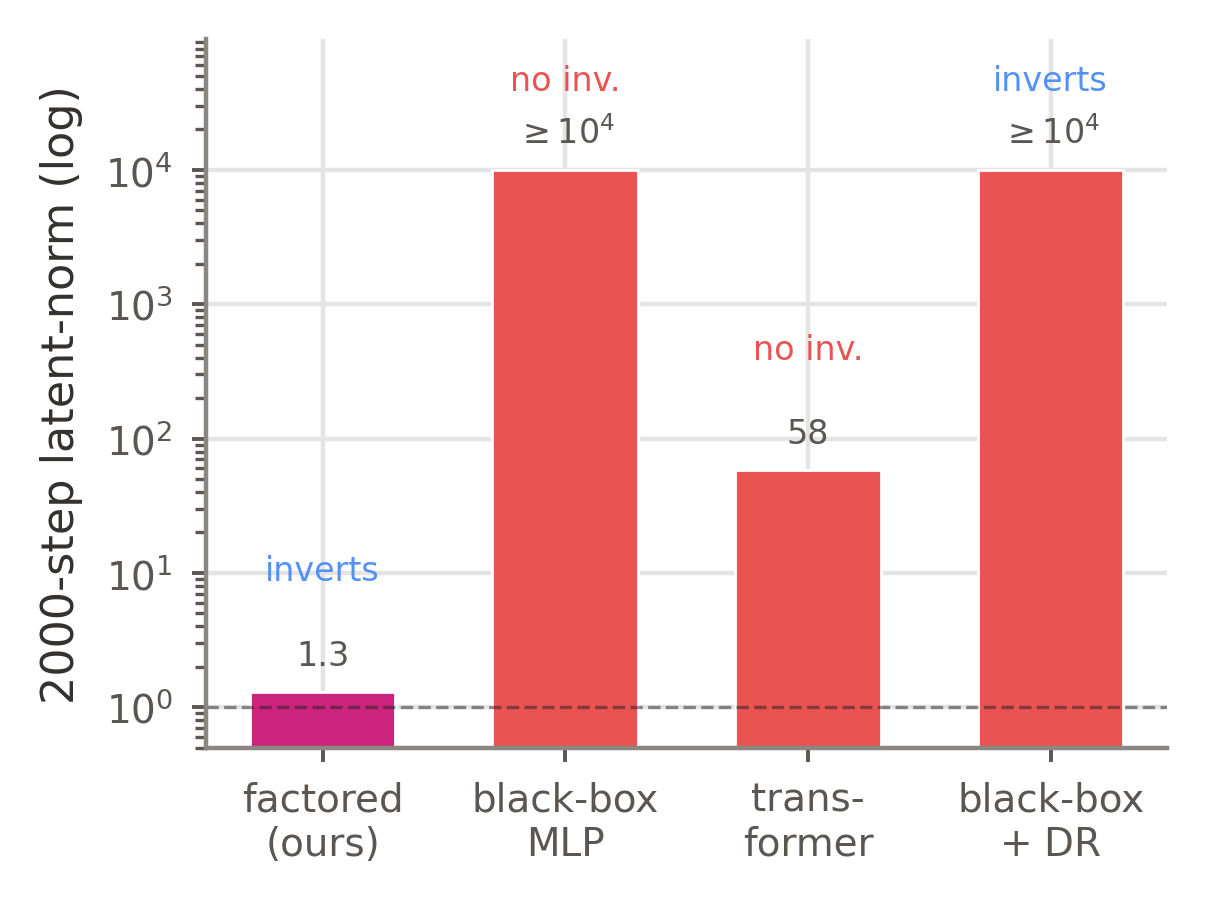}
  \caption{The factored model is bounded
  and inverts; the transformer diverges and does not invert; domain randomisation
inverts after training on both coupling signs but diverges.}
\end{figure}

\begin{table}[H]
  \centering
  
  \caption{Transformer + domain-randomisation (3 seeds; inversion read at the 150-step
  horizon). Regime-match is primary; nMSE secondary.}
  \label{tab:f35}
  \begin{tabular}{@{}lcccc@{}}\toprule
    model & 2000-step & inversion & counterfactual & inverts short- \\
    & stability & regime-match & nMSE (150-step) & horizon (50)? \\\midrule
    factored HNN & $\mathbf{1.3}$ & $\mathbf{1}$ & $0.002$ & yes \\
    black-box MLP & $10^4$ (div.) & $0$ & $3.6{\times}10^3$ & yes \\
    transformer & $58$ (div.) & $\mathbf{0}$ & $0.29$ & no \\
    black-box $+$ dom.-rand. & $10^4$ (div.) & $1$ & $1.5$ & yes \\
    \bottomrule
  \end{tabular}
\end{table}

\subsection{H. Grounding a decode-free objective, and the object-binding boundary (extends \S\ref{sec:r5})}
\label{app:H}

\textbf{Encoder readout, scaling, and image-SSL.}
Is the advantage in the encoder or in the readout? VideoMAE's $0.13$ improves on our own
CNN's $0.45$, but that comparison changes two things at once, namely the encoder \emph{and}
the readout (our CNN used the earlier pooled head, whereas the public encoders use the spatial
cross-attention). We deconfound by running our \emph{own} trained CNN, frozen, through the
\emph{same} spatial readout: it improves to $0.34$, so the localisation-aware readout
explains part of the gap, but it stays $2.6\times$ worse than VideoMAE under the
identical readout. The ``better than our own encoder'' claim therefore softens from
$\sim\!5\times$ to a readout-matched $\sim\!2.6\times$, and the residual earns a mechanism:
the public video-SSL representation is richer than our task-trained CNN's, not
merely better-read. The readout-limited reading was then tested at scale: the same
frozen-encoder harness on V-JEPA2 at ViT-L ($300$M), ViT-H ($600$M), and ViT-g ($1$B),
identical spatial readout and head budget. The best-rung gain over ViT-L is $1.02\times$
(factored inversion nMSE $0.585 \to 0.573$, both rungs $10$ seeds; ViT-H $0.60$ at $3$)
against a pre-registered $1.5\times$ weak-gain bar, and the escalation itself is
instructive: a $3$-seed ViT-g interim read $3\times$ better ($0.195$) until the
pre-registered $10$-seed escalation surfaced the rung's own heavy-tail seed ($3.6$),
collapsing the gain. The full dissociation reproduces at $1$B (factored drift ${\sim}3$
bounded on $10/10$ seeds; unstructured $0.77$, diverges $5/10$): perception-side scale does
not buy the physics readout here.

\textbf{Image-SSL and the momentum axis.} A frozen \emph{image}-SSL
encoder, DINOv2 \citep{oquab2023dinov2}, transfers only partially under the spatial
readout: the stability dissociation holds (factored drift $2.8$ vs unstructured $6.8$) but
the inversion-nMSE gap is small ($0.62$ vs $0.65$), consistent with a per-frame encoder
exposing position but not momentum. We tested that mechanism directly: a readout that reads
a per-frame position from DINOv2's tokens and \emph{differences} it across strided frames
(momentum lives in differences, at the representation level) closes the gap to nMSE $0.096$, matching VideoMAE, while its
unstructured twin stays at $0.52$. So image-SSL encoders \emph{do} transfer fully once
momentum is differenced rather than read from a single frame. The apparent
``video-encoders-only'' advantage resolves into one momentum-visibility law. (A global pooled readout remains a confound throughout: VideoMAE $+$ pool
is null, nMSE $0.76$.) The law was then \emph{quantified} on the cited model itself (a
  supervised $(q,p)$-readout probe at matched head budget, $3$ seeds $\times$ $4$ temporal
configurations): position readout is configuration-independent (nMSE $0.04$--$0.09$
everywhere), while momentum readout splits exactly along the temporal-difference channel:
V-JEPA2's native multi-frame clip reads momentum at $0.54$ against $1.35$ for the same clip
with its centre frame repeated (temporal information ablated at identical compute; a
$2.5\times$ advantage, pre-registered bar $2\times$), and DINOv2-single-frame sits at $1.29
\approx$ the frame-repeated video encoder (ratio $0.95$). A video encoder therefore buys
exactly the temporal-difference channel the momentum boundary demands, that is, the channel
rather than the architecture (Table~\ref{tab:grounding}). This is encoder transfer on synthetic renders, not
real-video object physics (joint learned detection-and-binding remains the open rung).

This block develops the decode-free JEPA-style predictor column of the pixel comparison into a grounded family
(3 seeds, matched encoder/data/objective/anchor/budget, so the only variable within a
pair is the dynamics module). The unstructured dynamics head has $36{,}748$
parameters and the factored fixed-kinetic head $4{,}417$ (${\sim}8\times$ smaller),
so every hybrid win is against a strictly larger unconstrained competitor. CI lower
bounds are clipped at the physical floor of each non-negative metric for display.

\textbf{A label-free image anchor lets the physics prior separate on a
decode-free objective.} A decode-free JEPA-style objective (predict
future \emph{latents} against an exponential-moving-average target; no decoder)
with the physics prior alone does not invert: the fixed-kinetic factored
JEPA-style predictor fails (nMSE $1.46$) and its latent grows (${\sim}22$), reproducing the
empirical recovery condition that the prior needs a fixed latent gauge to act on. Supplying that
gauge \emph{without a renderer or state labels} (a soft-argmax image centroid for
position and a central finite difference for momentum) separates dynamics
structure from grounding cleanly (Table~\ref{tab:f36}, Fig.~\ref{fig:c4jepa}): the
anchored unstructured model never inverts (match $0$, nMSE $0.68$), while the
anchored factored hybrid inverts on every seed (match $1$, nMSE $0.40$) and stays
bounded ($1.39$), ${\sim}41\%$ lower error under matched everything. This partially
closes the residual position-grounding boundary (position grounding without a decoder). Two
limits: the anchor assumes calibrated per-body channels, and its centroid/finite-%
difference grounding is noisier than the renderer, so the hybrid ($0.40$) trails the
renderer-grounded factored fixed-kinetic model ($0.12$): a reliable \emph{regime} inversion
decoder-free, not near-oracle fidelity. A grounding-quality control isolates which
of the two differences drives that gap: replacing the noisy centroid with an
exact-position anchor (same world coordinates and finite-difference momentum, minus
centroid and occlusion noise), holding the JEPA objective fixed, does \emph{not} move
the number (nMSE $0.405$ vs $0.401$, match $1$ every seed), so the anchor is
exonerated and the residual gap to $0.12$ is a property of the decode-free objective,
not of grounding quality. A composited-RGB ladder confirms the
operative ingredient is persistent temporal grounding, not per-body channels: naive
RGB differencing breaks both variants, but a seven-frame robust tracker (palette
  demixing, visibility-weighted trajectory fitting; a controlled probe, not a general
perception method) restores the hybrid advantage (nMSE $0.45$ vs $0.82$, every seed).
\begin{figure}[H]
  \centering
  \includegraphics[width=0.62\linewidth]{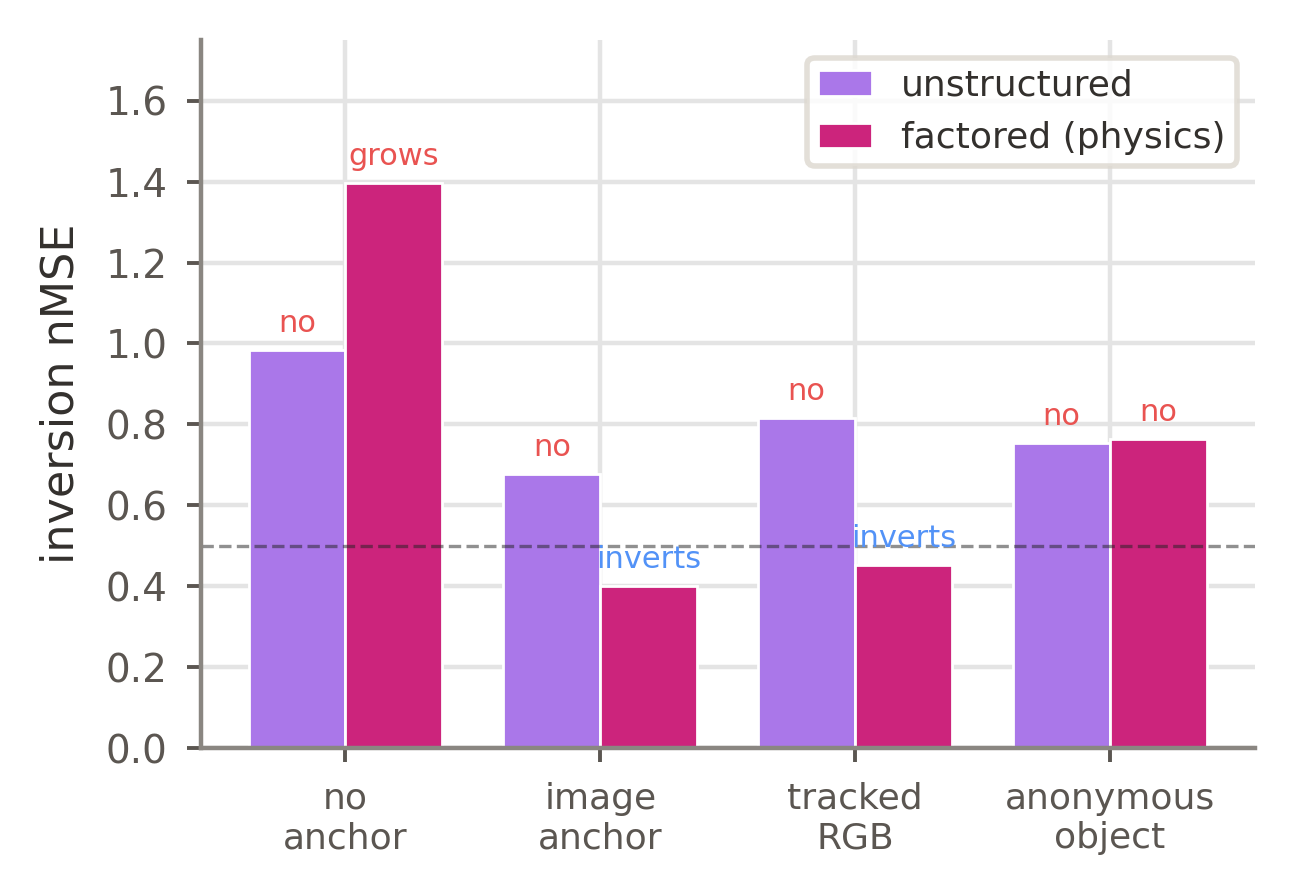}
  \caption{Decode-free inversion nMSE across the
  grounding ladder: the image-anchored factored hybrid inverts (match $1$) where its
  unstructured twin does not, on channels and on tracked RGB; with object identity
removed the advantage disappears.\label{fig:c4jepa}}
\end{figure}

\begin{table}[H]
  \centering
  
  \caption{Decode-free JEPA family (3 seeds). ``match'' is anti-physics
  regime-match; nMSE the inversion error; growth the $60$-step latent-norm ratio.}
  \label{tab:f36}
  \begin{tabular}{@{}lccc@{}}\toprule
    variant (grounding) & dynamics & match & inversion nMSE \\\midrule
    none (no anchor) & unstructured & $0$ & $0.99$ \\
    none (no anchor) & factored & $1$ & $1.46$ (grows) \\
    image anchor & unstructured & $0$ & $0.68$ \\
    image anchor & factored & $\mathbf{1}$ & $\mathbf{0.40}$ \\
    tracked RGB & unstructured & $0$ & $0.82$ \\
    tracked RGB & factored & $\mathbf{1}$ & $\mathbf{0.45}$ \\
    anonymous object & unstructured & $0.33$ & $0.76$ \\
    anonymous object & factored & $0.33$ & $0.76$ \\
    \bottomrule
  \end{tabular}
\end{table}

\textbf{Removing object identity relocates the blocker to learned object
binding.} When per-body channels \emph{and} colour identity are removed
(all bodies share one monochrome frame, bound across time by anonymous peak
detection and matching), the hybrid advantage vanishes: both the structured and
unstructured variants match the flipped law on only $1/3$ seeds with
indistinguishable error (nMSE ${\sim}0.76$), because close encounters make the
momentum/identity assignment ambiguous. The physics dynamics are not the blocker;
learned persistent \emph{object binding} (long-context slots or tracks) is. The label-free anchor moves
the boundary from decoder-free \emph{position} grounding down to \emph{object}
binding, and this anonymous-object result marks where the hand-engineered binder fails; the next two results attack that
blocker directly.

\textbf{A learned binder resolves the anonymous-object boundary, and the physics
prior disambiguates the learned binding.} The anonymous-object collapse traces to
the hand-engineered binder: it detects $N$ blobs and assigns identity across frames by
a hard minimum-acceleration match, which is ambiguous at a close encounter, so the
prior is grounded on mis-bound $(q,p)$. Replacing that match with a learned,
differentiable, confidence-aware soft-assignment (a ``learned tracker''; the detector
  is unchanged, and the binder is the encoder, its grounded centroids rolled and matched
under the decode-free objective) resolves it (Table~\ref{tab:f38},
Fig.~\ref{fig:c4bind}): the learned-tracker factored hybrid inverts on every seed
(match $1$, nMSE $0.15$, bounded growth ${\sim}1.4$), while its unstructured twin does
not (match $0$, nMSE $0.64$) and the hand-engineered floor does not (match $0.33$, nMSE
$0.76$). Extended to the ten-seed standard, the dissociation is exact:
factored inverts on $\mathbf{10/10}$ seeds and the unstructured control on $\mathbf{0/10}$
(one-sided exact binomial $p\!\approx\!10^{-3}$ each). This is the same double dissociation once
more, now on \emph{learned} binding: with the same binder, factored inverts and
unstructured fails. It holds whether the binder
is trained through the physics loss ($0.152$) or self-supervised on temporal smoothness
and then frozen ($0.155$), so physics need not train the binder end-to-end to obtain
the separation; and $0.15$ is below the channelised hybrid ($0.40$), approaching the
renderer-grounded $0.12$, because the compact anonymous disks detect cleanly and only
the matching needed fixing.
\begin{figure}[H]
  \centering
  \includegraphics[width=0.55\linewidth]{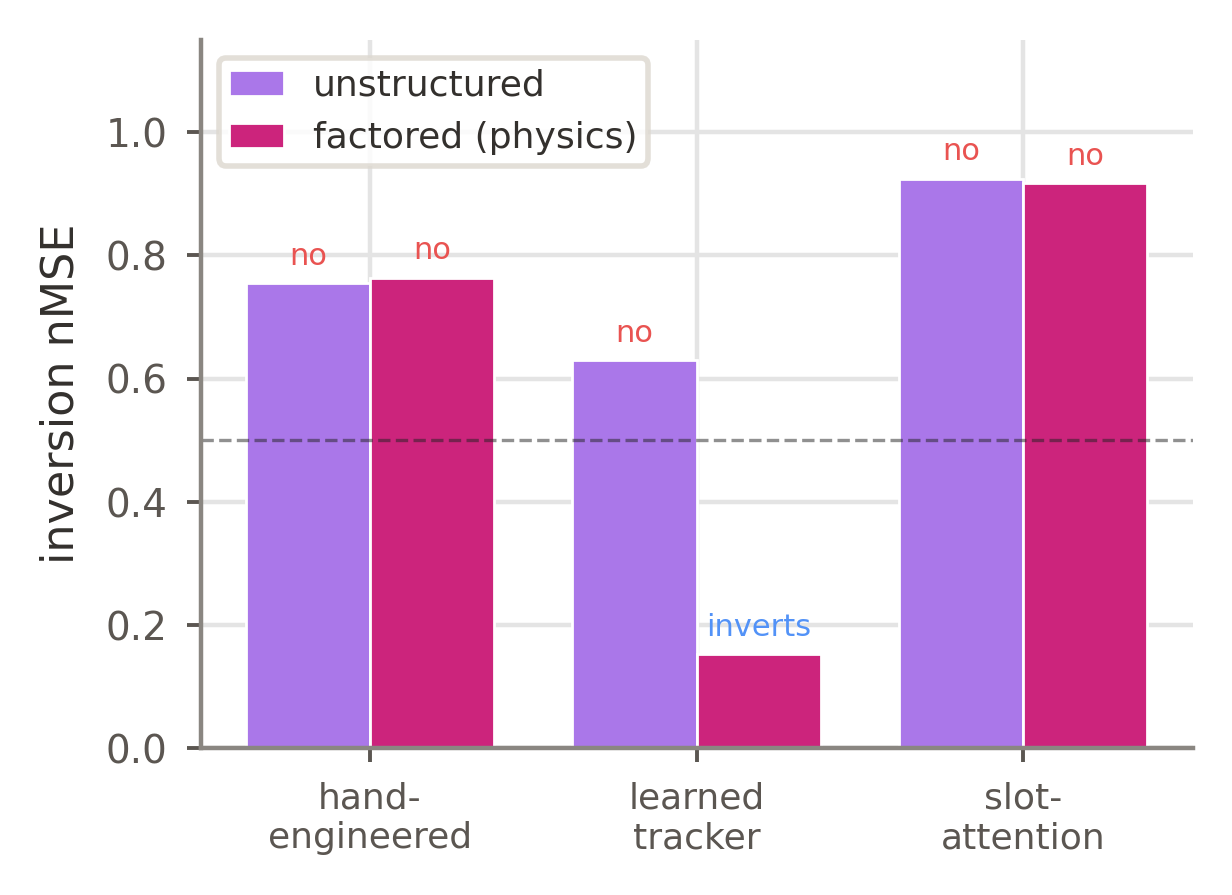}
  \caption{Inversion nMSE on the anonymous world
  across the binding ladder: only the learned-tracker factored hybrid inverts (below the
  $0.5$ guide), where the hand-engineered floor and slot-attention do not.%
\label{fig:c4bind}}
\end{figure}

\begin{table}[H]
  \centering
  
  \caption{Learned binding on the anonymous world (3 seeds). ``match''
  is anti-physics regime-match; nMSE the inversion error.}
  \label{tab:f38}
  \begin{tabular}{@{}llccc@{}}\toprule
    binder & dynamics & coupling & match & inversion nMSE \\\midrule
    hand-engineered & factored & --- & $0.33$ & $0.76$ \\
    learned tracker & unstructured & e2e & $0$ & $0.63$ \\
    learned tracker & factored & e2e & $\mathbf{1}$ & $\mathbf{0.15}$ \\
    learned tracker & factored & twostage & $\mathbf{1}$ & $\mathbf{0.16}$ \\
    slot-attention & factored & e2e & $0$ & $0.92$ \\
    slot-attention & unstructured & e2e & $0$ & $0.93$ \\
    \bottomrule
  \end{tabular}
\end{table}

\textbf{Joint learned detection-and-binding is the next rung.} A
slot-attention binder, which must learn \emph{detection} as well as assignment, does
not reach the inversion (match $0$, nMSE $0.9$--$1.0$ across factored/unstructured and
both couplings). Learning the binding on top of a fixed detector is therefore solved;
learning detection and binding jointly and decode-free is not. The perception
frontier accordingly moves one rung further: from persistent object binding, now
resolved by a learned soft-assignment, to \emph{joint} learned detection-and-binding. The learned quantity in the tracker is only the assignment ($4{,}609$ params,
detector fixed); the slot-attention binder is $105{,}728$ params.

\textbf{The grounded approach transfers to a frozen public video encoder we did not
build.} To test whether the pixel result depends on our specific
encoder, we replace the trained CNN with a frozen public encoder (weights not trained;
only a small head and the dynamics train) in the tracked setting, reading out
$(q,p)$ with $N$ object queries that cross-attend over its patch tokens. On two frozen
\emph{video}-SSL encoders, VideoMAE \citep{tong2022videomae} and the V-JEPA2 model
\citep{vjepa2_2025} this paper's introduction cites (both arms now at the full 10-seed
standard: factored inverts $10/10$, exact binomial $p=0.001$), the factored approach inverts on
every seed with bounded drift while its unstructured twin is far worse and unstable.
Over $10$ seeds VideoMAE inverts at nMSE $0.13$ (oracle $0.067$) against $0.65$
unstructured (which diverges on a seed); V-JEPA2 inverts at $0.59$ against $0.82$
unstructured (the 10-seed expansion widens the 3-seed $0.25$ with two heavy-tail
seeds; the unstructured twin diverges $5/10$, drift up to ${\sim}3100$). This is the
same dissociation on representations we did not build (Table~\ref{tab:grounding}).

\textbf{The world-model half of V-JEPA2, characterised, with the positive control as the
finding.} Our external-encoder result tested the public model's \emph{encoder}; here we test its
\emph{predictor}: the official action-conditioned V-JEPA2-AC (ViT-g encoder + 300M frame-causal
predictor, frozen, official release weights), rolled autoregressively on our rendered physics
(native 256px, 16-frame context at the dynamics cadence) with action tokens held at no-op, in
both the zero pose and the release example's real end-effector pose (disclosed; a no-op stream
is itself part of the domain gap for a model trained on robot manipulation). A pre-registered
positive control gates the physics probes: the predictor's one-step latent
prediction on our renders must beat a copy-last-frame baseline by a stated margin before any
probe is interpretable. It does not: pred/copy $= 0.978$--$0.986$ across $3$ seeds and both
no-op variants, against a certified-fair baseline (consecutive-frame representations sit at
  $82\%$ of the \emph{unrelated}-frame distance, since bodies traverse $8\%$ of the scene per frame,
so there is real predictable signal that copying does not capture). The same gate shows the
\emph{encoder} half does encode the state: our spatial readout recovers $(q,p)$ from its tokens
at $0.28$--$0.43$ nMSE. Per the pre-registered rule, the imagined-rollout probes (energy drift,
counterfactual tracking, second-law flags) are VOID and reported as such, not as model
failures. The finding is that the representation transfers across the domain gap (consistent
with the encoder-transfer result) while the \emph{predictive} half, under no-op conditioning far from its training
distribution, adds nothing over frame persistence. This is a characterisation rather than a strawman: the domain
gap (Droid tabletop $\to$ rendered $N$-body) is the headline variable, the fairness artifact
and both no-op variants are in the record, and a same-domain test would require fine-tuning
the predictor, which is out of scope for a frozen-model claim.

\textbf{Composing a public encoder, learned binding, and the physics prior, the boundary
that localises the residual.} The learned-binder result and the frozen-encoder result each
remove one obstacle separately: the frozen-encoder result keeps the hand-engineered $7$-frame tracker anchor, and the learned-binder result's
detections are \emph{real} image centroids. We composed them so that no part of the perception stack
is author-built: a frozen VideoMAE reading the \emph{anonymous} world, a \emph{learned} soft-argmax
detector replacing the hand-built centroids, and the learned tracker, under a factored prior. It does
\emph{not} compose. The factored cell is unstable across seeds (inversion nMSE $0.8$--$7.5$, latent
  growth $1.2$--$42\times$; the same seed even flips between bounded and divergent under GPU
nondeterminism, i.e.\ it sits on the edge of stability) and never reaches the clean inversion of the learned-binder result
($0.15$) or the frozen-encoder result ($0.13$). The unstructured control is worse (nMSE up to $315$). The factored prior
keeps only a weak stability edge. The boundary is informative: this composition removes \emph{both}
physical-grounding sources that made the two earlier results work (real centroids; a built anchor), and the
physics prior, which can disambiguate binding and transfer across frozen encoders
\emph{given} grounding, does not bootstrap grounding in this frozen-token configuration. This is the joint learned detection-and-grounding
residual at the representation level: the tested joint detection-and-grounding configuration remains unresolved.

\textbf{Damping-rate ordering from video [partial].}
A grounded seven-frame tracker supplies latent $(q,p)$ for a disk world
with conservative or damped dynamics. The model uses the positive port
$p\leftarrow p\exp[-\mathrm{softplus}(a)\Delta t]$.
The effective rate is $r=\mathrm{softplus}(a)>0$; the raw parameter $a$
is not constrained to be nonnegative. This update contracts canonical
quadratic kinetic energy during the isolated port, not arbitrary total energy.

The source reports rate diagnostics $1.25$ for damped video and $0.84$
for conservative video, against reference rates $1.5$ and $0$, across
five seeds. Their ordering is recovered, and a label-free window head
tracks a medium switch. The parameter-to-rate reporting convention
must be checked before interpreting these diagnostics as calibrated
physical rates.
The reported $\sigma$ values are transformed effective rates,
$\sigma=\mathrm{softplus}(\log\sigma)$ read from the trained port (not the raw parameter and
not a fitted diagnostic), in units of inverse physical time at the frame cadence; the
estimator validation below gives the grounding comparison and the untested step-size
sensitivity.
The conservative diagnostic remains nonzero; per-clip arrow-of-time
detection is not demonstrated (prediction nMSE approximately $0.9$).

\textbf{Oracle-grounding and reversed-video controls.}
Replacing grounded pixels with true simulator states leaves the conservative
diagnostic near $0.885$, versus $0.84$ from pixels (five seeds).
Oracle grounding therefore does not eliminate the floor. This directs
attention to model, objective, optimisation, and temporal-resolution
limitations rather than grounding alone; it does not uniquely identify
the cause.
On reversed damped clips, the constrained-port momentum-loss floor rises
from $1.03$ to $1.24$, while the sign-free control changes from $0.84$
to $0.87$. The pooled $1.42\times$ contrast misses the declared
$2\times$ threshold. The retained findings are ordering and a channel
constraint, not accurate conservative-rate recovery or demonstrated
video-level refusal.

\textbf{A grounded energy-balance estimator [pre-registered;
estimator amended before the run, disclosed].} The preceding experiments show
that oracle grounding does not eliminate the learned-model floor; they do not
uniquely identify its cause. A separate estimator uses grounded positions and
the known interaction law. For unit masses and linear velocity drag, it regresses
the windowed energy balance
$\Delta E_j\approx-2\gamma\bar K_j\Delta t$, with $E=K+V(q;G)$.
For general masses under $\dot p_i|_{\rm drag}=-\gamma p_i/m_i$,
the dissipated power is instead $-\gamma\sum_i\|p_i\|^2/m_i^2$.
A time-constant additive measurement bias cancels in the energy difference;
time-varying tracking noise and discretisation error need not cancel.
\emph{Estimator validation.} The learned-rate readout stores the transformed rate
$\sigma=\mathrm{softplus}(\log\sigma)$, applied as $p\leftarrow p\,e^{-\sigma\Delta t}$ with
$\Delta t$ the frame cadence ($M$ simulator steps), so $\sigma$ is a rate per unit physical
time; masses are $1$ in this world, so the $p/m$ and $p/m^{2}$ conventions coincide. The
$\sim0.85$ positive bias of the learned conservative rate is \emph{not} tracking noise: the
oracle-grounded arm (true positions) gives $\sigma_{\rm cons}=0.85$--$0.93$ and the
pixel-grounded arm $0.80$--$0.89$ (five seeds each), so the bias is the co-training transient
described above, which the grounded energy-balance estimator removes
($\hat\gamma_{\rm cons}=-0.006$). The damped rate is recovered at $1.22$--$1.30$ against a
truth of $1.5$ (about $15\%$ low) under both groundings. Sensitivity to the frame cadence was
not tested, so the video result is an ordering and an approximate rate, not a calibrated
measurement.
One physics amendment was made to the pre-registration \emph{before} the full run,
on smoke evidence: the originally specified kinetic-envelope form fails on self-gravitating
worlds, where $K$ \emph{grows} even without drag (virial infall, measured $0.066\to1.52$); the
energy-balance form replaced it with all gates unchanged. Result: pixel-grounded
$\hat\gamma_{\text{cons}}=-0.006$ and
$\hat\gamma_{\text{damped}}=1.63$ (truth $1.5$), both gates pass at $5$ seeds; the port carries
$\sigma=\max(\hat\gamma,0)$, so the estimator supplies the rate and the architecture the
guarantee. The estimator is label-free by construction (per-episode, no medium index), and a
sliding version reads a mid-clip medium switch cleanly ($\hat\gamma$ median $-0.05\to1.62$ across
a conservative$\to$damped concatenation), a label-free discovery pattern
completed at the estimator level. Scope: the meter here uses the known
interaction law $V(q;G)$; on real video, selecting the form that \emph{becomes} the meter is
an open form-selection problem. On time-reversed footage the estimator reads $\hat\gamma=-1.66$ while the
$\sigma\!\ge\!0$ port clamps its carried rate to zero, giving arrow-of-time \emph{rate} measurement
with architectural refusal, extending the classification-only contrast of
\citet{pickup2014arrow} and \citet{wei2018arrow}. The training-floor refusal channel was pursued through two pre-registered designs and is
now \emph{retired as an instrument}: the original sign-free control reached $1.93\times$ against
the $2\times$ bar (not demonstrated under the pre-registered threshold), and the exponential-twin control (the port's
exact unrestricted twin, the cleanest nested pair) inverted the prediction for an identified
reason: a fixed growth rate amplifies the core's error by $e^{|\hat\gamma|\Delta t}$ per step
(compounding $\approx\!6.9\times$ in squared error over the horizon; predicted floor $6.6$,
measured $5.5$), so under exponential growth a training floor measures \emph{error amplification},
not rate-fit, so both models degrade on reversed footage on $10/10$ seeds and the channel cannot
discriminate. This is a metric retired by its own control; the
refusal claim rests on the estimator channel and the architectural guarantee. Finally, a
distill-initialisation probe on the composition boundary (detector distilled from the centroid detector, then released
with EMA, gradient clipping, and $0.1\times$ detector learning rate) \emph{removes the instability} (latent growth $1.2$--$1.3\times$ on every seed, versus the bounded$\leftrightarrow$
divergent flipping) while inversion plateaus at nMSE $0.73$, and the residual is now
\emph{localised}: the cured detector's median position error against the centroid teacher is
$0.71$--$0.75$ world units, larger than the disk radius ($0.6$), so frozen tokens fail at
\emph{localisation} on this anonymous world before binding even starts (measured on the trained
stack; the tokens' information ceiling is not probed). That number is the opening problem for
future work.

\section{Claim--evidence map}
\label{app:claims}
This appendix maps each claim of the paper to the evidence that supports it and the class we
assign that evidence.
\renewcommand{\arraystretch}{1.15}
{
\begin{longtable}{@{}p{0.30\linewidth}p{0.45\linewidth}p{0.16\linewidth}@{}}
  \toprule
  \textbf{Claim} & \textbf{Evidence} & \textbf{Class} \\
  \midrule
  \endfirsthead
  \toprule
  \textbf{Claim} & \textbf{Evidence} & \textbf{Class} \\
  \midrule
  \endhead
  \midrule
  \multicolumn{3}{r}{\emph{continued on the next page}}\\
  \endfoot
  \bottomrule
  \endlastfoot
    In the evaluated conservative systems, conservation by construction is
    associated with lower long-horizon energy error than the evaluated unstructured baselines; neural-ODE finiteness shows that structure is not necessary for bounded finite-horizon states.
    & Bounded vs diverged energy and robustness to tuning; 10k-step saturation.
    & Well-supported. \\
    Within the evaluated model families, law inversion occurs with factored coupling
    and not with conservation alone.
    & Non-factored fails while factored inverts in the reported fidelity-valid comparisons; the double-pendulum engine and DeLaN/SymODEN comparison. Graph-network response labels are descriptive pending short-horizon validity checks.
    & Well-supported. \\
    Under grounded pixel training, a canonical fixed kinetic is additionally
    required for inversion in the evaluated factored models.
    & The discovered-perception and pixel-comparison results (learned-kinetic factored fails like non-factored).
    & Supported. \\
    The results extend to contact dynamics and a generalised-coordinate
    double-pendulum engine.
    & Contacts; the engine with learned $M(q)$.
    & Supported. \\
    The evaluated boundaries include hard contacts, dissipation, reversibility, and
    ungrounded visual-state discovery.
    & Contacts; dissipation; the reversibility flip; perception; the unsupervised pixel branch (invalid).
    & Supported as boundaries. \\
    When an explicit shortcut is available, penalising its use increases reliance on
    the factored pathway and improves inversion fidelity.
    & The shortcut dose--response; the null result without an available shortcut.
    & Supported. \\
    The factored model exhibits object-level transfer and cross-family
    generalisation within the tested ranges.
    & $N$-transfer with the regime persisting to 32; Coulomb and the mixed generator.
    & Supported (degradation reported). \\
    The dissociation is retained under end-to-end pixel training with grounded perception.
    & The end-to-end pixel comparison.
    & Supported, scoped. \\
    Discovered-from-pixels perception does not provide position grounding through the
    objective alone; a label-free image anchor supplies it without a decoder, and a
    learned soft-assignment binder then resolves object binding, leaving joint learned
    detection-and-binding as the open rung.
    & The discovered-perception results show the objective alone fails; a decoder-free image
    anchor grounds position and the physics prior then separates; a hand-engineered
    binder collapses on anonymous objects; a learned tracker $+$ the physics prior
    inverts (match $1$, nMSE $0.15$); slot-attention, learning detection too, does
    not yet reach it.
    & Supported (boundary moved two levels down). \\
    The structured FDT model has the lowest stationary-temperature error among the
    evaluated models.
    & The structured model is correct while the unstructured model diverges and the reversible-only model never thermalises; this holds under noise.
    & Well-supported. \\
    Under the passive-friction formulation of Prop.~\ref{prop:tension}, a global
    positive-semidefinite friction guarantee is incompatible with friction-sign inversion.
    & Prop.~\ref{prop:tension}; the passive channel cannot anti-damp, whereas sign reversal relinquishes that channel guarantee; the double-pendulum engine.
    & Well-supported (with proof). \\
    Energy injection, as a factored drive, generalises over drive strength and sign.
    & The driven steady state at trained and held-out drives, including the unseen
    sign ($A\ge\gamma$ produces the expected divergence).
    & Well-supported. \\
    Accurate counterfactual inference from factoring depends on correct
    specification of the coupling's functional form.
    & Wrong coupling power gives good fit but a wrong counterfactual with no warning, and the correct structure restores the result.
    & Well-supported (a boundary). \\
    The dissociation is reproduced with the evaluated graph-network and transformer
    baselines; domain randomisation inverts only when both coupling signs are
    included in training and does not provide bounded rollouts.
    & The transformer diverges and does not invert at the inversion horizon; domain
    randomisation includes both signs in training and does not provide stability.
    & Well-supported. \\
    On a decode-free future-latent objective, the physics prior separates once a
    label-free image anchor fixes the latent gauge; grounding \emph{quality} is not the
    limiter.
    & The anchored factored hybrid inverts every seed at nMSE $0.40$ vs anchored
    unstructured $0.68$; an exact-position anchor does not improve on the noisy centroid,
    so the residual gap to the renderer-grounded $0.12$ is a property of the objective.
    & Well-supported. \\
    The grounded factored approach does not depend on our own encoder: it transfers to
    frozen public video encoders, one within $2\times$ of oracle fidelity (VideoMAE $0.13$ vs oracle $0.067$) and one at ${\sim}8.8\times$ (V-JEPA2 $0.59$), with a
    localisation-aware readout.
    & Frozen VideoMAE and V-JEPA2 $+$ spatial readout $+$ factored invert at nMSE
    $0.13$ and $0.59$ vs unstructured $0.65$/$0.82$ and diverging [$10$ seeds throughout];
    a global pooled readout misses it; readout-matched, VideoMAE stays $2.6\times$ better
    than our own frozen CNN; scale-flat across V-JEPA2 ViT-L/H/g, gain $1.02\times$;
    a frozen image encoder transfers fully once momentum is differenced across
    frames.
    & Well-supported (encoder transfer; synthetic renders). \\
\end{longtable}
}

\section{Evidence classification}
\label{app:evidence}
We classify evidence as \textbf{well-supported} when effects are large, rest on an
exact test or a clear per-seed separation rather than on the non-overlap of
marginal intervals, and replicate across the stated architectures, seeds, or
epochs. Findings in this class are conservation-to-stability; factored-coupling-to-%
inversion (across MLP, Hamiltonian, Lagrangian, and graph-net; pixel-trained under
grounded perception); the canonical-fixed-kinetic condition; the capacity and
data-efficiency results; the misspecification boundary; the transformer /
domain-randomisation dissociation; the grounded decode-free hybrid, where an
image anchor lets the physics prior separate on a JEPA-style objective; the
learned-binding crack, where a learned soft-assignment binder $+$ the physics prior
inverts the anonymous-object world while its unstructured twin fails; the
external-encoder transfer, where the factored approach inverts on two frozen public video
encoders we did not build, one within $2\times$ of oracle fidelity; and,
beyond conservation,
structure-to-correct-temperature, the guarantee/invertibility tension (with proof),
and the factored-drive generalisation. We classify evidence as
\textbf{directional} when only a single configuration was evaluated or the per-seed separation is not clear. Findings in this class are the reversibility effect on the
engine; the Coulomb margin ($3\times$ vs $15$--$40\times$ for gravity); $N$-scaling
($N{=}32$); the necessity dose-response; in-regime control (overlapping CIs); the
perception partial-recovery; and the engine energy-shed and noise
robustness. Two further tags complete the vocabulary used in the findings table and claim
map: \textbf{supported} marks effects that replicate but with a scoped configuration,
a qualified control, or only some of their pre-registered gates passed (stated in the
row); and \textbf{boundary} marks a mapped limit, a place where the approach
demonstrably stops working, reported as a result of equal standing. A
\textbf{pre-registered negative} is a pre-registered hypothesis whose gates failed, kept
at full strength because the failure is the finding (the energy-aux unification is the
instance).

\section{Pre-registration and reproducibility}
\label{app:prereg}
The protocol, including the metrics, the equal-capacity rule, the non-finite cap,
and the permitted conclusions, was fixed in dated documents ahead of the corresponding
runs, and serves as a prospective specification and as provenance for the results.
The following were recorded before the corresponding
experiments: the fixed-kinetic and decorrelation designs; the
perception and canonical-latent design; the pixel-based comparison
and its success criterion; the necessity analysis; and the
stochastic-dissipative design, including the initial go/no-go experiment and
temperature as the primary metric. The guarantee--invertibility
trade-off and the anti-friction boundary interpretation in the engine and driven experiments
were specified in advance. For the later arc, the specifications are dated ahead of their
results: the oracle-grounding and refusal arms of the video dissipation experiment (the
first of which \emph{refuted our initial grounding-noise attribution}); the
energy-auxiliary unification (all three predictions failed under the pre-registered
thresholds); and the final estimator and its gates, including one physics amendment
recorded \emph{before} the full run (the kinetic-envelope form fails on self-gravitating
worlds, and the energy-balance form replaced it with the gates unchanged) and the
exponential-twin control whose outcome retired the training-floor instrument. The video
world-model characterisation was specified in advance, with its thresholds fixed before
any run; its positive control was registered as the experiment's power condition, and
its firing is the reported finding. The frozen public components used across the
external-encoder experiments are DINOv2-small, VideoMAE-base, V-JEPA2 (ViT-L/H/g), and
V-JEPA2-AC (ViT-g), each from its official public release pinned to a fixed version.
\emph{Statistics.} Aggregates are the mean over at least three seeds, and at
least five seeds at scale. When five or more seeds are available, the reported
bracket $[$lo, hi$]$ is a normal-approximation 95\% confidence interval,
$\text{mean}\pm1.96\,\text{SE}$, a normal-approximation interval that we label as such
wherever it appears; it is not an observed range, and five seeds do not by themselves
justify the approximation, so for the core comparisons we report the per-seed values and,
where the design is paired, the paired difference with its interval. Significance
statements rest on exact tests with a stated null: the binomial regime-match test has
null $\Pr(\text{match})=\tfrac12$ per seed, and the paired sign-flip tests of the
factorial have the null of exchangeable signs of the per-seed differences, Holm-corrected
across the declared family. Display clipping must be distinguished from numerical failure. A finite
value above the temperature-error threshold of $100$ or drift threshold of
$10^4$ remains a finite observation. A non-finite state or diagnostic is
a numerical failure and has no measured finite error. We do not interpret
ratios to substituted failure sentinels quantitatively. Recomputing
finite-only aggregate statistics requires the underlying per-run records.
\emph{Horizons.} Stability is evaluated at
2000 steps, $100\times$ the training horizon. Inversion error is evaluated at
150 steps, within the pre-Lyapunov window used for all models. At longer
horizons, trajectory MSE saturates even for accurate chaotic-system models.
\emph{Regime-match} is therefore the primary inversion metric and nMSE the
secondary metric. The released reproducibility materials comprise
per-experiment model checkpoints, aggregate result tables, the scripts that
regenerate the figures, and a fixed software environment.

\begin{table}[H]
  \centering
  
  \caption{Compute and hyperparameters (defaults; per-experiment overrides in the
    released configs). All runs use Adam and one-step prediction loss unless noted.
  Each run uses one GPU, with runs distributed across eight GPUs.}
  \label{tab:compute}
  \begin{tabular}{@{}l p{0.6\linewidth}@{}}\toprule
    setting & value \\\midrule
    optimiser / learning rate & Adam / $10^{-3}$ \\
    batch size & 1024 (one-step transitions) \\
    epochs & 150--300 (200 typical) \\
    seeds & $\ge3$ ($\ge5$ at scale) \\
    capacity (MLP / HNN / factored) & ${\sim}138$k parameters (matched) \\
    transformer & 2 layers, 4 heads, width 96 (${\sim}90$k) \\
    decode-free JEPA pair & shared $383$k encoder; unstructured dynamics
    $36{,}748$ vs factored $4{,}417$ params \\
    learned binders & tracker $4{,}609$ (assignment only) vs slot-attention
    $105{,}728$ params \\
    frozen encoders & DINOv2 $22$M / VideoMAE $86$M / V-JEPA2 ViT-L/H/g $0.3$--$1$B frozen; head $+$ dyn.\ trained \\
    integrator & leapfrog ($dt{=}0.002$); Tao \citep{tao2016explicit} for the engine \\
    rollout horizons & stability 2000 steps; inversion 150 steps \\
    hardware / time & one GPU per cell; ${\sim}1$--$4$ min/cell (oracle), longer on pixels \\
    \bottomrule
  \end{tabular}
\end{table}

\end{document}